\documentclass[10pt]{article} %
\usepackage{times}

\usepackage{amsmath,amsfonts,bm}

\def\eqref#1{equation~\ref{#1}}

\def\1{\bm{1}}

\DeclareMathAlphabet{\mathsfit}{\encodingdefault}{\sfdefault}{m}{sl}
\SetMathAlphabet{\mathsfit}{bold}{\encodingdefault}{\sfdefault}{bx}{n}

\newcommand{\pdata}{p_{\rm{data}}}

\newcommand{\E}{\mathbb{E}}

\usepackage{algpseudocode}

\usepackage{listings}
\usepackage{amsmath}
\usepackage{amssymb}
\usepackage{mathtools}
\usepackage{amsthm}
\usepackage{algcompatible}

\def\pdata{P_{\text{data}}}
\def\pdatan{\widehat{P}_{\text{data}}}
\def\E{\mathbb{E}}
\def\Rn{\widehat{R}_n}
\def\Rdron{\widehat{R}_{D, n}}

\usepackage{amssymb}%
\usepackage{pifont}%

\usepackage{hyperref}
\usepackage[capitalize,noabbrev]{cleveref}
\usepackage[accepted]{tmlr}

\usepackage{microtype}
\usepackage{graphicx}
\usepackage{subfigure}
\usepackage{booktabs} %
\usepackage{enumitem}

\usepackage{algorithm}%
\usepackage{algpseudocode}%

\theoremstyle{plain}
\newtheorem{theorem}{Theorem}[section]
\newtheorem{proposition}[theorem]{Proposition}

\theoremstyle{definition}

\theoremstyle{remark}

\usepackage[textsize=tiny]{todonotes}

\newcommand\data[1]{{\normalfont \texttt{#1}}}
\newcommand{\std}[1]{\tiny{$\pm$ #1}}
\newcommand{\stdlarge}[1]{\LARGE{$\pm$ #1}}

\newcommand{\namel}{\textsc{RGD}}
\newcommand{\namet}{\textsc{RGD-RevKL}}

\newcommand{\namee}{\textsc{RGD}}
\newcommand{\namex}{\textsc{RGD-$\chi^2$}}
\usepackage{float}

\usepackage[utf8]{inputenc} %
\usepackage[T1]{fontenc}    %

\usepackage{url}            %
\usepackage{booktabs}       %
\usepackage{amsfonts}       %
\usepackage{nicefrac}       %
\usepackage{microtype}      %
\usepackage{xcolor}         %
\usepackage{xspace}
\usepackage{microtype}
\usepackage{amsmath}
\usepackage{graphicx}
\usepackage{subfigure}
\usepackage{booktabs}

\hypersetup{
    colorlinks=true,
    linkcolor=blue,
    filecolor=magenta,      
    urlcolor=cyan,
    citecolor=blue
}

\title{Stochastic Re-weighted Gradient Descent \\ via Distributionally Robust Optimization}

\author{\name Ramnath Kumar \email ramnathk@google.com \\
      \addr Google Inc. 
      \AND
      \name Kushal Majmundar \email majak@google.com \\
      \addr Google Inc.
      \AND
      \name Dheeraj Nagaraj \email dheerajnagaraj@google.com \\
      \addr Google Inc.
      \AND
      \name Arun Sai Suggala \email arunss@google.com \\
      \addr Google Inc.}
      
\def\month{09}  %
\def\year{2024} %
\def\openreview{\url{https://openreview.net/forum?id=KCf5CLAXZq}} %

\begin{document}

\maketitle

\begin{abstract}
We present Re-weighted Gradient Descent (RGD), a novel optimization technique that improves the performance of deep neural networks through dynamic sample  re-weighting.  Leveraging insights from distributionally robust optimization (DRO) with Kullback-Leibler divergence, our method dynamically assigns importance weights to training data during each optimization step. RGD is simple to implement, computationally efficient, and compatible with widely used optimizers such as SGD and Adam. We demonstrate the effectiveness of RGD on various learning tasks, including supervised learning, meta-learning, and out-of-domain generalization. Notably, RGD achieves state-of-the-art results on diverse benchmarks, with improvements of \textcolor{blue}{\textbf{+0.7\%}} on DomainBed, \textcolor{blue}{\textbf{+1.44\%}} on tabular classification, \textcolor{blue}{\textbf{+1.94\%}} on GLUE with BERT, and \textcolor{blue}{\textbf{+1.01\%}} on ImageNet-1K with ViT.
\end{abstract}

\section{Introduction}
\label{sec:intro}
Deep neural networks (DNNs) have become essential for solving a wide range of tasks, including image classification, object detection, machine translation, and speech recognition. The most commonly used paradigm for learning DNNs is empirical risk minimization (ERM~\citet{vapnik1999overview}), which aims to identify a network that minimizes the average loss of training data points. Several algorithms, including SGD \citep{nemirovsky1983wiley}, Adam \citep{kingma2014adam}, and Adagrad \citep{duchi2011adaptive}, have been proposed for solving ERM. However, a drawback of ERM is that it weighs all the samples equally, often ignoring the rare and more difficult samples and focusing on the easier and abundant samples. This leads to suboptimal performance on unseen data, especially when the training data is scarce~\citep{namkoong2017variance}. 
Consequently, recent works have developed data re-weighting techniques for improving the performance of ERM. One particularly fruitful approach in this line of work is the framework of Distributionally Robust Optimization (DRO)~\citep{ben2013robust}, 
which assigns higher weights to hard examples, often leading to models with better performance than ERM.

DRO selects the best model while also accounting for various uncertainties in the training data distribution~\citep{ben2013robust}. In particular, DRO treats the data distribution as uncertain and finds models which are robust to perturbations in the data distribution (e.g., removing a small fraction of points, adding random noise to each data point, etc.). This makes the model more robust to noise in the training dataset. For example, in the context of classification, this forces the model to place less emphasis on noisy features and more emphasis on useful and predictive features. Consequently, models optimized using DRO have good generalization guarantees on unseen samples, and good performance on heterogeneous subpopulations in the data~\citep{namkoong2017variance, duchi2018learning}.

Concretely, let $\Theta$ be the model space, $\pdata$ be the data distribution, and $\ell(\theta, z)$ be the loss of point $z$ w.r.t model $\theta$. Unlike ERM which minimizes the average loss $\E_{z\sim\pdata}[\ell(\theta, z)],$ DRO minimizes the following objective
\[
\inf_{\theta}\sup_{P': D(P'||\pdata) \leq \rho} \E_{P'}[\ell(z,\theta)].
\]
Here, $\rho$ is the perturbation radius, and $D$ is a divergence that measures the distance between two probability distributions.
Popular choices for $D$ include $f$-divergences, which are defined as $D_f(Q||P) = \E_{P}[f(dQ/dP)]$ for some convex function $f: \mathbb{R}_+ \to \mathbb{R}$. In this case, \citet{shapiro2017distributionally} derived the following \emph{equivalent} dual formulation of the above objective
\begin{align}
\label{eqn:dual_dro}
    \inf_{\lambda \geq 0} \inf_{\eta \in \mathbb{R}} \E_{z\sim \pdata}\left[\lambda f^*\left(\frac{\ell(\theta, z)-\eta}{\lambda}\right)\right] + \lambda \rho + \eta.
\end{align}
Here, $f^*(s) = \sup_t\{st-f(t)\}$ is the Fenchel conjugate of $f$. This alternative way of expressing DRO shows how it implicitly reweighs samples using the conjugate $f^*$. The seminal works of \citet{duchi2016statistics, namkoong2017variance} studied this objective for various $f$-divergences and showed that it has variance reduction properties, and leads to models with good generalization performance under \emph{small} perturbations ($\rho = O(1/n)$, where $n$ is the data size). Furthermore, \citet{duchi2018learning} showed that DRO under \emph{large} perturbations ($\rho = O(1)$) leads to models with good fairness, tail risk guarantees. In another line of work, \citet{li2020tilted} (and its extended version \citet{beirami2023tilted}) considered KL divergence DRO, which is obtained by choosing $f(x) = x\log{x}$,  and showed that setting dual variable $\lambda$ to a negative value results in robust models that can withstand corruptions in the training data.

Inspired by these impressive properties, several recent studies have developed algorithms for optimizing DRO and designed data re-weighting techniques for various learning tasks. These algorithms fall into two broad categories:
(a) \textit{Primal-Dual techniques} which rely on alternating mirror ascent, descent to solve the min-max DRO objective~\citep{namkoong2016stochastic, yan2020optimal, fidon2020distributionally}, and (b)
\textit{Compositional optimization techniques} which solve an equivalent compositional/dual form of DRO, which takes the form $g(\mathbb{E}_{z}[h(z, \theta)])$, for some functions  $g, h$~\citep{qi2021online,qi2020attentional, qi2022stochastic, beirami2023tilted}.
While these algorithms come with good convergence guarantees, they have certain drawbacks that limit their use in practice. (a) \emph{Scalability:} primal-dual algorithms require updating and sampling from a probability distribution over the entire dataset at each iteration, making them computationally expensive. Although compositional optimization techniques alleviate this issue, gradient estimation within these algorithms is non-trivial as the objective is no longer an empirical mean of the losses evaluated at the training data points. Overcoming this often necessitates maintaining moving averages of sample weights, which introduces additional hyperparameters, complicating their application to large-scale scenarios~\citep{qi2020attentional, li2020tilted}.  (b) \emph{Robustness to Outliers:} many real-world datasets contain outliers\footnote{{An outlier is a data point that lies significantly outside the typical pattern of a dataset. These outliers could be because of noise in data collection process or could be introduced by a malicious adversary. In this work, we are primarily concerned about the former type of outliers.}}, which pose challenges to algorithms optimizing DRO. In particular, these outliers often result in poor performance and instability during the DRO training process~\citep{zhu2022generalized, zhai2021doro}.  Existing works often fail to account for outliers in real datasets, leading to subpar performance (see Table~\ref{tab:comparison_prior_works} for a detailed comparison).

\begin{table*}[tbh!]
	\centering
	\resizebox{\textwidth}{!}{
		\begin{tabular}{c|| c| c | c|c}
			\hline
			Paper & Algorithm  & \begin{tabular}{c}Per-step compute\\ complexity\\ (compared to  SGD)\end{tabular} & Hyper-parameters  & \begin{tabular}{c} Handles\\ outliers in DRO?\end{tabular} %
			\\ [0.5ex]
			\hline\hline
			\citet{namkoong2016stochastic}  & \begin{tabular}{c}P-D\end{tabular} &  \begin{tabular}{c}  $O(\log{n})\times$ more expensive\\\end{tabular}  & lr of dual variables & No
			\\\hline
			\citet{qi2021online} & \begin{tabular}{c}C-M\end{tabular}&  \begin{tabular}{c}$2\times$ more expensive\end{tabular}  & moving avg. parameter& No
			\\\hline
			\begin{tabular}{c}  \citet{li2020tilted, beirami2023tilted}\\\citet{qi2020attentional}\end{tabular}  & \begin{tabular}{c}C-M\end{tabular} & same as SGD& \begin{tabular}{c} exponential scale,\\ moving avg. parameter\end{tabular} & No\footnotemark
			\\\hline
			\textbf{This Work} & \begin{tabular}{c} stochastic optimization\\ of inner obj.\end{tabular} & same as SGD & \begin{tabular}{c} clipping level\end{tabular} & Yes
			\\\hline
		\end{tabular}
	}
 \caption{Comparison with relevant prior works for optimizing KL-DRO. See Section~\ref{sec:related_work} for a detailed discussion. P-D in the $2^{nd}$ column refers to primal-dual, and C-M refers to compositional minimization. $3^{rd}$ column corresponds to the cost of running each step of the algorithm. $4^{th}$ column corresponds to additional parameters introduced by the algorithm on top of learning rate (lr) of primal variables. }
	\label{tab:comparison_prior_works}
\end{table*}

In this work, we address the aforementioned limitations of  DRO optimization techniques. We focus on KL divergence-based DRO, and develop a lightweight algorithm for efficiently solving the resulting objective. Our algorithm simply optimizes the inner objective in Equation~\ref{eqn:dual_dro} using SGD. This gives rise to our stochastic Re-weighted Gradient Descent (\namel) algorithm, a variant of the classical SGD, that re-weights data points during each optimization step based on their difficulty. A key component of our algorithm is weight clipping that we introduce to protect against (benign) outliers and stabilize the algorithm. 
As demonstrated in our experiments (see Section~\ref{sec:experiments}, Appendix~\ref{appendix:addn_results}), weight clipping significantly improves the performance of our algorithm on numerous learning tasks involving real-world datasets. Another noteworthy aspect of our algorithm is that it has the same runtime as SGD, has only one hyper-parameter, and scales to models with billions of parameters.

\subsection{Evaluation}
In our experiments, we show that using our re-weighting scheme on top of existing learning algorithms improves their generalization performance in a variety of learning tasks including supervised learning, meta learning, out-of-domain generalization. While prior works focused on settings involving fairness, class imbalance to show improvements of DRO type methods, our work is the first to show significant improvements in generalization in large scale learning tasks across various domains. 

\paragraph{Supervised Learning:} We evaluate RGD on several supervised learning tasks in language and vision domains. In the language domain, we apply RGD for BERT fine-tuning on the General Language Understanding Evaluation (GLUE) benchmark and show that RGD outperforms the BERT baseline by $+1.94\%$. In the vision domain, we apply RGD for ImageNet-1K classification using ViT-S model, and show that RGD outperforms the ViT-S baseline by $+1.01\%$.

\paragraph{Tabular Classification:} 
Recently, \citet{majmundar2022met} introduced a tabular representation learning method called MET. Deep learning methods trained with the learned representations from MET achieved SOTA performance on downstream classification tasks, significantly improving upon Gradient Boosting decision trees (GBDT; \citet{friedman2001greedy}). 
Our experiments show that applying RGD to the MET framework improves its performance by $1.51\%$ and $1.27\%$ on binary and multi-class tabular classification, respectively.

\paragraph{Domain Generalization:} In domain generalization, the distributions of train and test datasets could be different (for example, training on pictures of real dogs and evaluating cartoon dogs). This task requires robustness to distribution shifts and DRO is a natural framework in this context. \citet{gulrajani2020search} showed that the ERM framework applied over deep networks, is highly effective for this problem. Perhaps surprisingly, this remained the state-of-the-art (SOTA) algorithm for a long time. Only recently, ERM has been beaten on this challenging task~\citep{cha2022domain, addepalli2022learning}. In this work, we show that using \namel~on top of these recent techniques further boosts their performance by $0.7\%$ and gives SOTA results on this task.

\paragraph{Meta-Learning:} In meta-learning, we aim to learn models that generalize to new tasks with limited data. Predominant approaches in this domain use the classical ERM to optimize deep networks \citep{finn2017model,  snell2017prototypical, kumar2022effect}. However, a common issue in this domain is that the learned models solve most tasks but fail catastrophically in some tasks. Consequently, this has promoted works that focus on worst-case performance \citep{collins2020task}. Recently, \cite{beirami2023tilted} used KL-DRO to tackle this problem. However, the authors applied their algorithm for solving meta-regression on a toy-dataset that is free of outliers, and haven't showcased their algorithm on  practically relevant datasets such as Omniglot, \textit{mini}ImageNet. In this work, we show that using \namel~as an off-the-hat addition to Model-Agnostic Meta-Learning (MAML)~\citep{finn2017model} can significantly improve the worst-case accuracy of these models on meta-classification benchmarks such as Omniglot, \textit{mini}ImageNet by up to $3\%$.

\subsection{Contributions}
This work makes the following key contributions:
\begin{itemize}
\item \textbf{KL-DRO Inspired Re-weighting (\namel).}  We introduce \namel, a novel, lightweight data re-weighting technique that improves the generalization of deep neural networks. Inspired by the principles of KL-DRO, \namel~dynamically re-weights samples during optimization based on their difficulty. We further enhance robustness with weight clipping to mitigate the influence of outliers\footnotetext{\citet{li2020tilted} developed hierarchical TERM framework for handling outliers in DRO. But their framework is only applicable to settings such as group fairness where the learner has a priori knowledge of group memberships.}. \namel~is versatile and easily integrated with widely used optimizers like Adam, SGD.

\item \textbf{State-of-the-Art (SOTA) Performance Enhancement.} Extensive experiments demonstrate that \namel~delivers significant performance gains across diverse learning tasks. In tabular classification, \namel~boosts the accuracy of MET  \citep{majmundar2022met} by +1.44\%. For out-of-domain generalization, \namel~outperforms FRR \citep{addepalli2022learning} on DomainBed by +0.7\%.  Additionally, \namel~improves the performance of BERT on GLUE benchmarks by +1.94\% and ViT on ImageNet-1K by +1.01\%. (see Section~\ref{sec:experiments}, Appendix~\ref{appendix:addn_results})
\end{itemize}
\section{Related Work}
\label{sec:related_work}
This section reviews relevant research on DRO. For a comprehensive overview of other popular data reweighting methods in machine learning, including AdaBoost, curriculum learning, please refer to Appendix~\ref{appendix:related_work}.

DRO dates back to the early works of \citet{ben2009robust, ben2013robust}. Since then several works have studied various statistical and optimization aspects of DRO. The seminal works of \citet{lam2016robust, namkoong2017variance, duchi2021statistics} formally showed that minimizing empirical DRO risk - under \emph{small} perturbations ($\rho = O(1/n)$, where $n$ is the data size) - is equivalent to minimizing sum of empirical risk and its standard deviation. Consequently, optimizing DRO leads to a better bias-variance trade-offs and generalizing models. In this work, we rely on this property of DRO to develop our re-weighting scheme. In another seminal work, \citet{duchi2018learning} showed that DRO  - under \emph{large} perturbations ($\rho = O(1)$) - leads to models with good tail performance.

The aforementioned properties of DRO has led to numerous works applying it in various learning scenarios. For instance,~\citet{duchi2018learning, sagawa2019distributionally,qi2021online,qi2020attentional, li2020tilted, beirami2023tilted} used DRO to tackle problems of class-imbalanaced classification and fairness. In another line of work, \citet{namkoong2017variance, fidon2020distributionally} studied DRO for designing models that generalize better than ERM. Our work falls in this second category of works that focus on generalization.

From an optimization perspective, several works have focused on designing efficient algorithms for optimizing the DRO objective. These algorithms can be classified into two broad categories: \emph{primal-dual}~\citep{namkoong2016stochastic, fidon2020distributionally, yan2020optimal}, \emph{compositional optimization techniques}~\citep{qi2021online, qi2020attentional, qi2022stochastic, li2020tilted}. One of the key drawbacks of primal-dual techniques is that they update and sample from a probability distribution over the entire training data at each step (\emph{aka.} dual variables). A naive implementation of this step takes $O(n)$ time, which is prohibitive for large-scale tasks. \citet{namkoong2016stochastic} reduced the complexity of this step to $O(\log{n})$ using data structures such as balanced binary search trees. However, the resulting algorithms are hard to implement in practice. Another drawback of these algorithms is that they require storing a buffer of weights for the entire dataset, which is infeasible at the scale of LLMs. Even if storing the weights is feasible, the presence of data augmentations complicates the resulting algorithms. Compositional optimization algorithms overcome these drawbacks by working with an equivalent dual formulation of DRO that can be written as composition of two functions: $g(\mathbb{E}_z[h(z,\theta)])$.  One major drawback of these techniques though is that estimating the gradient $\nabla_{\theta}g(\mathbb{E}_z[h(z,\theta)])$ from a mini-batch is non-trivial. This requires certain additional steps in the algorithm which add to its complexity. For instance, the algorithm of \citet{qi2021online} requires making two backward passes at each step of SGD. The convergence analysis of \citet{li2020tilted} required two independent mini-batches at each iteration. The algorithms of \citet{qi2020attentional,li2020tilted} both require maintaining a moving average of the weights of the mini-batches, thus adding an additional hyper-parameter to the algorithm. 

\paragraph{Outlier robust DRO.}  Recent works have shown that DRO is highly sensitive to outliers~\citep{zhai2021doro, zhu2022generalized}. This is because DRO tends to magnify the influence of outliers by upweighting them further. This is one of the  reasons for the suboptimal performance of existing DRO algorithms on real world datasets. To address this, \citet{zhai2021doro} consider an adversarial model for outliers (where $\epsilon$-fraction of training data could be arbitrarily corrupted by a malicious adversary). They develop a heuristic to remove the outliers during each descent step for $\chi^2$-DRO and CVaR. While interesting, this is too strong of an adversary model which leads to data wastage in practice. 

\paragraph{Min-min DRO.} Min-min DRO minimizes the following objective: $\inf_{P': D(P'||\pdata) \leq \rho} \E_{P'}[\ell(z,\theta)]$. Contrast this with DRO which minimizes: $\sup_{P': D(P'||\pdata) \leq \rho} \E_{P'}[\ell(z,\theta)]$. Unlike DRO which is primarily studied for generalization and fairness properties, min-min DRO is studied for training in the presence of outliers~\citep{li2020tilted, kumar2021constrained, majidi2021exponentiated}. Instead of upweighting high loss points, min-min DRO downweights them.

\paragraph{Other applications of DRO.} \citet{sagawa2019distributionally} optimized Group DRO for fair models when the group information is known. \cite{sinha2017certifying} studied DRO with Wasserstein divergence for learning models that are robust to adversarial perturbations. 
DRO also appears in many classical statistical problems. For example, many boosting algorithms (including AdaBoost) can be viewed as performing DRO with KL-divergence-based uncertainty sets~\citep{ arora2012multiplicative, friedman2001greedy}. \citet{faury2020distributionally} relied on KL-DRO for counterfactual risk minimization. \citet{sakhi2020improving} use DRO for improving offline contextual bandits algorithms.

\paragraph{Optimization Techniques for Improved Generalization.} Several optimization techniques, that fall outside the umbrella of DRO, have been proposed for improving the generalization of ML models. Of these, Sharpness-Aware Minimization (SAM)~\citep{foretsharpness} is perhaps the most popular technique. From a theoretical perspective, SAM performs robust optimization in the weight space (that is SAM tries to learn a model that is robust to perturbations of weights). In contrast, RGD performs robust optimization in the distribution space. So, RGD and SAM are orthogonal to each other and can potentially be merged together to boost the performance. See Appendix~\ref{sam_comparison} for empirical comparison between RGD and SAM.

\section{Algorithm and Derivation}
\label{sec:algo}
In this section, we first formally introduce DRO and describe its generalization properties. Next, we derive \namel~as a technique for solving DRO. 
\subsection{Distributionally Robust Optimization}

Consider a general learning problem where we are given $n$ i.i.d samples $\{z_i\}_{i=1}^n$ drawn from some unknown distribution $\pdata$. Let $\pdatan$ be the empirical distribution over these samples. Our ideal goal is to find a model $\theta\in \Theta$ that minimizes the population risk: $R(\theta) \coloneqq \E_{\pdata}[\ell(z; \theta)]$. Here $\ell(z; \theta)$ is the loss of $z$ under model $\theta$. Since $\pdata$ is typically unknown, a standard practice in ML/AI is to minimize the empirical risk, which is defined as $$\Rn(\theta) \coloneqq \E_{\pdatan}[\ell(z;\theta)] = \frac{1}{n}\sum_{i=1}^n \ell(z_i; \theta).$$  

In DRO, we assume that a ``worst-case'' data distribution shift may occur, which can harm a model's performance. So, DRO optimizes the loss for samples in that ``worst-case'' distribution, making the model robust to perturbations (see Figure~\ref{fig:dro} for illustration). Letting $D$ be a divergence that measures the distance between two probability distributions, the population and empirical DRO risks w.r.t $D$ are defined as
\[
R_{D}(\theta) \coloneqq \sup_{P': D(P'||\pdata) \leq \rho} \E_{P'}[\ell(z,\theta)],
\quad \Rdron(\theta) \coloneqq \sup_{P': D(P'||\pdatan) \leq \rho} \E_{P'}[\ell(z,\theta)].
\]
Here, $\rho$ is the perturbation radius.
Popular choices for $D$ include $f$-divergences, which are defined as $D_f(Q||P) = \E_{P}[f(dQ/dP)]$ for some convex function $f: \mathbb{R}_+ \to \mathbb{R}$. We note that many popular divergences, such as Kullback–Leibler (KL) divergence ($f(x) = x\log{x}$), Total Variation distance ($f(x) = \frac{1}{2}|x-1|$), and $\chi^2$-divergence ($f(x) = (x-1)^2$), fall into this category.\vspace{0.05in}\\
\textbf{Generalization.} Models learned using ERM can suffer from poor generalization (\emph{i.e.,} performance on unseen data) in high-variance settings. For instance, consider the following well-known generalization guarantee that  holds with high probability for any $\theta \in \Theta$~\citep{wainwright2019high} 
\begin{align}
\label{eqn:pop_risk_ub}
    R(\theta) \leq \Rn(\theta) + c_1\sqrt{\frac{\text{Var}_{\pdatan}(\ell(z; \theta))}{n}} + \frac{c_2}{n}.
\end{align}
Here, $c_1, c_2>0$ are constants, and $\text{Var}_{P}(\ell(z; \theta))$ is the variance of $\ell(z;\theta)$ w.r.t distribution $P$.
Such bounds hold under certain regularity conditions on $\ell$ and $\Theta$. While ERM minimizes the first term in the RHS above, it totally ignores the second term involving the variance. Consequently, in high-variance and/or small $n$ settings where $R(\theta)$ and $\Rn(\theta)$ are far away from each other, ERM tends to have poor generalization guarantees. A natural technique to address this issue is to learn models that consider the bias-variance trade-off and minimize the following objective.
\[
\min_{\theta \in \Theta}\Rn(\theta) + c_1\sqrt{\frac{\text{Var}_{\pdatan}(\ell(z; \theta))}{n}}.
\]
However, minimizing this objective is computationally intractable even when the loss $\ell$ is convex in $\theta$, as the overall objective is non-convex. Recent works have made an interesting connection between the above objective and DRO to address this issue. 
Specifically, when $D$ is an $f$-divergence, the following result holds with high probability, whenever the perturbation radius $\rho = \frac{c}{n}$, for some appropriately chosen constant $c>0$~\citep{lam2019recovering, namkoong2017variance} 
\[
\Rdron(\theta) = \Rn(\theta) + c_1\sqrt{\frac{\text{Var}_{\pdatan}(\ell(z; \theta))}{n}} + \frac{c_3}{n} \quad \forall \theta \in \Theta.
\]
 Ignoring the lower order terms (i.e., $1/n$ terms), the above equation, together with Equation~\ref{eqn:pop_risk_ub}, shows that the empirical DRO risk $\Rdron(\theta)$ is an upper bound of the population risk $R(\theta)$ at any $\theta \in \Theta$. Furthermore, it can be seen that empirical DRO is equal to the empirical risk plus a variance term (modulo the lower order term). This variance term acts as a regularizer during the optimization of empirical DRO and leads to models with smaller variance, and with good generalization guarantees.

\subsection{Stochastic Re-weighted Gradient Descent (\namel)}
The above discussion motivates the use of DRO risk for learning models, especially in high variance and/or low sample regime. We now derive our $\namel$~algorithm as a technique to minimize the empirical DRO risk $\Rdron$, and to learn models with better generalization guarantees than ERM. Specifically, we consider KL divergence-based DRO, where one adds perturbations to create data distributions that are close to the original data distribution in the KL divergence metric, and learn a model with best performance over all possible perturbations. 
The following proposition derives the \emph{equivalent} dual representation of the KL divergence-DRO objective, the proof of which can be found in Appendix~\ref{sec:proof1}.
\begin{proposition}{\citep{shapiro2017distributionally}}
    \label{prop:dual_dro_kl}
     Consider DRO with KL-divergence-based uncertainty set.  Then $\min_{\theta\in\Theta}\Rdron$ can be rewritten as:
    $$
    \min_{\theta\in\Theta}\frac{1}{\gamma}\log{\E_{\pdatan}[e^{\gamma\ell(z; \theta)}]},
    $$
    for some constant $\gamma>0$ that is independent of $\theta$.
\end{proposition}
\begin{algorithm}[t]
\caption{Re-weighted Gradient Descent (\namel)}
\label{alg:algorithm_overall}
\begin{algorithmic}[1]
  \small
  \STATE \textbf{Input:}   Data $\{z_i\}_{i=1}^n,$ learning rate sequence $\{\eta_t\}_{t=1}^T$, number of iterations $T$, loss function $\ell$, re-weighting function $g$, mini-batch size $B$
  \FOR{$t = 0 \dots T-1$}
  \STATE Sample minibatch $\{z_i\}_{i=1}^B$
  \STATE Compute losses for points in the minibatch:
       $$\ell_i \leftarrow \ell(z_i;\theta_t), \ \forall i \in {1 \dots B}$$
  \STATE Compute per-sample weights:
    $$w_i \leftarrow g(\ell_i) \ \forall i \in {1 \dots B}$$
  \STATE Compute the weighted pseudo-gradient: $$v_t \leftarrow \frac{1}{B}\sum_{i=1}^B w_i \nabla_{\theta} \ell(z_i; \theta_t)$$
  \STATE Update model parameters:  $$\theta_{t+1} \leftarrow \Pi_{\Theta}(\theta_t - \eta_t v_t)$$
  \ENDFOR
\end{algorithmic}
\end{algorithm}
Equipped with the dual representation, we now derive our \namel~algorithm. Observe that minimizing the compositional objective $\log{\E[\exp{(\gamma\ell(z; \theta))}]}$ is equivalent to minimizing the inner objective $\E[\exp{(\gamma\ell(z;\theta))}].$  In this work, we perform SGD on this inner objective, which leads to the following update:
\[
\theta_{t+1} \leftarrow \Pi_{\Theta}\left( \theta_t - \gamma \eta_t \frac{1}{B}\sum_{i-1}^B e^{\gamma \ell(z_i; \theta_t)}\nabla_{\theta_t}  \ell(z_i;\theta_t)\right).
\]
Here $\Pi_\Theta$ is the projection onto the feasible set $\Theta$. The following proposition shows that this update rule converges to the minimum of $\Rdron$ under certain conditions on $\ell(z;)$. The proof, given in Appendix~\ref{sec:proof2} follows from an application of \citet[Theorem 2]{shamir2013stochastic}.

\begin{proposition}
    \label{prop:convergence}
    Assume that $\Theta$ is a convex and compact set.
    For all data points $z$ suppose that $\ell(z;\cdot)$ is convex, continuously differentiable and bounded in the range $[-M,M]$, and $\nabla\ell(z;\cdot)$ is uniformly bounded over the set $\Theta$. Let the step-size sequence $(\eta_t)_t$ be such that $\eta_t = \frac{C}{\sqrt{t}}$ $\forall$ $1\leq t \leq T$  or $\eta_t = \frac{C}{\sqrt{T}}$ $\forall$ $1\leq t\leq T$. Then the sub-optimality gap satisfies: 
    $$\E_{\theta_T}\frac{1}{\gamma}\log{\E_{\pdatan}[e^{\gamma\ell(z; \theta_T)}]} 
    - \min_{\theta \in \Theta}\frac{1}{\gamma}\log{\E_{\pdatan}[e^{\gamma\ell(z; \theta)}]} = O\left(\frac{\log T}{\sqrt{T}}\right).$$
    
\end{proposition}

This shows that choosing the re-weighting function $g(u)$ in Algorithm~\ref{alg:algorithm_overall} as $e^{\gamma u}$ (for some appropriate choice of $\gamma$) leads to robust models with better generalization guarantees. This is the choice of $g$ we use in our paper. One thing to note here is that, even though the proposition is specific to SGD, it can be easily extended to other optimization techniques such as Adam. 

\paragraph{Weight Clipping.} In our experiments, when computing per-sample weights, we clip the loss $\ell$ at some constant $\tau>0$; that is, we use $g(u) = e^{\gamma\min(u, \tau)}$. We observed this clipping to help stabilize the training in the presence of outliers (see Section~\ref{appendix:robustness_hparam} for empirical evidence). We note that this is different from other common techniques, such as loss, gradient clipping, which are used to make the learning process robust to outliers in empirical risk minimization~\citep{yang2010relaxed, catoni2018dimension, menon2019can, koloskova2023revisiting}. We plan to investigate the robustness properties of weight clipping in more detail in future work. In our experiments, we choose the scale parameter $\gamma = 1/(\tau+1)$. Even with this fixed choice of $\gamma$, our algorithm provides a significant boost in performance over vanilla optimization techniques (see Section~\ref{appendix:robustness_hparam} for ablations on our design choices).

\paragraph{Choice of divergences.} One could rely on other divergences instead of the KL-divergence we used in the above derivation. From a theoretical perspective, DRO with many $f$-divergences (KL, reverse KL, chi-squared etc.) provides an upper bound on the true population risk \citep{lam2016robust, duchi2021statistics}. For example, using $\chi^2$-divergence ($f(x) = (x-1)^2$) gives us the following re-weighting function for positive loss functions: $g(u) = u + \tau$. Using reverse KL-divergence ($f(x) = -\log{x}$) gives us the following re-weighting function $g(u) = \frac{1}{\tau-u}$ for some appropriate choice of $\tau$ (see Appendix~\ref{sec:other_divergences} for more details). Observe that the reverse KL-divergence based weighting function is much more aggressive in up-weighting high loss points than KL-divergence based weighting function. In the sequel, we denote the approach with $g(u) = u + \tau$ as \namex, $g(u) = \frac{1}{\tau-u}$ as \namet, and $g(u) = e^{\min(u, \tau)/(\tau+1)}$ as \namee.  While all the three techniques provided a performance boost over ERM (see Appendix~\ref{appendix:difference}), KL-divergence based reweighting has better performance than reverse KL, chi-squared divergences. However, the current theoretical understanding of DRO doesn’t explain this nuanced behavior and we believe that this is an interesting direction for future research. From a practitioner perspective, the choice of $f$-divergence could be treated as a hyper-parameter which could be tuned using cross-validation.

\section{Experiments}
\label{sec:experiments}
In this section, we first present ablations on various design choices in our algorithm. Next, we present empirical evidence showing that  \namel~outperforms ABSGD~\citep{qi2020attentional}, TERM~\citep{li2020tilted}, two state-of-the-art algorithms for optimizing KL-DRO. Finally, we present large scale experiments showing that \namel~can be widely applied across tasks such as supervised learning, meta learning to boost the generalization performance of existing learning algorithms. 
Details regarding hyperparameter tuning are relegated to Appendix~\ref{appendix:reproducibility}. Additional experiments on class imbalanced classification, and  large-scale tasks such as miniGPT pre-training, EfficientNet finetuning are presented in Appendix~\ref{appendix:addn_results}. 

To ensure fair comparisons, we integrated RGD into existing baseline codebase for each of the experiments, maintaining the same optimizer (primarily Adam for most experiments, SGD for CIFAR experiments) for both baseline and RGD versions. RGD introduces only one additional hyperparameter, the clipping factor ($\tau$). The optimizer, weight decay, batch sizes, etc. were kept constant with the baseline across all our experiments. 

\begin{figure}[hbt!]
\centering
\includegraphics[width=0.6\linewidth]{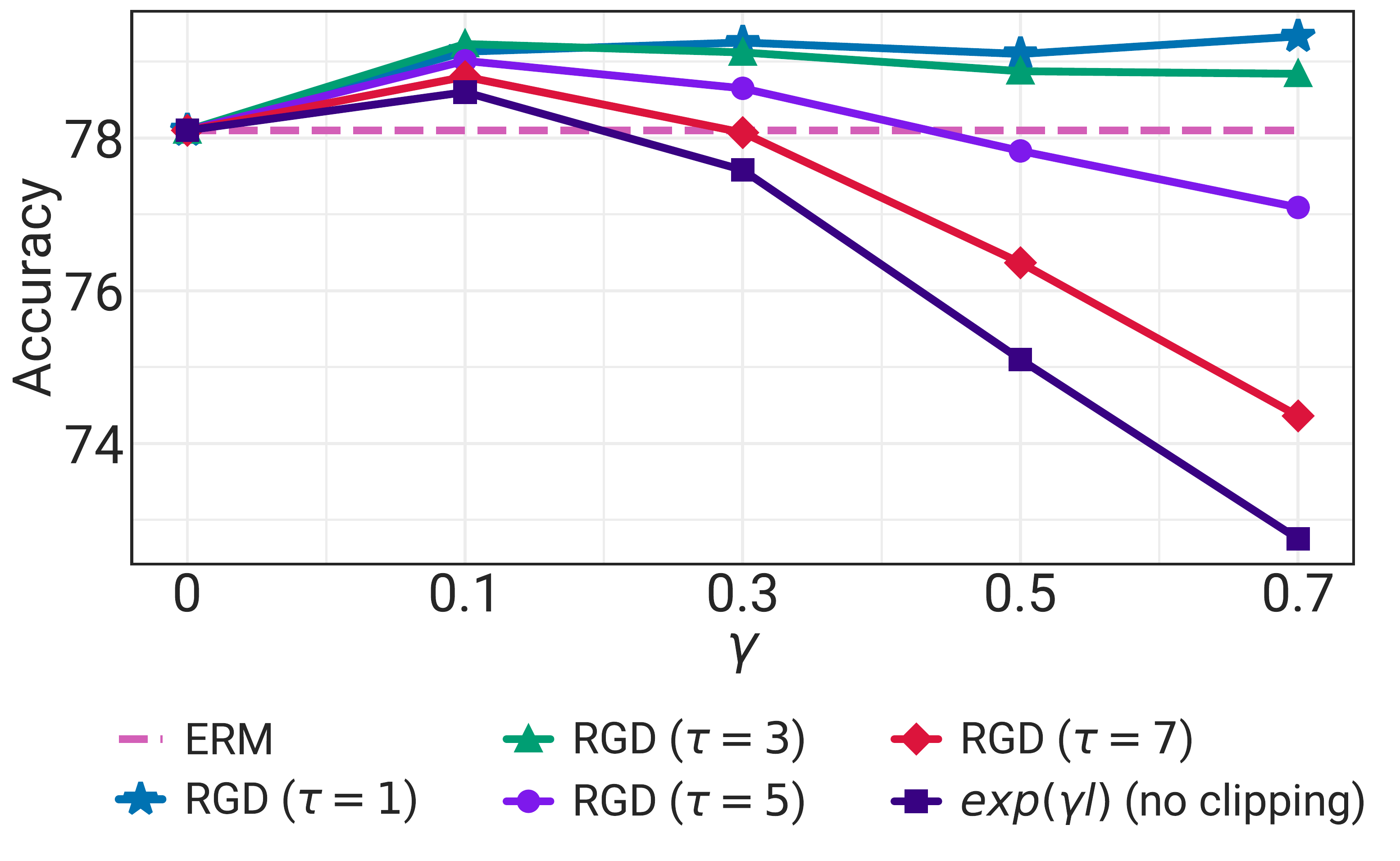}
\vspace{-0.2in}
\caption{Ablation of scaling and clipping factor of RGD training regime on the Imagenet dataset with a ViT-S backbone.}
\label{fig:ablation_hyperparam}
\end{figure}

\subsection{Ablation Studies}
\label{appendix:robustness_hparam}
In this section, we present ablations justifying the key design choices in our algorithm. All the results presented here are for ImageNet-1K classification using ViT-S model.
\paragraph{Choice of scale parameter ($\gamma$).} Recall, \namel~uses $\gamma = \frac{1}{\tau + 1}$, where $\tau$ is the clipping level. To understand the utility of this choice, we vary $\gamma$. To be precise, we set $\tau = 1$ and get accuracy numbers of \namel~with various values of $\gamma$.  Figure~\ref{fig:ablation_hyperparam}  (light blue line with \emph{stars}) presents the results from this experiment. For each value of $\gamma$, we report the accuracy obtained using the best learning rate identified using hold-out set validation. It can be seen that \namel~is fairly robust to the choice of $\gamma$.

\paragraph{Importance of clipping.} We now study the importance of clipping. To this end, we replace the proposed reweighting function in \namel~with $g(u) = e^{\gamma u}$. Figure~\ref{fig:ablation_hyperparam}  (dark blue line with \emph{squares}) presents the results from this experiment, for various values of $\gamma$.  It can be seen that the best accuracy without clipping is $1\%$ off compared to $\namel$~with clipping (with $\tau=1$). We believe this performance drop primarily happens because of the high weights given to the outliers. This demonstrates the importance of clipping. Next, we vary the clipping factor $\tau$. It is evident that as $\tau$ increases, the performance of \namel~drops and approaches the performance of no clipping. This shows the importance of setting an appropriate $\tau$.

To further demonstrate that our weight clipping can effectively handle benign outliers, we perform the following experiment. We randomly flip the labels in CIFAR-10, CIFAR-100 datasets (we vary the proportion of flips from $0\%$ to $40\%$) and compare the performance of RGD with state-of-the-art KL-DRO optimization technique TERM~\citep{li2020tilted}. The results from this experiment are reported in Table~\ref{tab:cifar_label_noise_uniform}. Note that TERM, which does not account for outliers, performs poorly in this experiment, as it upweights the corrupted/mislabeled points. Whereas, RGD which clips the weights of outliers, achieves the best performance.

\begin{table*}[hbt!]
\centering
\caption{Results on CIFAR-10 and CIFAR-100 dataset with corrupted labels.}
\vskip 0.15in
\resizebox{\linewidth}{!}{
\begin{tabular}{lcccc|cccc}
\toprule
Dataset &
      \multicolumn{4}{c|}{\data{CIFAR-10}} &
      \multicolumn{4}{c}{\data{CIFAR-100}} \\
Loss &
 \data{0\%} &
  \data{20\%} &
  \data{40\%} &
  \textbf{Avg.} &
  \data{0\%} &
  \data{20\%} &
  \data{40\%} &
  \textbf{Avg.} \\
\midrule
Default (Cross Entropy) & 92.89 \std{0.32}  & 76.83 \std{2.30}  &  70.77 \std{2.31}  & 80.16 & 70.50 \std{0.12} & 50.86 \std{0.27} & 43.01 \std{1.16} &  54.79\\
TERM~\citep{li2020tilted} & 92.90 \std{0.09} & 58.7 \std{39.76} & 73.17 \std{31.59} & 74.92& 70.42 \std{0.41} & 63.85 \std{1.40} & 46.59 \std{22.80} & 60.29 \\
\namee~\textbf{(Ours)}     & \textbf{93.04} \std{0.24}  & \textbf{90.69} \std{0.24}   &  \textbf{88.90} \std{0.15}  & \textbf{90.88}  &  \textbf{71.06} \std{0.22} & \textbf{64.61} \std{0.39} & \textbf{57.17} \std{0.80} &  \textbf{64.28} \\

\bottomrule
\end{tabular}}
\label{tab:cifar_label_noise_uniform}

\end{table*}

\subsection{Comparison with existing KL-DRO optimization techniques}
\label{sec:ablation}
In this section, we present experimental results on ImageNet classification with ViT-S model to showcase the efficacy of \namel~over state-of-the-art KL-DRO optimization techniques ABSGD, and TERM (see Table~\ref{tab:rgd_vs_kldro}). We were unable to compare with the recent SCDRO algorithm of \citet{qi2022stochastic} due to technical difficulties with implementing it in JAX~\citep{jax2018github} (the main difficulty arises from syncing values of parameters across various devices). Furthermore, the algorithm is significantly more complex with many hyperparameters.  Table~\ref{tab:rgd_vs_kldro} shows that ABSGD, TERM with extensive tuning of three hyperparameters (exponential scale, learning rate , and moving average parameter) achieve similar performance as baseline ERM. In contrast, \namel~outperforms the baseline ERM by $1.1\%$ with minimal tuning of hyperparameters, highlighting its efficacy. One of the reasons for this performance difference between \namel, and ABSGD, TERM is the weight clipping we perform in our algorithm, which guards it from outliers. We note that SCDRO  doesn't perform weight clipping and could potentially suffer from a similar performance drop as ABSGD, TERM. Additional details about the experiment can be found in Appendix~\ref{appendix:ablation}.

Next, we compare \namel~with ABSGD, TERM for the problem of class imbalanced classification. For this experiment, we consider the long-tailed CIFAR-10, CIFAR-100 datasets~\citep{cui2019class}. Table~\ref{tab:long_tailed_cifar_kldro_priors} presents the results from this experiment. It can be seen that RGD outperforms both ABSGD and TERM by $>1\%$ on average. Additional details about this experiment, including comparison with specialized techniques for class imbalanced classification such as \emph{class-balanced} loss \citep{cui2019class} and \emph{focal} loss~\citep {lin2017focal},  can be found in Appendix~\ref{sec:toy_example}.

\begin{table}[hbt!]
\vspace{-0.2in}
\centering
\caption{Comparison of \namel~with existing algorithms for solving KL-DRO on ImageNet-1K classification with ViT-S model.}
\vskip 0.1in
\resizebox{0.6\linewidth}{!}{
\begin{tabular}{lcccc}
\toprule
\textbf{Algorithm} & ERM & TERM & ABSGD & \namel \\
\midrule
Accuracy   & 78.44 \std{0.35}  & 78.03 \std{0.27} & 77.12 \std{0.82} & \textbf{79.11} \std{0.12}   \\
\bottomrule
\end{tabular}}
\label{tab:rgd_vs_kldro}
\vspace{-0.1in}
\end{table}

\begin{table}[hbt!]
\centering
\caption{Comparison of RGD with existing KL-DRO optimization techniques for class imbalanced classification. The numbers represent test accuracy on Long-Tailed CIFAR-10, and CIFAR-100 datasets using ResNet-32. Moreover, we perform hypothesis tests to confirm that these results are statistically significant with a p-value that is less than 0.05.}
\vskip 0.15in
\resizebox{0.8\linewidth}{!}{
\begin{tabular}{lccc|ccc}
\toprule
Dataset&
      \multicolumn{3}{c}{\data{CIFAR-10}}  &
      \multicolumn{3}{c}{\data{CIFAR-100}}   \\ \cline{2-4} \cline{4-7}
      
Loss / Imbalance Factor&
  \data{100} &
  \data{10} &
  \textbf{Avg.} &
  \data{100} &
  \data{10} &
  \textbf{Avg.} \\
\cline{1-7}

\midrule
\multicolumn{7}{l}{Cross Entropy (CE)}\\
\midrule
Default  &   71.75 \std{0.75}  & 87.64 \std{0.45}  & 79.7 & 38.35 \std{38.35} & 56.91  \std{0.41} & 47.63 \\
TERM~\citet{li2020tilted} & 72.20 \std{0.18} & 87.52 \std{0.07} & 79.86 & 39.75 \std{0.04} & 57.91 \std{0.39} & 48.83 \\
ABSGD~\citet{qi2020attentional} & 72.43 \std{0.31} & \textbf{87.93} \std{0.25} & 80.18 & 39.77 \std{0.34} & 57.44\std{0.25} & 48.61 \\
\namee~\textbf{(Ours)}    &   \textbf{73.71} \std{0.04}  & 87.78 \std{0.15} & \textbf{80.75} &  
\textbf{41.47} \std{0.30} &  \textbf{58.49} \std{0.04} & \textbf{49.98} \\

\bottomrule
\end{tabular}}
\label{tab:long_tailed_cifar_kldro_priors}
\end{table}

\subsection{Supervised Learning}
\label{sec:supervised_learning}

This section studies our approach when applied on standard supervised learning tasks such as BERT finetuning on GLUE benchmark, and Imagenet-1K classification. We use a base model of ViT-S for the latter task. Table~\ref{tab:bert_glue} depicts our results from this experiment. On GLUE tasks, our \namel~algorithm outperforms the baseline by \textcolor{blue}{\textbf{+1.94\%}} with a standard deviation of 0.42\%. Furthermore, we perform hypothesis testing to condit that these results are statistically significant with a p-value that is less than 0.05. On Imagenet-1K, we show a \textcolor{blue}{\textbf{+1.01\%}} improvement over baseline with the off-the-hat addition of the RGD reweighing and no additional complexity in terms of compute: memory and time. 

Furthermore, we also experiment with pre-training of the BERT-base model. We use the BooksCorpus (800M words) \citep{zhu2015aligning} and English Wikipedia (2,500M words) as our pre-training corpus. We trained the Bert-base model for 450K steps, and tuned the learning rate (lr) for baseline, and lr, clipping factor for \namel. We report both the MLM (Masked Language Model) accuracy and NSP (Next Sequence Prediction) accuracy comparisons of RGD vs Default (ERM). It can be seen that our approach boosts the MLM accuracy and NSP accuracy by \textcolor{blue}{\textbf{+0.2\%}} and \textcolor{blue}{\textbf{+0.9\%}} respectively (see Table~\ref{tab:bert_glue}). Furthermore, through hypothesis testing, we show that the results are statistically significant with a p-value of 0.05. Additional experiments on EfficientNet fine-tuning, DeiT model~\citep{touvron2021training} for ImageNet-1K classification, and miniGPT~\citep{zhu2023minigpt} pre-training are discussed in Appendix~\ref{appendix:supervised_learning}. 

\begin{table*}[hbt!]
\centering
\caption{Performance of RGD for various Supervised Learning tasks. }
\vskip 0.15in

\resizebox{1.0\linewidth}{!}{
\begin{tabular}{ccccccccc|c|cc}
\toprule
 & \data{\textbf{MNLI}}           & \data{\textbf{QQP}}            & \data{\textbf{QNLI}}           & \data{\textbf{SST-2}}          & \data{\textbf{MRPC}}           & \data{\textbf{RTE}}            & \data{\textbf{COLA}}           & \begin{tabular}{c}\data{\textbf{Avg on}}\\\data{\textbf{GLUE}}\end{tabular} & \data{\textbf{ImageNet-1K}} & \begin{tabular}{c}\data{\textbf{BERT Pretraining}}\\ \data{(MLM)}\end{tabular} & \begin{tabular}{c}\data{\textbf{BERT Pretraining} }\\\data{(NSP)}\end{tabular}  \\
\midrule
Default           & 81.33          & 89.62          & 87.93          & 90.63          & \textbf{89.55} & 67.19          & 54.53          & 80.11 &  78.44 \std{0.35} & 71.31 & 98.01        \\
\namee~\textbf{(Ours)}           & \textbf{83.06} & \textbf{91.06} & \textbf{90.35} & \textbf{91.78} & 88.28          & \textbf{71.48} & \textbf{58.56} & \textbf{82.05}  & \textbf{79.11}	\std{0.12} & \textbf{71.35} \std{0.11} & \textbf{98.92} \std{0.09} \\

\bottomrule
\end{tabular}}
\label{tab:bert_glue}
\vspace{-0.1in}
\end{table*}

\subsection{Tabular Classification}

\begin{table}[hbt!]
\centering
\vspace{-0.15in}
\caption{Results on standard binary-class tabular datasets (AUROC): The bottom partition shows results of our method with \namel\ loss. We show that the addition of our proposed approach significantly outperforms existing methods, as well as SOTA.}
\vskip 0.15in
\Huge
\resizebox{0.7\linewidth}{!}{
\begin{tabular}{lccccc}
\toprule
\textbf{Algorithm} &
 \data{Obesity} &
  \data{Income} &
  \data{Criteo} &
  \data{Thyroid} &
  \textbf{Avg.} \\
\midrule
MLP      & 52.3      & 89.39       & 79.82              & 62.3 & 70.95 \\
RF \cite{breiman2001random}     & 64.36    & 91.53        & 77.57             & 99.62 & 83.27 \\
 \midrule
\multicolumn{6}{l}{MET-S}\\
\midrule
 Default \cite{majmundar2022met} & 71.84    & 93.85 & 86.17    & 99.81 & 87.92 \\
\namee~\textbf{(Ours)} &  \textbf{76.87} &  \textbf{93.96}       &  \textbf{86.98}  &  \textbf{99.92} & \textbf{89.43} \\
\bottomrule
\end{tabular}}
\label{tab:met_binary_auroc}
\end{table}

\begin{table}[hbt!]
\centering
\caption{Results on standard multi-class tabular datasets (Accuracy): The bottom partition shows results of our method with \namee\ loss. We show that the addition of our proposed approach significantly outperforms existing methods, as well as SOTA.}
\vskip 0.15in
\Huge
\resizebox{0.7\linewidth}{!}{
\begin{tabular}{lccccc}
\toprule
\textbf{Algorithm} &
 \data{FMNIST} &
  \data{CIFAR10} &
  \data{MNIST} &
  \data{CovType} &
  \textbf{Avg.} \\
\midrule
MLP      & 87.62       & 16.50          & 96.95         & 65.47       & 66.64  \\
RF \cite{breiman2001random}     & 88.43       & 42.73          & 97.62         & 71.37       & 75.04  \\
MET \cite{majmundar2022met} & \textbf{91.68} & 47.82 & 99.19 & 76.71 & 78.85\\
 \midrule
\multicolumn{6}{l}{MET-S}\\
\midrule
Default \cite{majmundar2022met}   & 90.94       & 48.00          & 99.01         & 74.11       & 78.02 \\
\namee~\textbf{(Ours)} &  91.54 &      \textbf{49.54}   &   \textbf{99.69} &  \textbf{79.72}  &  \textbf{80.12} \\
\bottomrule
\end{tabular}}
\label{tab:met_multi_label}

\end{table}

Learning with tabular data is a task where traditional machine learning methods, like random forest \cite{breiman2001random, friedman2001greedy} are incredibly competitive against deep learning-based methods \citep{yoon2020vime}. 
Recently, \citet{majmundar2022met} obtained SOTA results for tabular classification using self-supervised representation learning and relying on the learned representations in the downstream classification tasks (see Section~\ref{sec:intro}).
Their work developed two algorithms namely, MET (representation learning with adversarial training) and MET-S (representation learning without adversarial training). The adversarial training adds robustness to the learned representations, thus improving performance. In this experiment, we integrate \namel~with MET-S instead of doing adversarial training. This allows us to test the robustness properties of the models trained with \namel. Table \ref{tab:met_multi_label} and Table~\ref{tab:met_binary_auroc} shows gains on multiple tabular datasets for the multi-class classification and binary classification tasks. Notably, our approach outperforms previous SOTA in this problem by \textcolor{blue}{\textbf{+1.27\%}}, and \textcolor{blue}{\textbf{+1.5\%}} on the multi-class and binary classification tasks respectively. We refer to Appendix~\ref{tabular_additional_results} for a comprehensive comparison with baselines such as Gradient Boosting Decision Trees \citep{friedman2001greedy}, VIME \citep{yoon2020vime}, SubTab \citep{ucar2021subtab}, TabNet \citep{arik2019tabnet}, DACL+ \citep{verma2021towards} and many more.
Our motivation to experiment on these “permuted” MNIST, “permuted” CIFAR, and “permuted” FMNIST can be traced back to the introduction of these datasets in the works of \cite{yoon2020vime, ucar2021subtab}. Subsequently, other recent works such as \cite{majmundar2022met} also experimented on these datasets and have become a standard benchmark for tabular classification.

\subsection{Out Of Domain Generalization}

\begin{table*}[hbt!]
\caption{Results on DomainBed (Model selection: training-domain validation set): The bottom partition shows results of our method with \namel\ loss. In both cases, with (top) and without (bottom) fixed linear layer, the proposed approach outperforms existing methods, as well as SOTA.}
\vskip 0.15in
\centering
\Huge
\resizebox{0.7\linewidth}{!}{
\begin{tabular}{lccccl}
\toprule
\textbf{Algorithm} &
 \data{PACS} &
  \data{VLCS} &
  \data{OfficeHome} &
  \data{DomainNet} &
  \textbf{Avg.} \\
\midrule
ERM   \cite{gulrajani2020search}   & 85.5   \stdlarge{0.1}       & 77.5 \stdlarge{0.4}          & 66.5 \stdlarge{0.2}         & 40.9 \stdlarge{0.1}               & 67.6   \\
MIRO \cite{cha2022domain} & 85.4 \stdlarge{0.4}          & {\textbf{79.0} \stdlarge{0.0}} & {70.5 \stdlarge{0.4}}  & {44.3 \stdlarge{0.2}} & {69.8}  \\
\midrule
\multicolumn{6}{l}{ERM + FRR-L }\\
\midrule
Default \cite{addepalli2022learning} & 85.7 \stdlarge{0.1}  & 76.6 \stdlarge{0.2} & 68.4 \stdlarge{0.2} & 44.2 \stdlarge{0.1} &  68.73  \\
\namee~\textbf{(Ours)} &  87.6 \stdlarge{0.3} & 78.6 \stdlarge{0.3} & 69.8 \stdlarge{0.2} & \textbf{46.00} \stdlarge{0.0} & 70.48 \\
 \midrule
\multicolumn{6}{l}{ERM + FRR}\\
\midrule
 Default  \cite{addepalli2022learning} & 87.5 \stdlarge{0.1} & 77.6 \stdlarge{0.3} & 69.4 \stdlarge{0.1}  & 45.1 \stdlarge{0.1}  & 69.90
 \\
\namee~\textbf{(Ours)} & \textbf{88.2} \stdlarge{0.2} & 78.6 \stdlarge{0.3} & 69.8 \stdlarge{0.2}& 45.8 \stdlarge{0.0} & \textbf{70.60}\\
\bottomrule
\end{tabular}}
\label{tab:domain_bed}
\vspace{-0.2in}
\end{table*}

In this section, we show that our technique can be used to boost the performance of OOD generalization techniques. We experiment on DomainBed, a standard benchmark used to study the out-of-domain performance of models. More information about the benchmark, the task to solve, and the metric is discussed in Appendix~\ref{domainbed_dataset}.
The benchmark is notorious since the most basic approach, such as straightforward Empirical Risk Minimization (ERM) as evaluated by \cite{gulrajani2020search}, was the SOTA method for a long time. Most new approaches either performed worse than ERM or marginally better. In recent years, breakthroughs such as MIRO \citep{cha2022domain} and FRR \citep{addepalli2022learning} have pushed the problem further by significantly improving the benchmarks. We show that integrating our proposed approach \namel~with these approaches (specifically FRR) significantly improves performance (an average of \textcolor{blue}{\textbf{+0.7\%}}). Table~\ref{tab:domain_bed} illustrates the accuracy performance numbers of a few baseline methods and our proposed approach. A more comprehensive comparison with additional baselines such as IRM \citep{arjovsky2019invariant}, CORAL \citep{sun2016deep}, MTL \citep{blanchard2021domain}, SagNet \citep{nam2021reducing}, and many more in the Table~\ref{tab:domain_bed_additional}. We further depict the environment-wise breakdown of the accuracy of each of the baseline algorithms in Appendix~\ref{domain_bed_additional_results}.

\subsection{Meta-Learning}
\begin{table*}[hbt!]
\caption{Results on meta-learning datasets. We report the Worst-K\% performance as well to help study the performance distribution over all tasks.}
\vskip 0.15in
\centering
\Huge
\resizebox{0.7\linewidth}{!}{
\begin{tabular}{lcccccc}
\toprule
\textbf{Algorithm} &
 \data{Worst 10\%} &
 \data{Worst 20\%} &
  \data{Worst 30\%} &
  \data{Worst 40\%} &
  \data{Worst 50\%} &
  \data{Overall} \\
\midrule
\multicolumn{7}{l}{Omniglot 5-way 1-shot}\\
\midrule
MAML  & 91.71 \stdlarge{0.73}	& 94.16 \stdlarge{0.50}	& 95.41 \stdlarge{0.39} &	96.22 \stdlarge{0.32} &	96.76 \stdlarge{0.27} & 98.38 \stdlarge{0.17} \\
MAML + \namee  & \textbf{92.14} \stdlarge{0.84} &	\textbf{94.54} \stdlarge{0.53} &	\textbf{95.72} \stdlarge{0.40} &	\textbf{96.46} \stdlarge{0.33} &	\textbf{96.90} \stdlarge{0.27} &  \textbf{98.45} \stdlarge{0.17} \\
\midrule
\multicolumn{7}{l}{Omniglot 20-way 1-shot}\\
\midrule
MAML  & 84.33 \stdlarge{0.40} &	85.86 \stdlarge{0.29} &	86.92 \stdlarge{0.26} &	87.73 \stdlarge{0.24} &	88.42 \stdlarge{0.22} & 91.28 \stdlarge{0.22} \\
MAML + \namee  & \textbf{86.61} \stdlarge{0.36} &	\textbf{88.09} \stdlarge{0.28} &	\textbf{89.09} \stdlarge{0.24} &	\textbf{89.87} \stdlarge{0.23} &	\textbf{90.50} \stdlarge{0.21} & \textbf{93.01} \stdlarge{0.20} \\

\midrule
\multicolumn{7}{l}{\textit{mini}ImageNet 5-way 1-shot}\\
\midrule
MAML  & 30.94 \stdlarge{0.70} &	34.52 \stdlarge{0.62} &	36.93 \stdlarge{0.57} &	38.94 \stdlarge{0.55} &	40.68 \stdlarge{0.53} & 48.86 \stdlarge{0.62} \\
MAML + \namee  & \textbf{33.33} \stdlarge{0.90} &	\textbf{36.67} \stdlarge{0.65} &	\textbf{39.12} \stdlarge{0.59} &	\textbf{41.20} \stdlarge{0.56} &	\textbf{42.96} \stdlarge{0.55} & \textbf{51.21} \stdlarge{0.63}\\

\bottomrule
\end{tabular}}
\label{tab:meta_learning_subset}
\end{table*}

In meta-learning, the goal is to learn representations that generalize effectively to new tasks, even when provided with limited examples. However, task heterogeneity poses a significant challenge.  Some tasks may be inherently simpler to learn, leading models to prioritize these and neglect the more difficult, less frequent tasks. While Empirical Risk Minimization (ERM) may perform well on common tasks, its performance can deteriorate drastically on rare and challenging ones. This necessitates a mechanism for re-weighting tasks to ensure balanced learning.
Building upon the experimental results of \citet{kumar2022effect}, we make comparisons with our MAML + \namee~approach as the proposed variant. We evaluate \namel~not only based on the average performance across tasks but also on the Worst-K\% of tasks in a fixed task pool. Our experiments on various benchmarks, including Omniglot 5-way 1-shot, Omniglot 20-way 1-shot, and \textit{mini}ImageNet 5-way 1-shot, demonstrate significant improvements in the Worst-K\% metric (Table~\ref{tab:meta_learning_subset}). For example, on Omniglot 20-way 1-shot, our proposed reweighting scheme improves overall performance by \textcolor{blue}{\textbf{1.83\%}} and the Worst-10\% performance by \textcolor{blue}{\textbf{2.28\%}}. Similarly, on the challenging \textit{mini}ImageNet 5-way 1-shot benchmark, we achieve a substantial improvement of approximately \textcolor{blue}{\textbf{3\%}} across the board. Further results in this domain are discussed in Appendix~\ref{appendix:meta_learning}.

\section{Conclusion, Limitations and Future Work}

We introduced a re-weighted gradient descent (\namel) technique that effectively boosts the performance of deep learning across a wide range of tasks and domains. It is simple to implement and can be seamlessly integrated into existing algorithms with just two lines of code change. Our algorithm is derived from Kullback-Leibler distributionally robust optimization, a known method for improving model generalization. 

While RGD shows promising results, it has the following limitations that warrant further investigation:  (a) \textit{outlier handling:} our current approach uses weight clipping to mitigate the impact of outliers. While empirically effective, a more principled approach to outlier robustness in DRO is worth exploring, and (b)
\textit{robustness to noise:} while RGD can handle benign noise, it can fail in the presence of adversarial/systematic noise.  In the future, we plan to develop variants of \namel~that can tolerate adversarial corruptions in the training data, while simultaneously improving the model generalization. Additionally, we plan to evaluate our technique on large-scale tasks, such as fine-tuning Large Language Models (LLMs) and other foundation models. This will help us  better understand the usefulness and limitations of our approach.

\subsubsection*{Ethical Statement and Broader Impact}

Our proposed approach is compatible with any learning objective expressed as an expectation over samples. We showcased its effectiveness with various loss functions, including Mean Square Error, Cross Entropy, and others, outperforming previous state-of-the-art methods considerably. Implementing our approach is straightforward, and it has broad applicability across domains such as Natural Language Processing (NLP), Vision, and Time Series data. This paper presents work whose goal is to advance the field of Machine Learning. There are many potential societal consequences of our work, none which we feel must be specifically highlighted here.

\subsubsection*{Acknowledgments}
We would like to thank Ahmad Beirami, Virginia Smith,
Manzil Zaheer, Tian Li, and Maziar Sanjabi for providing
detailed feedback on our paper. We would like to thank Elan
Rosenfeld, and Tianbao Yang for pointing us to important
prior works. We extend our sincere gratitude to Prateek
Jain, Pradeep Shenoy, Anshul Nasery, Lovish Madaan, and
the numerous dedicated members of the machine learning
and optimization team at Google DeepMind India for their
invaluable feedback and contributions to this work.

\bibliography{main.bib}

\begin{thebibliography}{116}
\providecommand{\natexlab}[1]{#1}
\providecommand{\url}[1]{\texttt{#1}}
\expandafter\ifx\csname urlstyle\endcsname\relax
  \providecommand{\doi}[1]{doi: #1}\else
  \providecommand{\doi}{doi: \begingroup \urlstyle{rm}\Url}\fi

\bibitem[Addepalli et~al.(2023)Addepalli, Nasery, Radhakrishnan, Netrapalli,
  and Jain]{addepalli2022learning}
Sravanti Addepalli, Anshul Nasery, Venkatesh~Babu Radhakrishnan, Praneeth
  Netrapalli, and Prateek Jain.
\newblock Feature reconstruction from outputs can mitigate simplicity bias in
  neural networks.
\newblock In \emph{The Eleventh International Conference on Learning
  Representations}, 2023.
\newblock URL \url{https://openreview.net/forum?id=zH9GcZ3ZGXu}.

\bibitem[Arik \& Pfister(2021)Arik and Pfister]{arik2019tabnet}
Sercan~{\"O} Arik and Tomas Pfister.
\newblock Tabnet: Attentive interpretable tabular learning.
\newblock In \emph{Proceedings of the AAAI conference on artificial
  intelligence}, volume~35, pp.\  6679--6687, 2021.

\bibitem[Arjovsky et~al.(2019)Arjovsky, Bottou, Gulrajani, and
  Lopez-Paz]{arjovsky2019invariant}
Martin Arjovsky, L{\'e}on Bottou, Ishaan Gulrajani, and David Lopez-Paz.
\newblock Invariant risk minimization.
\newblock \emph{arXiv preprint arXiv:1907.02893}, 2019.

\bibitem[Arora et~al.(2012)Arora, Hazan, and Kale]{arora2012multiplicative}
Sanjeev Arora, Elad Hazan, and Satyen Kale.
\newblock The multiplicative weights update method: a meta-algorithm and
  applications.
\newblock \emph{Theory of computing}, 8\penalty0 (1):\penalty0 121--164, 2012.

\bibitem[Ben-Tal et~al.(2009)Ben-Tal, El~Ghaoui, and Nemirovski]{ben2009robust}
Aharon Ben-Tal, Laurent El~Ghaoui, and Arkadi Nemirovski.
\newblock \emph{Robust optimization}, volume~28.
\newblock Princeton university press, 2009.

\bibitem[Ben-Tal et~al.(2013)Ben-Tal, Den~Hertog, De~Waegenaere, Melenberg, and
  Rennen]{ben2013robust}
Aharon Ben-Tal, Dick Den~Hertog, Anja De~Waegenaere, Bertrand Melenberg, and
  Gijs Rennen.
\newblock Robust solutions of optimization problems affected by uncertain
  probabilities.
\newblock \emph{Management Science}, 59\penalty0 (2):\penalty0 341--357, 2013.

\bibitem[Bengio et~al.(2009)Bengio, Louradour, Collobert, and
  Weston]{bengio2009curriculum}
Yoshua Bengio, J{\'e}r{\^o}me Louradour, Ronan Collobert, and Jason Weston.
\newblock Curriculum learning.
\newblock In \emph{Proceedings of the 26th annual international conference on
  machine learning}, pp.\  41--48, 2009.

\bibitem[Blanchard et~al.(2021)Blanchard, Deshmukh, Dogan, Lee, and
  Scott]{blanchard2021domain}
Gilles Blanchard, Aniket~Anand Deshmukh, {\"U}run Dogan, Gyemin Lee, and
  Clayton Scott.
\newblock Domain generalization by marginal transfer learning.
\newblock \emph{The Journal of Machine Learning Research}, 22\penalty0
  (1):\penalty0 46--100, 2021.

\bibitem[Bradbury et~al.(2018)Bradbury, Frostig, Hawkins, Johnson, Leary,
  Maclaurin, Necula, Paszke, Vander{P}las, Wanderman-{M}ilne, and
  Zhang]{jax2018github}
James Bradbury, Roy Frostig, Peter Hawkins, Matthew~James Johnson, Chris Leary,
  Dougal Maclaurin, George Necula, Adam Paszke, Jake Vander{P}las, Skye
  Wanderman-{M}ilne, and Qiao Zhang.
\newblock {JAX}: composable transformations of {P}ython+{N}um{P}y programs,
  2018.
\newblock URL \url{http://github.com/google/jax}.

\bibitem[Braun et~al.(2017)Braun, Neil, and Liu]{braun2017curriculum}
Stefan Braun, Daniel Neil, and Shih-Chii Liu.
\newblock A curriculum learning method for improved noise robustness in
  automatic speech recognition.
\newblock In \emph{2017 25th European Signal Processing Conference (EUSIPCO)},
  pp.\  548--552. IEEE, 2017.

\bibitem[Breiman(2001)]{breiman2001random}
Leo Breiman.
\newblock Random forests.
\newblock \emph{Machine learning}, 45\penalty0 (1):\penalty0 5--32, 2001.

\bibitem[Castells et~al.(2020)Castells, Weinzaepfel, and
  Revaud]{castells2020superloss}
Thibault Castells, Philippe Weinzaepfel, and Jerome Revaud.
\newblock Superloss: A generic loss for robust curriculum learning.
\newblock \emph{Advances in Neural Information Processing Systems},
  33:\penalty0 4308--4319, 2020.

\bibitem[Catoni \& Giulini(2018)Catoni and Giulini]{catoni2018dimension}
Olivier Catoni and Ilaria Giulini.
\newblock Dimension-free pac-bayesian bounds for the estimation of the mean of
  a random vector.
\newblock \emph{arXiv preprint arXiv:1802.04308}, 2018.

\bibitem[Cha et~al.(2021)Cha, Chun, Lee, Cho, Park, Lee, and Park]{cha2021swad}
Junbum Cha, Sanghyuk Chun, Kyungjae Lee, Han-Cheol Cho, Seunghyun Park, Yunsung
  Lee, and Sungrae Park.
\newblock Swad: Domain generalization by seeking flat minima.
\newblock \emph{Advances in Neural Information Processing Systems},
  34:\penalty0 22405--22418, 2021.

\bibitem[Cha et~al.(2022)Cha, Lee, Park, and Chun]{cha2022domain}
Junbum Cha, Kyungjae Lee, Sungrae Park, and Sanghyuk Chun.
\newblock Domain generalization by mutual-information regularization with
  pre-trained models.
\newblock In \emph{European conference on computer vision}, pp.\  440--457.
  Springer, 2022.

\bibitem[Chawla et~al.(2002)Chawla, Bowyer, Hall, and
  Kegelmeyer]{chawla2002smote}
Nitesh~V Chawla, Kevin~W Bowyer, Lawrence~O Hall, and W~Philip Kegelmeyer.
\newblock Smote: synthetic minority over-sampling technique.
\newblock \emph{Journal of artificial intelligence research}, 16:\penalty0
  321--357, 2002.

\bibitem[Chen \& Gupta(2015)Chen and Gupta]{chen2015webly}
Xinlei Chen and Abhinav Gupta.
\newblock Webly supervised learning of convolutional networks.
\newblock In \emph{Proceedings of the IEEE international conference on computer
  vision}, pp.\  1431--1439, 2015.

\bibitem[Collins et~al.(2020)Collins, Mokhtari, and
  Shakkottai]{collins2020task}
Liam Collins, Aryan Mokhtari, and Sanjay Shakkottai.
\newblock Task-robust model-agnostic meta-learning.
\newblock \emph{Advances in Neural Information Processing Systems},
  33:\penalty0 18860--18871, 2020.

\bibitem[Cui et~al.(2019)Cui, Jia, Lin, Song, and Belongie]{cui2019class}
Yin Cui, Menglin Jia, Tsung-Yi Lin, Yang Song, and Serge Belongie.
\newblock Class-balanced loss based on effective number of samples.
\newblock In \emph{Proceedings of the IEEE/CVF conference on computer vision
  and pattern recognition}, pp.\  9268--9277, 2019.

\bibitem[De~La~Torre \& Black(2003)De~La~Torre and Black]{de2003framework}
Fernando De~La~Torre and Michael~J Black.
\newblock A framework for robust subspace learning.
\newblock \emph{International Journal of Computer Vision}, 54\penalty0
  (1):\penalty0 117--142, 2003.

\bibitem[Duchi \& Namkoong(2018)Duchi and Namkoong]{duchi2018learning}
John Duchi and Hongseok Namkoong.
\newblock Learning models with uniform performance via distributionally robust
  optimization.
\newblock \emph{arXiv preprint arXiv:1810.08750}, 2018.

\bibitem[Duchi et~al.(2011)Duchi, Hazan, and Singer]{duchi2011adaptive}
John Duchi, Elad Hazan, and Yoram Singer.
\newblock Adaptive subgradient methods for online learning and stochastic
  optimization.
\newblock \emph{Journal of machine learning research}, 12\penalty0 (7), 2011.

\bibitem[Duchi et~al.(2021{\natexlab{a}})Duchi, Glynn, and
  Namkoong]{duchi2016statistics}
John~C Duchi, Peter~W Glynn, and Hongseok Namkoong.
\newblock Statistics of robust optimization: A generalized empirical likelihood
  approach.
\newblock \emph{Mathematics of Operations Research}, 46\penalty0 (3):\penalty0
  946--969, 2021{\natexlab{a}}.

\bibitem[Duchi et~al.(2021{\natexlab{b}})Duchi, Glynn, and
  Namkoong]{duchi2021statistics}
John~C Duchi, Peter~W Glynn, and Hongseok Namkoong.
\newblock Statistics of robust optimization: A generalized empirical likelihood
  approach.
\newblock \emph{Mathematics of Operations Research}, 46\penalty0 (3):\penalty0
  946--969, 2021{\natexlab{b}}.

\bibitem[El~Hanchi et~al.(2022)El~Hanchi, Stephens, and
  Maddison]{el2022stochastic}
Ayoub El~Hanchi, David Stephens, and Chris Maddison.
\newblock Stochastic reweighted gradient descent.
\newblock In \emph{International Conference on Machine Learning}, pp.\
  8359--8374. PMLR, 2022.

\bibitem[Fan et~al.(2017)Fan, He, Liang, and Hu]{fan2017self}
Yanbo Fan, Ran He, Jian Liang, and Baogang Hu.
\newblock Self-paced learning: An implicit regularization perspective.
\newblock In \emph{Proceedings of the AAAI Conference on Artificial
  Intelligence}, volume~31, 2017.

\bibitem[Faury et~al.(2020)Faury, Tanielian, Dohmatob, Smirnova, and
  Vasile]{faury2020distributionally}
Louis Faury, Ugo Tanielian, Elvis Dohmatob, Elena Smirnova, and Flavian Vasile.
\newblock Distributionally robust counterfactual risk minimization.
\newblock In \emph{Proceedings of the AAAI Conference on Artificial
  Intelligence}, volume~34, pp.\  3850--3857, 2020.

\bibitem[Fidon et~al.(2020)Fidon, Aertsen, Deprest, Emam, Guffens, Mufti,
  Van~Elslander, Schwartz, Ebner, Prayer, et~al.]{fidon2020distributionally}
Lucas Fidon, Michael Aertsen, Thomas Deprest, Doaa Emam, Fr{\'e}d{\'e}ric
  Guffens, Nada Mufti, Esther Van~Elslander, Ernst Schwartz, Michael Ebner,
  Daniela Prayer, et~al.
\newblock Distributionally robust deep learning using hardness weighted
  sampling.
\newblock \emph{arXiv preprint arXiv:2001.02658}, 2020.

\bibitem[Finn et~al.(2017)Finn, Abbeel, and Levine]{finn2017model}
Chelsea Finn, Pieter Abbeel, and Sergey Levine.
\newblock Model-agnostic meta-learning for fast adaptation of deep networks.
\newblock In \emph{International conference on machine learning}, pp.\
  1126--1135. PMLR, 2017.

\bibitem[Foret et~al.(2021)Foret, Kleiner, Mobahi, and
  Neyshabur]{foretsharpness}
Pierre Foret, Ariel Kleiner, Hossein Mobahi, and Behnam Neyshabur.
\newblock Sharpness-aware minimization for efficiently improving
  generalization.
\newblock In \emph{9th International Conference on Learning Representations,
  {ICLR} 2021, Virtual Event, Austria, May 3-7, 2021}. OpenReview.net, 2021.
\newblock URL \url{https://openreview.net/forum?id=6Tm1mposlrM}.

\bibitem[Freund \& Schapire(1997)Freund and Schapire]{freund1997decision}
Yoav Freund and Robert~E Schapire.
\newblock A decision-theoretic generalization of on-line learning and an
  application to boosting.
\newblock \emph{Journal of computer and system sciences}, 55\penalty0
  (1):\penalty0 119--139, 1997.

\bibitem[Friedman(2001)]{friedman2001greedy}
Jerome~H Friedman.
\newblock Greedy function approximation: a gradient boosting machine.
\newblock \emph{Annals of statistics}, pp.\  1189--1232, 2001.

\bibitem[Ganin et~al.(2016)Ganin, Ustinova, Ajakan, Germain, Larochelle,
  Laviolette, Marchand, and Lempitsky]{ganin2016domain}
Yaroslav Ganin, Evgeniya Ustinova, Hana Ajakan, Pascal Germain, Hugo
  Larochelle, Fran{\c{c}}ois Laviolette, Mario Marchand, and Victor Lempitsky.
\newblock Domain-adversarial training of neural networks.
\newblock \emph{The journal of machine learning research}, 17\penalty0
  (1):\penalty0 2096--2030, 2016.

\bibitem[Gonzalez \& Miikkulainen(2021)Gonzalez and
  Miikkulainen]{gonzalez2021optimizing}
Santiago Gonzalez and Risto Miikkulainen.
\newblock Optimizing loss functions through multi-variate taylor polynomial
  parameterization.
\newblock In \emph{Proceedings of the Genetic and Evolutionary Computation
  Conference}, pp.\  305--313, 2021.

\bibitem[Gulrajani \& Lopez-Paz(2021)Gulrajani and
  Lopez-Paz]{gulrajani2020search}
Ishaan Gulrajani and David Lopez-Paz.
\newblock In search of lost domain generalization.
\newblock In \emph{International Conference on Learning Representations}, 2021.
\newblock URL \url{https://openreview.net/forum?id=lQdXeXDoWtI}.

\bibitem[Hsieh et~al.(2019)Hsieh, Liu, and Cevher]{hsieh2019finding}
Ya-Ping Hsieh, Chen Liu, and Volkan Cevher.
\newblock Finding mixed nash equilibria of generative adversarial networks.
\newblock In \emph{International Conference on Machine Learning}, pp.\
  2810--2819. PMLR, 2019.

\bibitem[Hu \& Hong(2012)Hu and Hong]{hu2013kullback}
Zhaolin Hu and Jeff~Liu Hong.
\newblock Kullback-leibler divergence constrained distributionally robust
  optimization.
\newblock \emph{arXiv preprint arXiv:2001.02658}, 2012.
\newblock URL \url{https://api.semanticscholar.org/CorpusID:15049657}.

\bibitem[Huang et~al.(2020)Huang, Wang, Xing, and Huang]{huang2020self}
Zeyi Huang, Haohan Wang, Eric~P Xing, and Dong Huang.
\newblock Self-challenging improves cross-domain generalization.
\newblock In \emph{European Conference on Computer Vision}, pp.\  124--140.
  Springer, 2020.

\bibitem[Ivanova \& Ablin(2023)Ivanova and Ablin]{ivanova2023challenge}
Anastasia Ivanova and Pierre Ablin.
\newblock A challenge in reweighting data with bilevel optimization.
\newblock \emph{arXiv preprint arXiv:2310.17386}, 2023.

\bibitem[Jesson et~al.(2017)Jesson, Guizard, Ghalehjegh, Goblot, Soudan, and
  Chapados]{jesson2017cased}
Andrew Jesson, Nicolas Guizard, Sina~Hamidi Ghalehjegh, Damien Goblot, Florian
  Soudan, and Nicolas Chapados.
\newblock Cased: curriculum adaptive sampling for extreme data imbalance.
\newblock In \emph{International conference on medical image computing and
  computer-assisted intervention}, pp.\  639--646. Springer, 2017.

\bibitem[Jiang et~al.(2014{\natexlab{a}})Jiang, Meng, Mitamura, and
  Hauptmann]{jiang2014easy}
Lu~Jiang, Deyu Meng, Teruko Mitamura, and Alexander~G Hauptmann.
\newblock Easy samples first: Self-paced reranking for zero-example multimedia
  search.
\newblock In \emph{Proceedings of the 22nd ACM international conference on
  Multimedia}, pp.\  547--556, 2014{\natexlab{a}}.

\bibitem[Jiang et~al.(2014{\natexlab{b}})Jiang, Meng, Yu, Lan, Shan, and
  Hauptmann]{jiang2014self}
Lu~Jiang, Deyu Meng, Shoou-I Yu, Zhenzhong Lan, Shiguang Shan, and Alexander
  Hauptmann.
\newblock Self-paced learning with diversity.
\newblock \emph{Advances in neural information processing systems}, 27,
  2014{\natexlab{b}}.

\bibitem[Katharopoulos \& Fleuret(2018)Katharopoulos and
  Fleuret]{katharopoulos2018not}
Angelos Katharopoulos and Fran{\c{c}}ois Fleuret.
\newblock Not all samples are created equal: Deep learning with importance
  sampling.
\newblock In \emph{International conference on machine learning}, pp.\
  2525--2534. PMLR, 2018.

\bibitem[Kingma \& Ba(2015)Kingma and Ba]{kingma2014adam}
Diederik Kingma and Jimmy Ba.
\newblock Adam: A method for stochastic optimization.
\newblock In \emph{International Conference on Learning Representations
  (ICLR)}, San Diega, CA, USA, 2015.

\bibitem[Koloskova et~al.(2023)Koloskova, Hendrikx, and
  Stich]{koloskova2023revisiting}
Anastasia Koloskova, Hadrien Hendrikx, and Sebastian~U Stich.
\newblock Revisiting gradient clipping: Stochastic bias and tight convergence
  guarantees.
\newblock In \emph{International Conference on Machine Learning}, pp.\
  17343--17363. PMLR, 2023.

\bibitem[Krueger et~al.(2021)Krueger, Caballero, Jacobsen, Zhang, Binas, Zhang,
  Le~Priol, and Courville]{krueger2021out}
David Krueger, Ethan Caballero, Joern-Henrik Jacobsen, Amy Zhang, Jonathan
  Binas, Dinghuai Zhang, Remi Le~Priol, and Aaron Courville.
\newblock Out-of-distribution generalization via risk extrapolation (rex).
\newblock In \emph{International Conference on Machine Learning}, pp.\
  5815--5826. PMLR, 2021.

\bibitem[Kumar \& Amid(2021)Kumar and Amid]{kumar2021constrained}
Abhishek Kumar and Ehsan Amid.
\newblock Constrained instance and class reweighting for robust learning under
  label noise.
\newblock \emph{arXiv preprint arXiv:2111.05428}, 2021.

\bibitem[Kumar et~al.(2010)Kumar, Packer, and Koller]{kumar2010self}
M~Kumar, Benjamin Packer, and Daphne Koller.
\newblock Self-paced learning for latent variable models.
\newblock \emph{Advances in neural information processing systems}, 23, 2010.

\bibitem[Kumar et~al.(2011)Kumar, Turki, Preston, and
  Koller]{kumar2011learning}
M~Pawan Kumar, Haithem Turki, Dan Preston, and Daphne Koller.
\newblock Learning specific-class segmentation from diverse data.
\newblock In \emph{2011 International conference on computer vision}, pp.\
  1800--1807. IEEE, 2011.

\bibitem[Kumar et~al.(2023)Kumar, Deleu, and Bengio]{kumar2022effect}
Ramnath Kumar, Tristan Deleu, and Yoshua Bengio.
\newblock The effect of diversity in meta-learning.
\newblock In \emph{Proceedings of the AAAI Conference on Artificial
  Intelligence}, volume~37, pp.\  8396--8404, 2023.

\bibitem[Lam(2016)]{lam2016robust}
Henry Lam.
\newblock Robust sensitivity analysis for stochastic systems.
\newblock \emph{Mathematics of Operations Research}, 41\penalty0 (4):\penalty0
  1248--1275, 2016.

\bibitem[Lam(2019)]{lam2019recovering}
Henry Lam.
\newblock Recovering best statistical guarantees via the empirical
  divergence-based distributionally robust optimization.
\newblock \emph{Operations Research}, 67\penalty0 (4):\penalty0 1090--1105,
  2019.

\bibitem[Le et~al.(2011)Le, Ngiam, Coates, Lahiri, Prochnow, and
  Ng]{le2011optimization}
Quoc~V Le, Jiquan Ngiam, Adam Coates, Abhik Lahiri, Bobby Prochnow, and
  Andrew~Y Ng.
\newblock On optimization methods for deep learning.
\newblock In \emph{Proceedings of the 28th international conference on
  international conference on machine learning}, pp.\  265--272, 2011.

\bibitem[Lee \& Grauman(2011)Lee and Grauman]{lee2011learning}
Yong~Jae Lee and Kristen Grauman.
\newblock Learning the easy things first: Self-paced visual category discovery.
\newblock In \emph{CVPR 2011}, pp.\  1721--1728. IEEE, 2011.

\bibitem[Leng et~al.(2022)Leng, Tan, Liu, Cubuk, Shi, Cheng, and
  Anguelov]{leng2022polyloss}
Zhaoqi Leng, Mingxing Tan, Chenxi Liu, Ekin~Dogus Cubuk, Jay Shi, Shuyang
  Cheng, and Dragomir Anguelov.
\newblock Polyloss: {A} polynomial expansion perspective of classification loss
  functions.
\newblock In \emph{The Tenth International Conference on Learning
  Representations, {ICLR} 2022, Virtual Event, April 25-29, 2022}.
  OpenReview.net, 2022.
\newblock URL \url{https://openreview.net/forum?id=gSdSJoenupI}.

\bibitem[Li et~al.(2017)Li, Yan, Wei, Dong, Liu, and Zha]{li2017self}
Changsheng Li, Junchi Yan, Fan Wei, Weishan Dong, Qingshan Liu, and Hongyuan
  Zha.
\newblock Self-paced multi-task learning.
\newblock In \emph{Proceedings of the AAAI conference on artificial
  intelligence}, volume~31, 2017.

\bibitem[Li et~al.(2018{\natexlab{a}})Li, Yang, Song, and
  Hospedales]{li2018learning}
Da~Li, Yongxin Yang, Yi-Zhe Song, and Timothy Hospedales.
\newblock Learning to generalize: Meta-learning for domain generalization.
\newblock In \emph{Proceedings of the AAAI conference on artificial
  intelligence}, volume~32, 2018{\natexlab{a}}.

\bibitem[Li et~al.(2018{\natexlab{b}})Li, Pan, Wang, and Kot]{li2018domain}
Haoliang Li, Sinno~Jialin Pan, Shiqi Wang, and Alex~C Kot.
\newblock Domain generalization with adversarial feature learning.
\newblock In \emph{Proceedings of the IEEE conference on computer vision and
  pattern recognition}, pp.\  5400--5409, 2018{\natexlab{b}}.

\bibitem[Li et~al.(2021)Li, Beirami, Sanjabi, and Smith]{li2020tilted}
Tian Li, Ahmad Beirami, Maziar Sanjabi, and Virginia Smith.
\newblock Tilted empirical risk minimization.
\newblock In \emph{9th International Conference on Learning Representations,
  {ICLR} 2021, Virtual Event, Austria, May 3-7, 2021}. OpenReview.net, 2021.
\newblock URL \url{https://openreview.net/forum?id=K5YasWXZT3O}.

\bibitem[Li et~al.(2023)Li, Beirami, Sanjabi, and Smith]{beirami2023tilted}
Tian Li, Ahmad Beirami, Maziar Sanjabi, and Virginia Smith.
\newblock On tilted losses in machine learning: Theory and applications.
\newblock \emph{Journal of Machine Learning Research}, 24:\penalty0 1--79,
  2023.

\bibitem[Li et~al.(2018{\natexlab{c}})Li, Tian, Gong, Liu, Liu, Zhang, and
  Tao]{li2018deep}
Ya~Li, Xinmei Tian, Mingming Gong, Yajing Liu, Tongliang Liu, Kun Zhang, and
  Dacheng Tao.
\newblock Deep domain generalization via conditional invariant adversarial
  networks.
\newblock In \emph{Proceedings of the European Conference on Computer Vision
  (ECCV)}, pp.\  624--639, 2018{\natexlab{c}}.

\bibitem[Liang et~al.(2016)Liang, Jiang, Meng, and
  Hauptmann]{liang2016learning}
Junwei Liang, Lu~Jiang, Deyu Meng, and Alexander~G Hauptmann.
\newblock Learning to detect concepts from webly-labeled video data.
\newblock In \emph{IJCAI}, volume~1, pp.\  3--1, 2016.

\bibitem[Lin et~al.(2017)Lin, Goyal, Girshick, He, and
  Doll{\'{a}}r]{lin2017focal}
Tsung{-}Yi Lin, Priya Goyal, Ross~B. Girshick, Kaiming He, and Piotr
  Doll{\'{a}}r.
\newblock Focal loss for dense object detection.
\newblock In \emph{{IEEE} International Conference on Computer Vision, {ICCV}
  2017, Venice, Italy, October 22-29, 2017}, pp.\  2999--3007. {IEEE} Computer
  Society, 2017.
\newblock \doi{10.1109/ICCV.2017.324}.
\newblock URL \url{https://doi.org/10.1109/ICCV.2017.324}.

\bibitem[Majidi et~al.(2021)Majidi, Amid, Talebi, and
  Warmuth]{majidi2021exponentiated}
Negin Majidi, Ehsan Amid, Hossein Talebi, and Manfred~K Warmuth.
\newblock Exponentiated gradient reweighting for robust training under label
  noise and beyond.
\newblock \emph{arXiv preprint arXiv:2104.01493}, 2021.

\bibitem[Majmundar et~al.(2022)Majmundar, Goyal, Netrapalli, and
  Jain]{majmundar2022met}
Kushal Majmundar, Sachin Goyal, Praneeth Netrapalli, and Prateek Jain.
\newblock Met: Masked encoding for tabular data.
\newblock \emph{arXiv preprint arXiv:2206.08564}, 2022.

\bibitem[Menon et~al.(2019)Menon, Rawat, Reddi, and Kumar]{menon2019can}
Aditya~Krishna Menon, Ankit~Singh Rawat, Sashank~J Reddi, and Sanjiv Kumar.
\newblock Can gradient clipping mitigate label noise?
\newblock In \emph{International Conference on Learning Representations}, 2019.

\bibitem[Nam et~al.(2021)Nam, Lee, Park, Yoon, and Yoo]{nam2021reducing}
Hyeonseob Nam, HyunJae Lee, Jongchan Park, Wonjun Yoon, and Donggeun Yoo.
\newblock Reducing domain gap by reducing style bias.
\newblock In \emph{Proceedings of the IEEE/CVF Conference on Computer Vision
  and Pattern Recognition}, pp.\  8690--8699, 2021.

\bibitem[Namkoong \& Duchi(2016)Namkoong and Duchi]{namkoong2016stochastic}
Hongseok Namkoong and John~C Duchi.
\newblock Stochastic gradient methods for distributionally robust optimization
  with f-divergences.
\newblock \emph{Advances in neural information processing systems}, 29, 2016.

\bibitem[Namkoong \& Duchi(2017)Namkoong and Duchi]{namkoong2017variance}
Hongseok Namkoong and John~C Duchi.
\newblock Variance-based regularization with convex objectives.
\newblock \emph{Advances in neural information processing systems}, 30, 2017.

\bibitem[Nemirovsky et~al.(1983)Nemirovsky, Yudin, and
  DAWSON]{nemirovsky1983wiley}
AS~Nemirovsky, DB~Yudin, and ER~DAWSON.
\newblock Wiley-interscience series in discrete mathematics, 1983.

\bibitem[Pentina et~al.(2015)Pentina, Sharmanska, and
  Lampert]{pentina2015curriculum}
Anastasia Pentina, Viktoriia Sharmanska, and Christoph~H Lampert.
\newblock Curriculum learning of multiple tasks.
\newblock In \emph{Proceedings of the IEEE conference on computer vision and
  pattern recognition}, pp.\  5492--5500, 2015.

\bibitem[Pi et~al.(2016)Pi, Li, Zhang, Meng, Wu, Xiao, Zhuang,
  et~al.]{pi2016self}
Te~Pi, Xi~Li, Zhongfei Zhang, Deyu Meng, Fei Wu, Jun Xiao, Yueting Zhuang,
  et~al.
\newblock Self-paced boost learning for classification.
\newblock In \emph{IJCAI}, pp.\  1932--1938, 2016.

\bibitem[Qi et~al.(2021)Qi, Guo, Xu, Jin, and Yang]{qi2021online}
Qi~Qi, Zhishuai Guo, Yi~Xu, Rong Jin, and Tianbao Yang.
\newblock An online method for a class of distributionally robust optimization
  with non-convex objectives.
\newblock \emph{Advances in Neural Information Processing Systems},
  34:\penalty0 10067--10080, 2021.

\bibitem[Qi et~al.(2023{\natexlab{a}})Qi, Lyu, Chan, Bai, and
  Yang]{qi2022stochastic}
Qi~Qi, Jiameng Lyu, Kung-Sik Chan, Er-Wei Bai, and Tianbao Yang.
\newblock Stochastic constrained {DRO} with a complexity independent of sample
  size.
\newblock \emph{Transactions on Machine Learning Research}, 2023{\natexlab{a}}.
\newblock ISSN 2835-8856.
\newblock URL \url{https://openreview.net/forum?id=VpaXrBFYZ9}.

\bibitem[Qi et~al.(2023{\natexlab{b}})Qi, Xu, Yin, Jin, and
  Yang]{qi2020attentional}
Qi~Qi, Yi~Xu, Wotao Yin, Rong Jin, and Tianbao Yang.
\newblock Attentional-biased stochastic gradient descent.
\newblock \emph{Trans. Mach. Learn. Res.}, 2023, 2023{\natexlab{b}}.
\newblock URL \url{https://openreview.net/forum?id=B0WYWvVA2r}.

\bibitem[Rahimi \& Recht(2008)Rahimi and Recht]{rahimi2008weighted}
Ali Rahimi and Benjamin Recht.
\newblock Weighted sums of random kitchen sinks: Replacing minimization with
  randomization in learning.
\newblock \emph{Advances in neural information processing systems}, 21, 2008.

\bibitem[Ren et~al.(2018)Ren, Zeng, Yang, and Urtasun]{ren2018learning}
Mengye Ren, Wenyuan Zeng, Bin Yang, and Raquel Urtasun.
\newblock Learning to reweight examples for robust deep learning.
\newblock In \emph{International conference on machine learning}, pp.\
  4334--4343. PMLR, 2018.

\bibitem[Requeima et~al.(2019)Requeima, Gordon, Bronskill, Nowozin, and
  Turner]{requeima2019fast}
James Requeima, Jonathan Gordon, John Bronskill, Sebastian Nowozin, and
  Richard~E Turner.
\newblock Fast and flexible multi-task classification using conditional neural
  adaptive processes.
\newblock \emph{Advances in Neural Information Processing Systems}, 32, 2019.

\bibitem[Ruder(2016)]{ruder2016overview}
Sebastian Ruder.
\newblock An overview of gradient descent optimization algorithms.
\newblock \emph{arXiv preprint arXiv:1609.04747}, 2016.

\bibitem[Sagawa* et~al.(2020)Sagawa*, Koh*, Hashimoto, and
  Liang]{sagawa2019distributionally}
Shiori Sagawa*, Pang~Wei Koh*, Tatsunori~B. Hashimoto, and Percy Liang.
\newblock Distributionally robust neural networks.
\newblock In \emph{International Conference on Learning Representations}, 2020.
\newblock URL \url{https://openreview.net/forum?id=ryxGuJrFvS}.

\bibitem[Sakhi et~al.(2020)Sakhi, Faury, and Vasile]{sakhi2020improving}
Otmane Sakhi, Louis Faury, and Flavian Vasile.
\newblock Improving offline contextual bandits with distributional robustness.
\newblock \emph{arXiv preprint arXiv:2011.06835}, 2020.

\bibitem[Shamir \& Zhang(2013)Shamir and Zhang]{shamir2013stochastic}
Ohad Shamir and Tong Zhang.
\newblock Stochastic gradient descent for non-smooth optimization: Convergence
  results and optimal averaging schemes.
\newblock In \emph{International conference on machine learning}, pp.\  71--79.
  PMLR, 2013.

\bibitem[Shapiro(2017)]{shapiro2017distributionally}
Alexander Shapiro.
\newblock Distributionally robust stochastic programming.
\newblock \emph{SIAM Journal on Optimization}, 27\penalty0 (4):\penalty0
  2258--2275, 2017.

\bibitem[Shi et~al.(2015)Shi, Larson, and Jonker]{shi2015recurrent}
Yangyang Shi, Martha Larson, and Catholijn~M Jonker.
\newblock Recurrent neural network language model adaptation with curriculum
  learning.
\newblock \emph{Computer Speech \& Language}, 33\penalty0 (1):\penalty0
  136--154, 2015.

\bibitem[Shrivastava et~al.(2016)Shrivastava, Gupta, and
  Girshick]{shrivastava2016training}
Abhinav Shrivastava, Abhinav Gupta, and Ross Girshick.
\newblock Training region-based object detectors with online hard example
  mining.
\newblock In \emph{Proceedings of the IEEE conference on computer vision and
  pattern recognition}, pp.\  761--769, 2016.

\bibitem[Shu et~al.(2019)Shu, Xie, Yi, Zhao, Zhou, Xu, and Meng]{shu2019meta}
Jun Shu, Qi~Xie, Lixuan Yi, Qian Zhao, Sanping Zhou, Zongben Xu, and Deyu Meng.
\newblock Meta-weight-net: Learning an explicit mapping for sample weighting.
\newblock \emph{Advances in neural information processing systems}, 32, 2019.

\bibitem[Sinha et~al.(2018)Sinha, Namkoong, and Duchi]{sinha2017certifying}
Aman Sinha, Hongseok Namkoong, and John Duchi.
\newblock Certifiable distributional robustness with principled adversarial
  training.
\newblock In \emph{International Conference on Learning Representations}, 2018.
\newblock URL \url{https://openreview.net/forum?id=Hk6kPgZA-}.

\bibitem[Snell et~al.(2017)Snell, Swersky, and Zemel]{snell2017prototypical}
Jake Snell, Kevin Swersky, and Richard Zemel.
\newblock Prototypical networks for few-shot learning.
\newblock \emph{Advances in neural information processing systems}, 30, 2017.

\bibitem[Soviany(2020)]{soviany2020curriculum}
Petru Soviany.
\newblock Curriculum learning with diversity for supervised computer vision
  tasks.
\newblock \emph{arXiv preprint arXiv:2009.10625}, 2020.

\bibitem[Soviany et~al.(2022)Soviany, Ionescu, Rota, and
  Sebe]{soviany2022curriculum}
Petru Soviany, Radu~Tudor Ionescu, Paolo Rota, and Nicu Sebe.
\newblock Curriculum learning: A survey.
\newblock \emph{International Journal of Computer Vision}, 130\penalty0
  (6):\penalty0 1526--1565, 2022.

\bibitem[Spitkovsky et~al.(2010)Spitkovsky, Alshawi, and
  Jurafsky]{spitkovsky2009baby}
Valentin~I Spitkovsky, Hiyan Alshawi, and Dan Jurafsky.
\newblock From baby steps to leapfrog: How “less is more” in unsupervised
  dependency parsing.
\newblock In \emph{Human Language Technologies: The 2010 Annual Conference of
  the North American Chapter of the Association for Computational Linguistics},
  pp.\  751--759, 2010.

\bibitem[Sun \& Saenko(2016)Sun and Saenko]{sun2016deep}
Baochen Sun and Kate Saenko.
\newblock Deep coral: Correlation alignment for deep domain adaptation.
\newblock In \emph{European conference on computer vision}, pp.\  443--450.
  Springer, 2016.

\bibitem[Touvron et~al.(2021)Touvron, Cord, Douze, Massa, Sablayrolles, and
  J{\'e}gou]{touvron2021training}
Hugo Touvron, Matthieu Cord, Matthijs Douze, Francisco Massa, Alexandre
  Sablayrolles, and Herv{\'e} J{\'e}gou.
\newblock Training data-efficient image transformers \& distillation through
  attention.
\newblock In \emph{International conference on machine learning}, pp.\
  10347--10357. PMLR, 2021.

\bibitem[Tudor~Ionescu et~al.(2016)Tudor~Ionescu, Alexe, Leordeanu, Popescu,
  Papadopoulos, and Ferrari]{tudor2016hard}
Radu Tudor~Ionescu, Bogdan Alexe, Marius Leordeanu, Marius Popescu, Dim~P
  Papadopoulos, and Vittorio Ferrari.
\newblock How hard can it be? estimating the difficulty of visual search in an
  image.
\newblock In \emph{Proceedings of the IEEE Conference on Computer Vision and
  Pattern Recognition}, pp.\  2157--2166, 2016.

\bibitem[Ucar et~al.(2021)Ucar, Hajiramezanali, and Edwards]{ucar2021subtab}
Talip Ucar, Ehsan Hajiramezanali, and Lindsay Edwards.
\newblock Subtab: Subsetting features of tabular data for self-supervised
  representation learning.
\newblock \emph{Advances in Neural Information Processing Systems},
  34:\penalty0 18853--18865, 2021.

\bibitem[Vapnik(1999)]{vapnik1999overview}
Vladimir~N Vapnik.
\newblock An overview of statistical learning theory.
\newblock \emph{IEEE transactions on neural networks}, 10\penalty0
  (5):\penalty0 988--999, 1999.

\bibitem[Verma et~al.(2021)Verma, Luong, Kawaguchi, Pham, and
  Le]{verma2021towards}
Vikas Verma, Thang Luong, Kenji Kawaguchi, Hieu Pham, and Quoc Le.
\newblock Towards domain-agnostic contrastive learning.
\newblock In \emph{International Conference on Machine Learning}, pp.\
  10530--10541. PMLR, 2021.

\bibitem[Wainwright(2019)]{wainwright2019high}
Martin~J Wainwright.
\newblock \emph{High-dimensional statistics: A non-asymptotic viewpoint},
  volume~48.
\newblock Cambridge university press, 2019.

\bibitem[Wang \& Vasconcelos(2018)Wang and Vasconcelos]{wang2018towards}
Pei Wang and Nuno Vasconcelos.
\newblock Towards realistic predictors.
\newblock In \emph{Proceedings of the European Conference on Computer Vision
  (ECCV)}, pp.\  36--51, 2018.

\bibitem[Wang et~al.(2017)Wang, Kucukelbir, and Blei]{wang2017robust}
Yixin Wang, Alp Kucukelbir, and David~M Blei.
\newblock Robust probabilistic modeling with bayesian data reweighting.
\newblock In \emph{International Conference on Machine Learning}, pp.\
  3646--3655. PMLR, 2017.

\bibitem[Yan et~al.(2020{\natexlab{a}})Yan, Song, Li, Zou, and
  Ren]{yan2020improve}
Shen Yan, Huan Song, Nanxiang Li, Lincan Zou, and Liu Ren.
\newblock Improve unsupervised domain adaptation with mixup training.
\newblock \emph{arXiv preprint arXiv:2001.00677}, 2020{\natexlab{a}}.

\bibitem[Yan et~al.(2020{\natexlab{b}})Yan, Xu, Lin, Liu, and
  Yang]{yan2020optimal}
Yan Yan, Yi~Xu, Qihang Lin, Wei Liu, and Tianbao Yang.
\newblock Optimal epoch stochastic gradient descent ascent methods for min-max
  optimization.
\newblock \emph{Advances in Neural Information Processing Systems},
  33:\penalty0 5789--5800, 2020{\natexlab{b}}.

\bibitem[Yang et~al.(2010)Yang, Xu, White, Schuurmans, and Yu]{yang2010relaxed}
Min Yang, Linli Xu, Martha White, Dale Schuurmans, and Yao-liang Yu.
\newblock Relaxed clipping: A global training method for robust regression and
  classification.
\newblock \emph{Advances in neural information processing systems}, 23, 2010.

\bibitem[Yoon et~al.(2020)Yoon, Zhang, Jordon, and van~der
  Schaar]{yoon2020vime}
Jinsung Yoon, Yao Zhang, James Jordon, and Mihaela van~der Schaar.
\newblock Vime: Extending the success of self-and semi-supervised learning to
  tabular domain.
\newblock \emph{Advances in Neural Information Processing Systems},
  33:\penalty0 11033--11043, 2020.

\bibitem[Zadrozny(2004)]{zadrozny2004learning}
Bianca Zadrozny.
\newblock Learning and evaluating classifiers under sample selection bias.
\newblock In \emph{Proceedings of the twenty-first international conference on
  Machine learning}, pp.\  114, 2004.

\bibitem[Zaremba \& Sutskever(2014)Zaremba and Sutskever]{zaremba2014learning}
Wojciech Zaremba and Ilya Sutskever.
\newblock Learning to execute.
\newblock \emph{arXiv preprint arXiv:1410.4615}, 2014.

\bibitem[Zeiler(2012)]{zeiler2012adadelta}
Matthew~D Zeiler.
\newblock Adadelta: an adaptive learning rate method.
\newblock \emph{arXiv preprint arXiv:1212.5701}, 2012.

\bibitem[Zhai et~al.(2021)Zhai, Dan, Kolter, and Ravikumar]{zhai2021doro}
Runtian Zhai, Chen Dan, Zico Kolter, and Pradeep Ravikumar.
\newblock Doro: Distributional and outlier robust optimization.
\newblock In \emph{International Conference on Machine Learning}, pp.\
  12345--12355. PMLR, 2021.

\bibitem[Zhang et~al.(2015)Zhang, Meng, Li, Jiang, Zhao, and
  Han]{zhang2015self}
Dingwen Zhang, Deyu Meng, Chao Li, Lu~Jiang, Qian Zhao, and Junwei Han.
\newblock A self-paced multiple-instance learning framework for co-saliency
  detection.
\newblock In \emph{Proceedings of the IEEE international conference on computer
  vision}, pp.\  594--602, 2015.

\bibitem[Zhang et~al.(2021)Zhang, Marklund, Dhawan, Gupta, Levine, and
  Finn]{zhang2021adaptive}
Marvin Zhang, Henrik Marklund, Nikita Dhawan, Abhishek Gupta, Sergey Levine,
  and Chelsea Finn.
\newblock Adaptive risk minimization: Learning to adapt to domain shift.
\newblock \emph{Advances in Neural Information Processing Systems},
  34:\penalty0 23664--23678, 2021.

\bibitem[Zhang \& Pfister(2021)Zhang and Pfister]{zhang2021learning}
Zizhao Zhang and Tomas Pfister.
\newblock Learning fast sample re-weighting without reward data.
\newblock In \emph{Proceedings of the IEEE/CVF International Conference on
  Computer Vision}, pp.\  725--734, 2021.

\bibitem[Zhou et~al.(2020)Zhou, Wang, and Bilmes]{zhou2020curriculum}
Tianyi Zhou, Shengjie Wang, and Jeffrey Bilmes.
\newblock Curriculum learning by dynamic instance hardness.
\newblock \emph{Advances in Neural Information Processing Systems},
  33:\penalty0 8602--8613, 2020.

\bibitem[Zhu et~al.(2022{\natexlab{a}})Zhu, Jiao, and
  Steinhardt]{zhu2022generalized}
Banghua Zhu, Jiantao Jiao, and Jacob Steinhardt.
\newblock Generalized resilience and robust statistics.
\newblock \emph{The Annals of Statistics}, 50\penalty0 (4):\penalty0
  2256--2283, 2022{\natexlab{a}}.

\bibitem[Zhu et~al.(2024)Zhu, Chen, Shen, Li, and Elhoseiny]{zhu2023minigpt}
Deyao Zhu, Jun Chen, Xiaoqian Shen, Xiang Li, and Mohamed Elhoseiny.
\newblock Mini{GPT}-4: Enhancing vision-language understanding with advanced
  large language models.
\newblock In \emph{The Twelfth International Conference on Learning
  Representations}, 2024.
\newblock URL \url{https://openreview.net/forum?id=1tZbq88f27}.

\bibitem[Zhu et~al.(2022{\natexlab{b}})Zhu, Chen, Yin, See, and
  Liu]{zhureinforced}
Lanyun Zhu, Tianrun Chen, Jianxiong Yin, Simon See, and Jun Liu.
\newblock Reinforced sample reweighting policy for semi-supervised learning.
\newblock 2022{\natexlab{b}}.

\bibitem[Zhu et~al.(2015)Zhu, Kiros, Zemel, Salakhutdinov, Urtasun, Torralba,
  and Fidler]{zhu2015aligning}
Yukun Zhu, Ryan Kiros, Rich Zemel, Ruslan Salakhutdinov, Raquel Urtasun,
  Antonio Torralba, and Sanja Fidler.
\newblock Aligning books and movies: Towards story-like visual explanations by
  watching movies and reading books.
\newblock In \emph{Proceedings of the IEEE international conference on computer
  vision}, pp.\  19--27, 2015.

\end{thebibliography}
\bibliographystyle{tmlr}

\appendix
\section*{Appendix}

\section*{Reproducibility Statement}
\label{appendix:reproducibility}
Our proposed loss function is a single line of change. However, one would have to play around with the learning rate (generally lower than the baseline setting). Our experiments are based on public datasets and open-source code repositories. The proposed final formulation  \namee~requires \textcolor{blue}{\textbf{one line of code change}}. \\

\lstset{language=Python}
\lstset{frame=lines}
\lstset{basicstyle=\ttfamily\footnotesize}
\lstset{escapeinside={<@}{@>}}
Suppose the per-sample loss is given. Example code for applying \namee~in Jax is shown below.

\begin{lstlisting}
import jax.numpy as jnp
import jax

def rgd_e(loss, temp=alpha, reduce=True):
    <@\texttt{\textcolor{gray}{\# alpha >0.}}@>
  out = loss * jnp.exp(
      jnp.clip(jax.lax.stop_gradient(loss), a_min=0, a_max=temp) / (temp + 1)
  )
  return out.sum() / len(out) if reduce else out
\end{lstlisting}

\section{Extended Related Work}
\label{appendix:related_work}
\subsection{Per-Sample Reweighting}
In this section, we review data reweighting techniques developed outside of the DRO community. The idea of re-weighting samples can be dated back to the works of \cite{chawla2002smote, zadrozny2004learning}, which pre-computed per-sample weights using certain prior knowledge. Recent approaches alleviate the need for human supervision by dynamically computing the per-sample weights. One of the early works in this category is AdaBoost, which is a popular boosting algorithm~\citep{freund1997decision}. Similar to \namel, AdaBoost uses exponential weighting mechanism to reweight data points. However, AdaBoost is used for learning an ensemble of weak learners. Whereas, in this work, we are interested in learning a single model that can achieve better generalization guarantees. Furthermore, AdaBoost is only studied for supervised learning (in particular, classification and regression). In contrast, \namel~can be applied on any learning task. Recent works of~\cite{leng2022polyloss, lin2017focal} showed that certain modifications to standard cross entropy loss - that involve truncating its Taylor-series expansion - can improve the performance of DNNs. These techniques can be viewed as performing sample re-weighting. However, these techniques only apply to cross-entropy loss and are not easily extendable to general learning tasks.

Other approaches based on meta-learning have been proposed for class imbalance and label noise~\citep{ shu2019meta, ren2018learning, gonzalez2021optimizing}. Many popular approaches in this line of work require training a separate neural network for re-weighting the data points~\citep{ren2018learning, shu2019meta}. However, these approaches are seldom used in practice as the underlying  bi-level optimization problem is hard to implement~\citep{ivanova2023challenge}. Unlike these approaches, our \namel~algorithm does not require a separate neural network for re-weighting and thus doesn't add any computational overhead over vanilla training. Moreover, compared to existing sample re-weighting techniques, our approach applies to various learning tasks (see Section~\ref{sec:experiments}).
 Another line of work uses a history buffer which stores a snapshot of the trajectory of each point and facilitates giving more importance to points which leads to more learning in the model \citep{zhang2021learning}. Other approaches, such as \cite{zhureinforced}, use reinforcement learning to learn the per-sample weights using a ``pretraining-boosting'' two-stage MDP curriculum where the agent network is firstly pre-trained and optimized for deployment in the classification problem. Another line of work has considered sample re-weighting in the presence of outliers~\citep{kumar2010self, de2003framework, jiang2014easy, jiang2014self, wang2017robust, li2020tilted, beirami2023tilted}. These works down-weight points with high-loss value. The rationality behind this lies in the idea that these high-loss samples are more likely to be outliers and, thus, should be ignored during the training process.
Finally, works such \cite{castells2020superloss} propose a confidence-aware loss proportional to the lowest loss of that sample. They use a threshold ($\gamma$) to decide how practical or important each point is. 

An emerging line of work on optimization focuses on designing sample re-weighting for improving the convergence speed of SGD~\citep{katharopoulos2018not, el2022stochastic} by decreasing the variance in SGD iterates. Note that in contrast to these works which aim to minimize the ERM objective, we aim to solve the DRO objective which has better generalization guarantees than ERM in high variance, low sample complexity regime.

\subsection{Pre-conditioning}
Pre-conditioning can usually mean normalization of inputs, batch normalization, or scaling gradients in a few directions. This section predominantly discusses techniques that focus on scaling gradients in a few directions.  A common technique to improve the training speed in deep learning is using adaptive step-size optimization methods, such as the Newton method, which takes advantage of the second-order gradients. However, computing the Hessian matrix is computationally intensive, leading to Quasi-Newton methods: methods that approximate the value of the Hessian instead of computing them every time \cite{le2011optimization}. Another popular alternative is to use an element-wise adaptive learning rate, which has shown great promise in deep learning. Some of the popular techniques here include ADAgrad \citep{duchi2011adaptive}, RMSProp \citep{ruder2016overview}, ADAdelta \citep{zeiler2012adadelta}. For instance, ADAgrad is a diagonal pre-conditioning technique where the pre-conditioning across each dimension is computed as the inverse square root of the norms of gradients along that dimension accumulated over training. Unfortunately, this accumulation of gradients makes it susceptible to falling in a saddle point as the scaling factor decreases monotonically. %

\subsection{Curriculum Learning}
Another important research area that has explored data reweighting is Curriculum Learning (CL). CL, originally introduced by \cite{bengio2009curriculum}, is a vast domain focussing on how the model should be taught, and draws inspiration from how humans learn concepts. For instance, humans generally grasp on to easier concepts such as basic shapes (triangle, rectangle, etc.) before moving on to learning significantly more complex structures (heptagram, triquetra, etc.). Curriculum learning strategies have been widely used in various areas of machine learning and involves finding a way to rank samples, as well as the right pacing functions for introducing more difficult data in our training. The techniques developed for CL have typically focused on giving importance to easier samples at the beginning of training, and slowly progressing towards harder samples~\citep{bengio2009curriculum, chen2015webly, tudor2016hard, pentina2015curriculum, shi2015recurrent, spitkovsky2009baby, zaremba2014learning}. In contrast, DRO focuses the learning on harder samples throughout the training process. That being said, there have also been a class of works in CL which showed  the learning harder examples first, and then moving to easier ones could lead to improved performance in certain conditions, through Hard Example mining (HEM) or anti-curriculum~\citep{jesson2017cased, shrivastava2016training, wang2018towards, zhou2020curriculum, braun2017curriculum, pi2016self}. There have been predominantly three classes of CL, and various amalgamations of these in literature~\citep{soviany2022curriculum}. 

\textit{Vanilla CL:} The vanilla CL usually involves a pre-defined notion of hardness. For example, \cite{bengio2009curriculum} used geometric shapes to clearly differentiate easy and hard samples. Others such as \cite{spitkovsky2009baby} exploited the length of sequences as a signal for difficulty. 

\textit{Self-Paced Learning (SPCL):} This differs from the vanilla CL with respect to the evaluation of difficulty. This concept of ``difficulty'' is not known beforehand and is measured repeatedly during training. Works such as \cite{kumar2010self} used the likelihood of the prediction to rank the samples. Other works such as \cite{lee2011learning} used objectness as a measure to define the training schedule.

\textit{Balanced Curriculum (BCL):} In addition to prior works such as vanilla CL and SPCL, balanced curriculum approaches come with an added condition of diversity within a batch. These constraints (on classes, image regions, etc.) help the model learn robust features and not overfit to the spurious correlations of the easy samples \citep{zhang2015self, soviany2020curriculum}. 

In this work, we will resort to only describing few works which focus on instance level reweighting of data points \citep{kumar2010self, li2017self, kumar2011learning, pi2016self, liang2016learning, fan2017self, li2017self}. Most of these works~\citep{kumar2011learning, fan2017self} follow a binary weighting mechanism of $\{0,1\}$ to decide whether the model should learn using the current sample or not. Others such as \cite{li2017self, pi2016self, liang2016learning} are more continuous in the weighing mechanism and generally give higher weights to samples with lower losses. This makes them fundamentally different from  \namel~attempts to achieve in this work. \namel~attempts to focus more heavily on the harder samples throughout training and does so using a simple closed form expression of the loss. 

\subsection{Comparison with existing KL-DRO optimization approaches}
\label{appendix:kldro}

In this section, we discuss a few more aspects of the related work that weren't discussed in the main paper. We specifically focus on prior works that developed algorithms for KL divergence based DRO. The earliest work on KL-DRO dates back to 2013 by \citet{hu2013kullback}. However, it was only recently that these works have become widespread in deep learning. The RECOVER algorithm by \citet{qi2021online} was one of the early works to scale KL-DRO to deep neural networks. It attempted to solve the non-convex DRO problem with a duality-free stochastic method by formulating the min-max formulation into an equivalent stochastic compositional problem. ABSGD~\citep{qi2020attentional} and SCDRO~\citep{qi2022stochastic} improved upon RECOVER by designing more efficient algorithms. 
However, the performance of these algorithms on large-scale models and datasets is not rigorously studied, as the Imagenet-LT, iNaturalist experiments conducted in these works started from a pre-trained network and only finetuned the last layer. In contrast, in this work, we learn the entire ViT-S model from scratch and show improved generalization.

\section{Proofs of Section~\ref{sec:algo}}
\label{sec:proofs}
\subsection{Proof of Proposition~\ref{prop:dual_dro_kl}}
\label{sec:proof1}
\begin{proposition}
    Consider DRO with KL-divergence-based uncertainty set. Assume that the data set $(z_i)_{i=1}^{n}$ is comprised of unique points (i.e, no repeated data points). Then $\min_{\theta\in\Theta}\Rdron$ can be rewritten as
    \[
    \min_{\theta\in\Theta}\frac{1}{\gamma}\log{\E_{\pdatan}[e^{\gamma\ell(z; \theta)}]},
    \]
    for some constant $\gamma>0$ that is independent of $\theta$.
\end{proposition}
\begin{proof}
    
Recall the empirical DRO risk  $\Rdron(\theta)$ is defined as
\[
\Rdron(\theta) \coloneqq \sup_{P': D(P'||\pdatan) \leq \rho} \E_{P'}[\ell(z,\theta)]
\]

\def\E{\mathbb{E}}
Using Lagrangian duality, we rewrite $\Rdron(\theta)$ as
    \begin{align*}
         \sup_{P': D(P'||\pdatan) \leq \rho} \E_{P'}[\ell(z,\theta)] &= \sup_{P'} \inf_{\beta > 0}\E_{P'}[\ell(z,\theta)] - \beta ( D(P'||\pdatan) -\rho)\\ 
         & \stackrel{(a)}{=} \sup_{P'} \inf_{\beta > 0}\E_{P'}[\ell(z,\theta)] - \beta(\E_{P'}[\log{dP'}] +\log{n} - \rho),
    \end{align*}
    where $(a)$ follows from our choice of divergence, and the fact that the data points $\{z_i\}_{i=1}^{n}$ are all unique (i.e, no repetitions). 
Observe that the objective in the last expression is concave in $P'$ and linear in $\beta$. So the max-min problem above is concave-convex. Using Lagrangian duality to swap the order of min and max, we get
    \begin{align*}
         \sup_{P': D(P'||\pdatan) \leq \rho} \E_{P'}[\ell(z,\theta)] 
         &=  \inf_{\beta > 0}\sup_{P'} \E_{P'}[\ell(z,\theta)] - \beta(\E_{P'}[\log{dP'}] +\log{n} - \rho).
    \end{align*}
    This shows that minimizing $\Rdron$ is equivalent to the following problem
    \[
    \inf_{\beta > 0, \theta \in \Theta}\sup_{P'} \E_{P'}[\ell(z,\theta)] - \beta(\E_{P'}[\log{dP'}] +\log{n} - \rho).
    \]
    For any fixed $\beta, \theta$, the inner supremum is attained at a $P'$ that satisfies (see Theorem 1 of \cite{hsieh2019finding})
    \[
    P'(z) \propto \exp\left(\ell(z,\theta)/\beta\right).
    \]
    This can be derived using the following first order optimality condition: $\forall z, \ell(z,\theta) - \beta\log{P'(z)}-\beta = c$ for some constant $c$.
Substituting this in the previous equation, we get the following equivalent optimization problem
\[
\inf_{\beta > 0}\inf_{\theta \in \Theta} \beta\log\E_{P'}[e^{\ell(z,\theta)/\beta}] - \beta(\log{n} - \rho).
\]
Letting $\gamma^{-1}$ be the minimizer of the outer minimization problem, we get the required result.
\end{proof}
\subsection{Proof of Proposition~\ref{prop:convergence}}
\label{sec:proof2}

\begin{proof}
Note that whenever $\ell(z;)$ is convex, so is $f(\theta) := \E_{z\sim \pdatan} \exp(\gamma \ell(z;))$. It is easy to check that the function $f$ and the constraint set $\Theta$ satisfy the conditions in \citet[Theorem 2]{shamir2013stochastic}. From this, we conclude that:

\begin{equation}\label{eq:convergence}\E_{\theta_T}f(\theta_T) - \inf_{\theta \in \Theta}f(\theta) = O\left(\frac{\log T}{\sqrt{T}}\right)\end{equation}

The proof of~\eqref{eq:convergence} for the step size sequence $\eta_t = \frac{C}{\sqrt{t}}$ follows from the statement of \citet[Theorem 2]{shamir2013stochastic}. The case of the constant step-size $\eta_t = \frac{C}{\sqrt{T}}$ follows by a simple modification of the proof of \citet[Theorem 2]{shamir2013stochastic} where we substitute the appearance of $1/\sqrt{t}$ due to the step size with $1/\sqrt{T}$. We now convert the guarantees in Equation~\ref{eq:convergence} to guarantees in terms of $\log{f(\theta)}.$ Let $\theta^* \in \arg\inf_{\theta \in \Theta}f(\theta)$. By our assupmption, $\ell(z;)$ is bounded above and below.  So $\log f(\theta_T) - \log f(\theta^*) \leq \bar{C}(f(\theta_T)-f(\theta^*))$ for some $\bar{C}$. Combining this with~\eqref{eq:convergence}, we conclude the statement of the proposition.
\end{proof}

\subsection{Other Divergences} 
\label{sec:other_divergences}
\paragraph{$\chi^2$-divergence.} Consider $\chi^2$-divergence which is defined as
\[
D(P'||P) = \E_P\left[\left(\frac{dP'}{dP}-1\right)^2\right].
\]
We now follow a similar argument as in the proof of Proposition~\ref{prop:dual_dro_kl} to derive an equivalent expression for the DRO objective. 
We have
    \begin{align*}
         \sup_{P': D(P'||\pdatan) \leq \rho} \E_{P'}[\ell(z,\theta)] &= \sup_{P'} \inf_{\beta > 0}\E_{P'}[\ell(z,\theta)] - \beta ( D(P'||\pdatan) -\rho)\\ 
         & \stackrel{(a)}{=} \sup_{P'} \inf_{\beta > 0}\E_{P'}[\ell(z,\theta)] - \beta(\E_{P'}[dP'/d\pdatan] - 1-\rho)\\
         & \stackrel{(b)}{=} \sup_{P'} \inf_{\beta > 0}\E_{P'}[\ell(z,\theta)] - \beta(n\E_{P'}[dP'] - 1-\rho)\\
         & \stackrel{(c)}{=} \inf_{\beta > 0}\sup_{P'} \E_{P'}[\ell(z,\theta)] - \beta(n\E_{P'}[dP'] - 1-\rho),
    \end{align*}
    where $(a), (b)$ follow from the definition of the divergence and $(c)$ follows from Lagrangian duality.
Now, consider the DRO optimization problem
\begin{align*}
     \sup_{P': D(P'||\pdatan) \leq \rho} \E_{P'}[\ell(z,\theta)] &= \inf_{\beta > 0, \theta \in \Theta}\sup_{P'} \E_{P'}[\ell(z,\theta)] - \beta(n\E_{P'}[dP'] - 1-\rho).
\end{align*}
Suppose the loss $\ell$ is positive. For any fixed $\theta, \beta$, the inner supremum in the above optimization problem is attained at a $P'$ that satisfies 
\[
P'(z) \propto \left(\ell(z,\theta)/\beta n + 1\right).
\]
This follows from the first order optimality conditions. This gives rise to the re-weighting scheme $g(x) = x+\tau$, for some appropriately chosen $\tau$.
\paragraph{Reverse KL divergence.} The reverse KL-divergence is defined as
\[
D(P'||P) = \E_P\left[-\log{\frac{dP'}{dP}}\right].
\]
Using similar arguments as above, we can rewrite the DRO optimization problem as 
\begin{align*}
     \sup_{P': D(P'||\pdatan) \leq \rho} \E_{P'}[\ell(z,\theta)] &= \inf_{\beta > 0, \theta \in \Theta}\sup_{P'} \E_{P'}[\ell(z,\theta)] + \beta(\E_{P}[\log{dP'}] - \log{n}-\rho).
\end{align*}
For any fixed $\theta, \beta$, the inner supremum in the above optimization problem is attained at a $P'$ that satisfies 
\[
P'(z) \propto \frac{1}{\tau - \ell(z,\theta)},
\]
for some appropriate $\tau$. This gives rise to the re-weighting scheme $g(x) =1/(\tau-x)$. We call this algorithm \namet. In practice, we modify this re-weighting function

It can be implemented using the following pseudocode in Jax.
\begin{lstlisting}
import jax.numpy as jnp
import jax

def rgd_t(loss, temp=alpha, reduce=True):
    <@\texttt{\textcolor{gray}{\# alpha >0.}}@>
  out = loss * (1 -
      jnp.clip(jax.lax.stop_gradient(loss), a_min=0, a_max=temp) / (temp + 1))** (-1)
  return out.sum() / len(out) if reduce else out
\end{lstlisting}

\section{Choice of divergence in \namel}
\label{appendix:difference}

\namet~is a more aggressive weighing scheme in comparison to \namee. This is fairly simple to show if you re-write both the reweighting techniques using Taylor series expansion. \namet~ multiplies the loss $l$ with $(1+l+l^2+\dots )$. Whereas, \namee~multiplies the loss $l$ with $(1+l+l^2/2! + l^3/3! + \dots )$. \namet~is a more aggressive weighing scheme than \namee, and the choice between the two schemes should depend on the problem. Some preliminary results on the class imbalance setting is depicted in Table~\ref{tab:choice_divergences}. For \namex~$g(u) = u + \tau$, we also clip $u$ to be $\min(u, \tau)$. Similarly, for \namet, $g(u) = \frac{1}{\tau-u}$, we set $\tau=1$ and clipped $u$ as $\min(u, t)/(t+1)$ where for all practical purposes - $t$ has a similar grid search and function as $\tau$ from other divergences.

Our search space and clipping involved the same grid search space as our RGD algorithm as described in the reproducibility statement.

\begin{table*}[hbt!]
\caption{Test Accuracy of ResNet-32 on Long-Tailed CIFAR-10, and CIFAR-100 dataset.}
\vskip 0.15in
\resizebox{\linewidth}{!}{
\begin{tabular}{lccccccc|ccccccc}
\toprule
Dataset&
      \multicolumn{7}{c}{\data{CIFAR-10}}  &
      \multicolumn{7}{c}{\data{CIFAR-100}}   \\ \cline{2-8} \cline{9-15}
      
Loss / Imbalance Factor&
 \data{200} &
  \data{100} &
  \data{50} &
  \data{20} &
  \data{10} &
  \data{1} &
  \textbf{Avg.} &
  \data{200} &
  \data{100} &
  \data{50} &
  \data{20} &
  \data{10} &
  \data{1} &
  \textbf{Avg.} \\
\cline{1-15}
\multicolumn{15}{l}{Cross Entropy (CE)}\\
\midrule
 Default  &   65.98  & 70.36  &  74.81 & 82.23  & 86.39 & 92.89 & 78.78 & 
34.84 &    38.32  & 43.85  & 51.14 & 55.71  & 70.50   & 49.06 \\
\namet~\textbf{(Ours)}     &    64.16  & 72.56  &  77.86 & 83.88  & 86.84 & 92.99 & 79.72 & 
36.22  & 39.87  & 43.74 & 51.86  & 56.9 & 70.80  & 49.90 \\
\namex~\textbf{(Ours)}    &   67.16  & 72.20  &  77.93 & 84.7 & 86.90 & 93.00 & 80.32 &  
35.96  & 39.70  &  43.88 & 51.29 & 56.92 & 70.73  &  49.75\\
\namee~\textbf{(Ours)}    &   \textbf{67.90}  & \textbf{73.75}  &  \textbf{79.63} & \textbf{85.44} & \textbf{88.00} & \textbf{93.27}  & \textbf{81.33} &  
\textbf{38.62}  & \textbf{41.89}  &  \textbf{46.40} & \textbf{53.48} & \textbf{58.5} & \textbf{71.30}  & \textbf{51.70} \\

\bottomrule
\end{tabular}}
\label{tab:choice_divergences}
\end{table*}

\begin{figure}[bt!]
\centering
\includegraphics[width=\linewidth]{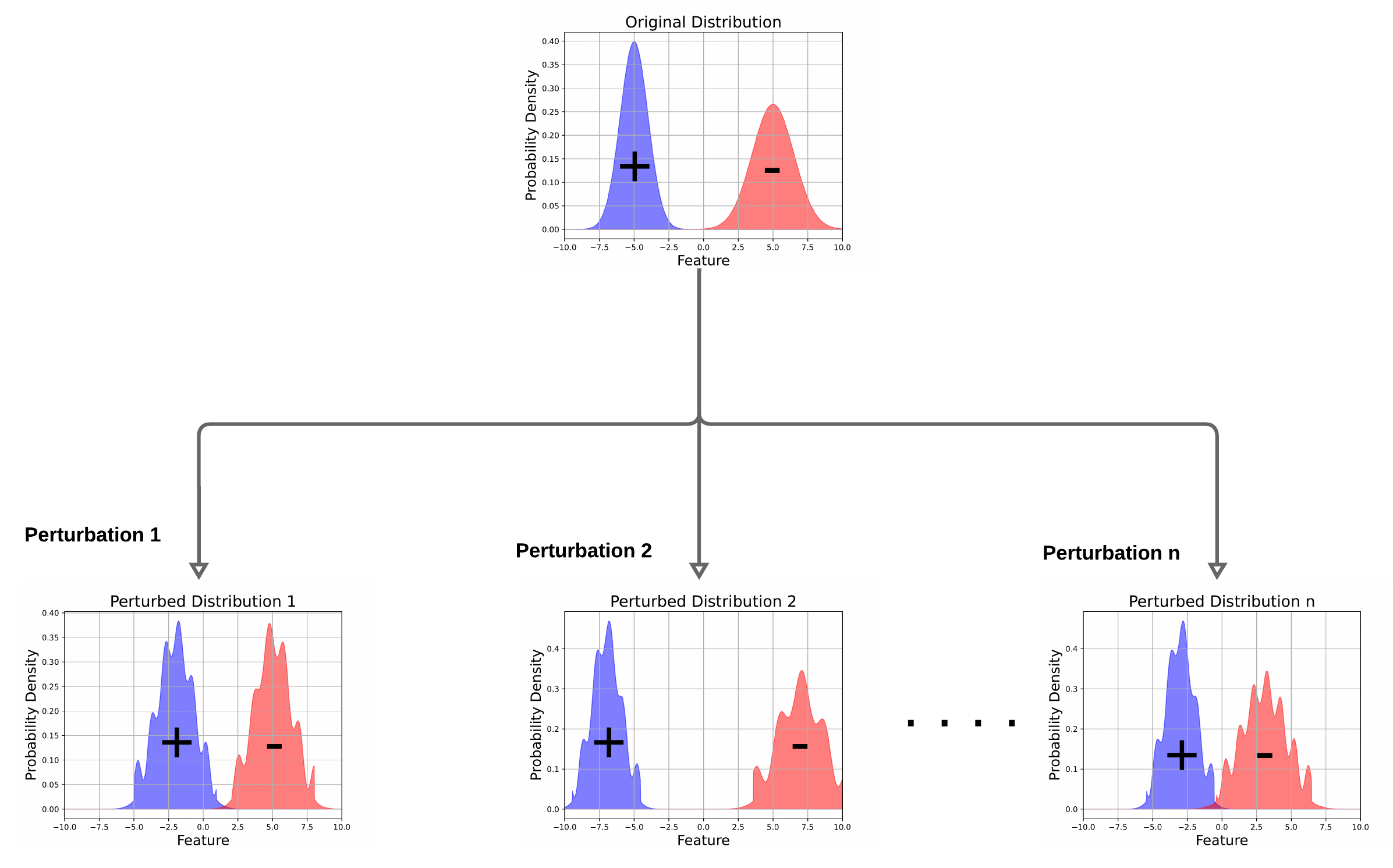}
\caption{Figure illustrating Distributionally Robust Optimization (DRO). In contrast to ERM which learns a model that minimizes expected loss over original data distribution, DRO learns a model that performs well simultaneously on several perturbed versions of the original data distribution.}
\label{fig:dro}
\end{figure}

\begin{figure}[bt!]
\centering
\includegraphics[width=0.6\linewidth]{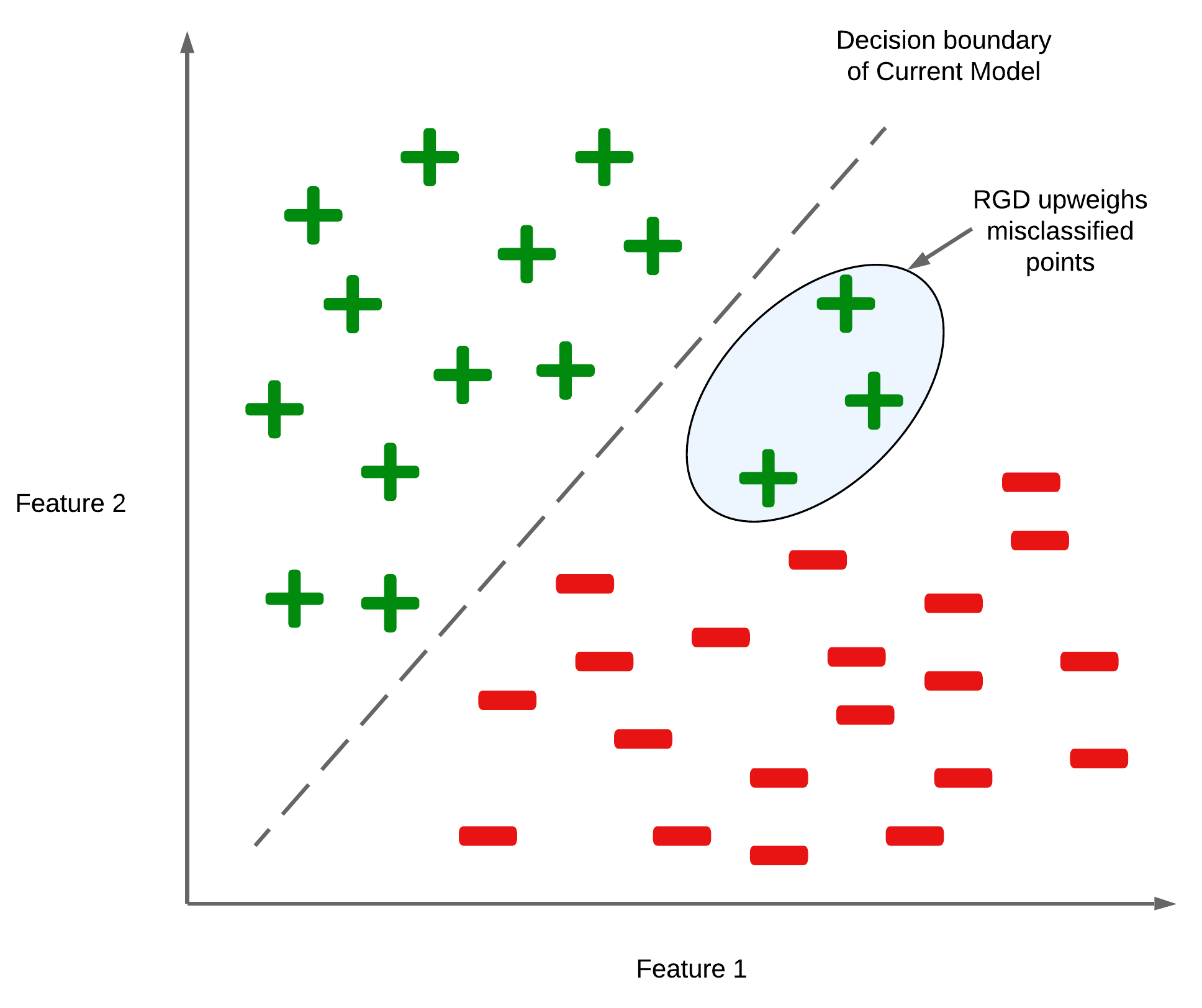}
\caption{Figure illustrating the intuitive idea behind the working of RGD in the binary classification setting. RGD upweights the points which have high losses - points which have been misclassified by the model.}
\label{fig:rgd}
\end{figure}

\section{Additional Experimental Results and Missing Details}
\label{appendix:addn_results}
This section provides additional experimental results and details that are missing in the main paper. The search space and hyperparameters tuned in our experiments of are depicted in Appendix~\ref{appendix:hyp_tuning}

\subsection{Hyperparameter Tuning}
\label{appendix:hyp_tuning}
In this section, we describe the common hyperparameter tuning space used across all experiments in our paper unless otherwise mentioned. The two parameters we tune were $\tau$ and $\texttt{lr}$. We use a simple grid search for $\tau$ in the order of $[1,3, 5, 7, 9]$ across the experiments where the scaling factor ($\gamma$) is by default set as $\frac{1}{\tau+1}$. This allowed our loss to be bounded between 0,1 and helped fairly compare \namex and \namel. The $\texttt{lr}$ was tuned by a proxy of $\texttt{lr\_mult}$ where we scaled the learning rate by a fraction in the range $[0.5,1.5]$. The effect of these hyperparameters ($\tau$, $\gamma$) is further depicted in Section~\ref{sec:ablation}.

\subsubsection{Existing KL-DRO techniques}
\label{appendix:ablation}

For TERM, we use the batch (non-hierarchical) version (as shown in Algorithm 1 of \cite{beirami2023tilted}) - requiring two degrees of hyperparameters (tilting coefficient $t$ and the learning rate ($lr$)). For the tilting coefficient, we use a search space of \cite{li2020tilted}: $\{0.2, 0.5, 1 , 3, 5\}$. For the learning rate, we use a lr multiplier to the baseline run as $\{0.7, 0.8, 0.9, 1, 1.1, 1.2, 1.3\}$. For the stochastic version of TERM, which is identical to ABSGD \citep{qi2020attentional} baseline - requires an additional coefficient of moving average ($\beta$). We use a grid search space of $\{0.25, 0.5, 0.75\}$ for tuning $\beta$, and tune the learning rate in a similar fashion to TERM. We also tune $\lambda$ (similar to the tilting coefficient of TERM) in the search space $\{1, 3, 5, 7\}$.

\subsection{Toy Experiments}
\label{sec:toy_example} 
In this section, we perform a simple experiment to better understand the robustness properties of our proposed approach.

\textbf{Linear Regression with rare features:}  We consider a linear regression problem where the covariates $x$ are sampled from  the set $\{h(0), \dots h(9)\}$. Here $h(\cdot)$ is a one-hot encoder that maps its inputs to a 10-dimensional vector. The label $y$ is generated according to the following linear model: $y = x^T\theta^*$, where $\theta^* \in \mathbb{R}^{10}$ is the regression vector which is sampled from a standard normal distribution. We construct an imbalanced dataset of tuples $\{(x_i, y_i)\}_{i=1}^n$ by taking 50 occurrences of the covariates $\{h(0)\dots h(4)\}$ and only one occurrence of the remaining five covariates. We consider two algorithms for learning the unknown parameter vector $\theta^*$: (i) SGD on the mean squared error in predictions (MSE) (ii) \namel~on the MSE loss. We set the step size to be $4$ for both algorithms and plot the evolution of  MSE and the Euclidean distance between the iterates ($\theta$) and the true parameter ($\theta^{*}$) (Figure~\ref{fig:toy_fixedlr}). It can be seen that our method achieves better performance due to prioritization of samples with higher loss, which corresponds to rare directions in the dataset. 

\begin{figure*}[hbt!]
\centering
\subfigure[MSE loss convergence]{%
\label{fig:mse_loss_lr4}
\includegraphics[width=0.30\linewidth]{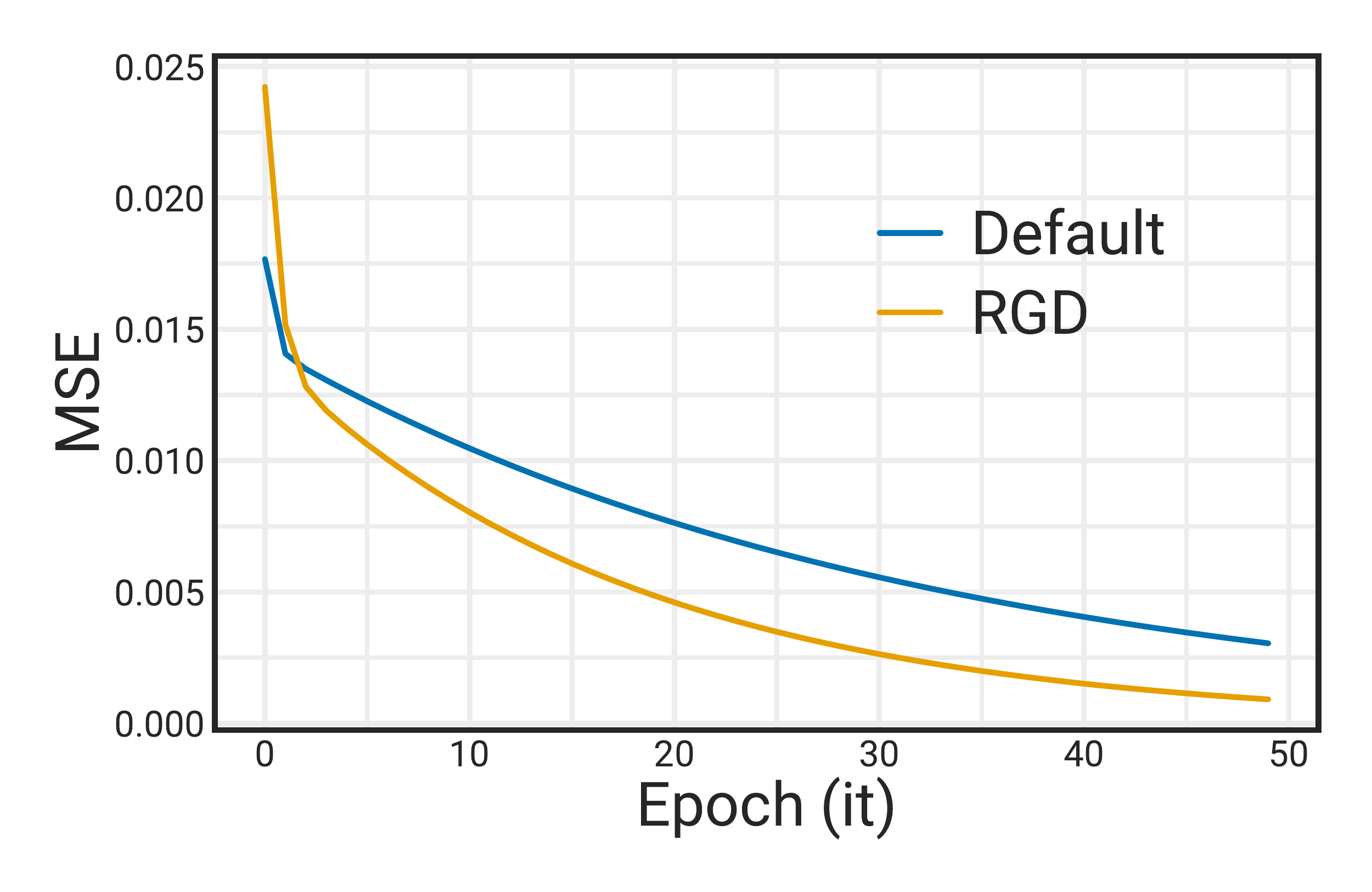}}
\quad
\subfigure[L2 Distance (Frequent feature)]{%
\label{fig:lr_4_good}
\includegraphics[width=0.30\linewidth]{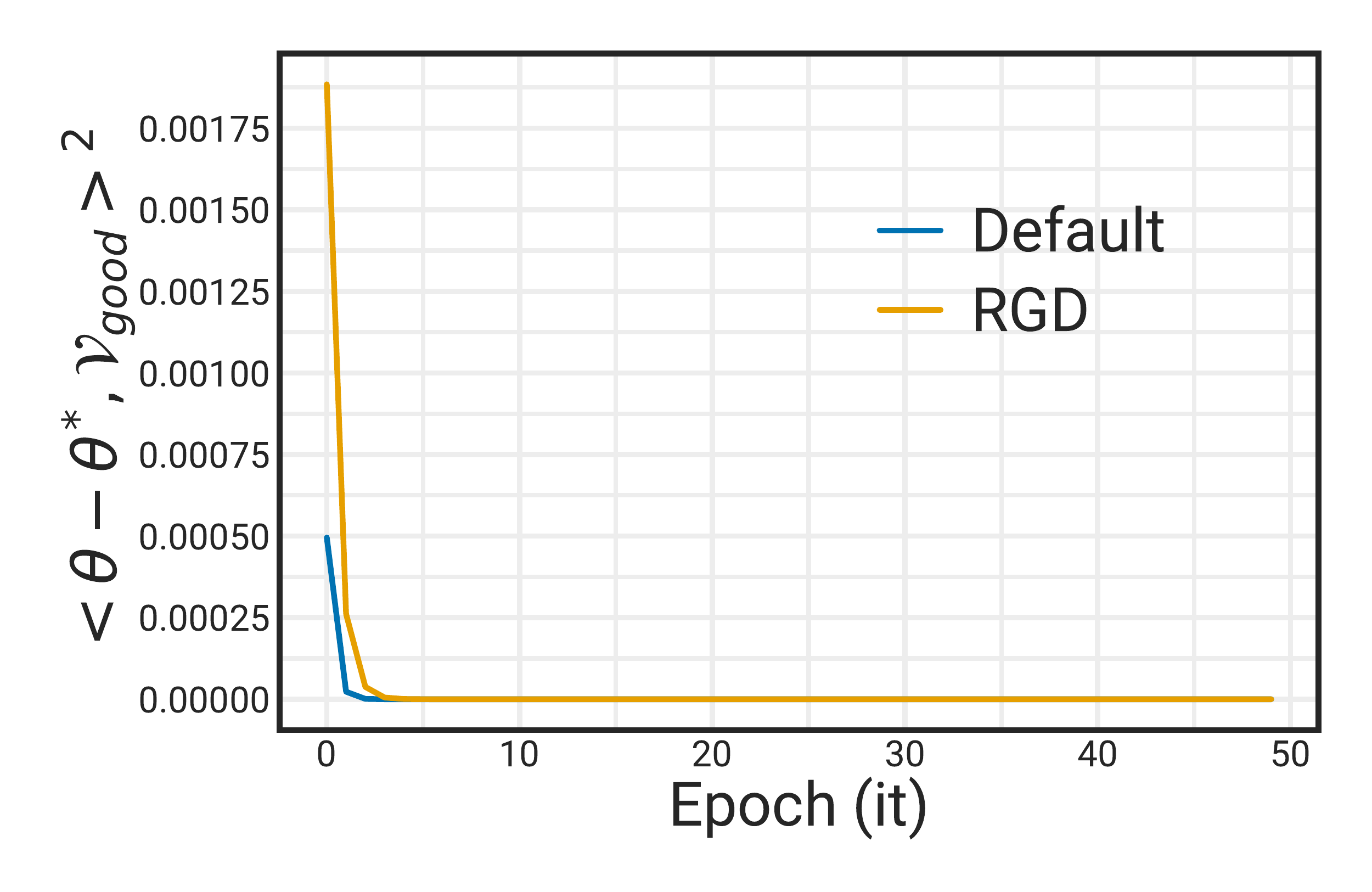}}
\quad
\subfigure[L2 Distance (Rare feature)]{%
\label{fig:lr_4_bad}
\includegraphics[width=0.30\linewidth]{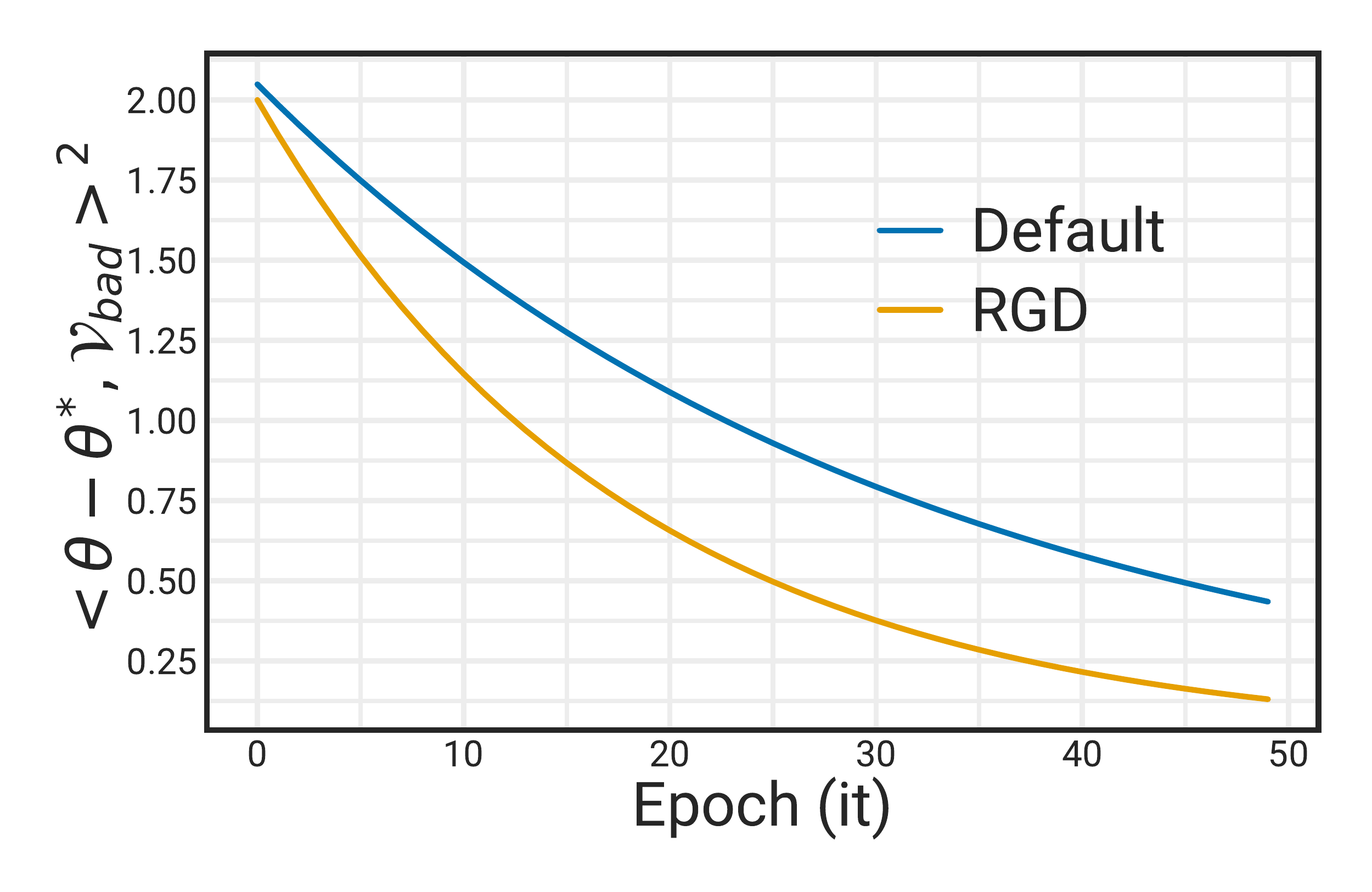}}

\caption{Figure~\ref{fig:mse_loss_lr4} showing the convergence of SGD, $\namel$ algorithms for estimating the linear regression parameter. The L2 distance between the iterates ($\theta$) and the true parameter $\theta^*$ is studied in \cref{fig:lr_4_good,fig:lr_4_bad}. Specifically, Figure~\ref{fig:lr_4_good} depicts the squared error in the frequently appearing directions, where all the techniques perform equally well. However, when it comes to learning rare directions, our proposed approach is much better (Figure~\ref{fig:lr_4_bad}).}
\label{fig:toy_fixedlr}
\end{figure*}

\subsection{Comparisons against Sharpness Aware Minimization}
\label{sam_comparison}

Below we compare with Sharpness-Aware Minimization (SAM)~\citep{foretsharpness} both from theoretical and empirical perspectives. From a theoretical perspective, SAM performs robust optimization in the weight space (that is SAM tries to learn a model that is robust to perturbations of weights). In contrast, RGD performs robust optimization in the distribution space. So, RGD and SAM are orthogonal to each other and can potentially be merged together to boost the performance. In Table~\ref{tab:rgd_sam_ablation}, we compare performance of various approaches for CIFAR-10 optimization. It can be seen that RGD marginally boosts the performance of SAM. Although, more thorough experiments are needed to understand the utility of RGD on top of SAM.

\begin{table}[hbt!]
\vspace{-0.2in}
\centering
\caption{Comparison of \namel~with SAM. Accuracy performance numbers have been reported over an average of 5 seeds. }
\vskip 0.1in
\resizebox{0.6\linewidth}{!}{
\begin{tabular}{lccc}
\toprule
\textbf{Algorithm} & SGD & SAM & SAM + RGD  \\
\midrule
Accuracy   & 96.80 \std{0.04}  & 97.31 \std{0.05}  & \textbf{97.41} \std{0.09}   \\
\bottomrule
\end{tabular}}
\label{tab:rgd_sam_ablation}
\vspace{-0.1in}
\end{table}

\subsection{Class Imbalance Experiments}
\label{inbalanced_cifar_additional_results}

\begin{figure*}[bt!]
\centering
\includegraphics[width=\linewidth]{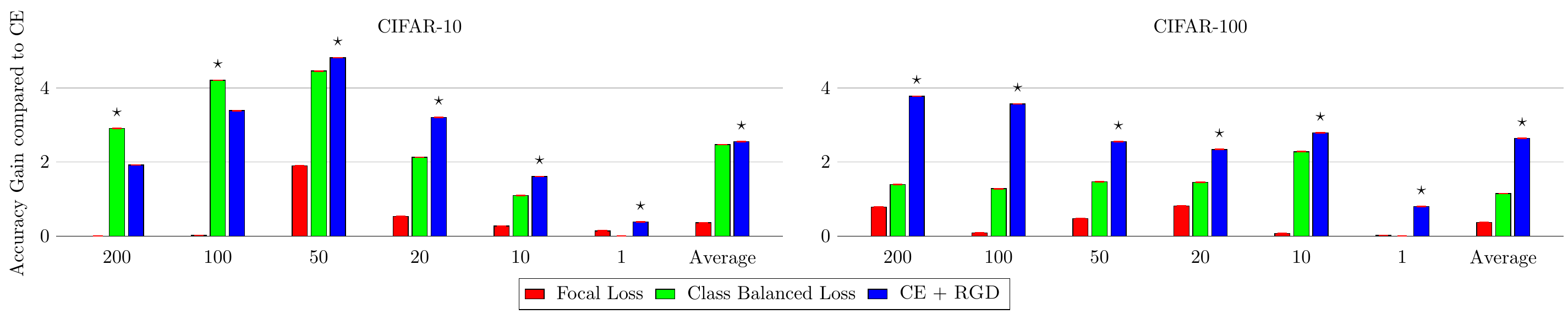}
\caption{Experiment comparing \namel~with baseline cross entropy Loss (CE), focal loss and class-balanced loss using a ResNet-32 backbone. $x$-axis represents the imbalance factor in the dataset.}
\label{fig:cifar_imbalaned}
\end{figure*}

This section briefly discusses additional results from our experiments on the Class Imbalance domain with imbalanced CIFAR-10 and CIFAR-100 datasets. 
It is well known that DRO outputs models with good tail performance~\citep{duchi2018learning}. Since \namel~directly solves the DRO objective, our models are also naturally endowed with this property.
To demonstrate this, we extend our experiments on linear regression to a more realistic image dataset, where some classes appear very rarely in the data set while some appear very frequently. We use the Long-Tailed CIFAR dataset, where we reduce the number of training samples per class according to an exponential function as proposed by \cite{cui2019class}. We define the imbalance factor of a dataset as the number of training samples in the largest class divided
by the smallest. Similar to the works of \cite{shu2019meta}, we use a ResNet-32 architecture for training.
Apart from Cross Entropy loss, we also include Focal Loss \citep{lin2017focal} and Class Balanced Loss \citep{cui2019class} as additional baselines.  We also experimented with the long-tailed CIFAR-100 dataset and showed that our proposed approach could again show significant improvements. Figure~\ref{fig:cifar_imbalaned} illustrates the performance of our approach in comparison to other state-of-the-art methods. Overall, in comparison to the SOTA approach in this task (Class Balanced Loss), our proposed approach brings about an improvement of \textcolor{blue}{\textbf{+0.79\%}}. 
A more comprehensive comparison with additional state-of-the-art baselines such as L2RW \citep{ren2018learning}, and Meta-Weight-Net \citep{shu2019meta} is illustrated in Table~\ref{tab:long_tailed_cifar_more_info}. Although these models use additional data as a meta-validation-set, our proposed approach outperforms L2RW and is roughly competitive with the Meta-Weight-Net model. 

Table~\ref{tab:long_tailed_cifar} depicts the accuracy metric of models on various levels of the imbalance factor. From Table~\ref{tab:long_tailed_cifar}, we show that our proposed approach \namee~outperforms other baselines such as Focal Loss and Class Balanced Loss by \textcolor{blue}{\textbf{+0.79\%}}. Furthermore, when models are trained on additional data, either by fine-tuning or by using a meta-learning framework to learn weights (such as Meta-Weight-Net and L2RW), we show that our proposed approach is competitively similar (\textcolor{blue}{\textbf{-0.22\%}}). Table~\ref{tab:long_tailed_cifar_more_info} illustrates this analysis further. The performance metrics of the baseline approaches were taken from \cite{shu2019meta}. Additional comparisons against other losses, such as cross-entropy with label smoothing and large margin softmax loss, are shown in Table~\ref{tab:lmargin_imbalance}.

For prior KL-DRO benchmarks such as TERM, we perform a grid-search where we tune the tilting coefficient ($t$) in the space $\{0.1, 0.3, 0.5, 0.7, 1, 2, 5\}$. Furthermore, we also tune the learning rate in the space $[5e-3, 1]$. Furthermore, for ABSGD, we replicate the baseline numbers from their paper which performs the same setup of experiments as us.

\begin{table*}[hbt!]
\caption{Test Accuracy of ResNet-32 on Long-Tailed CIFAR-10, and CIFAR-100 dataset.}
\vskip 0.15in
\resizebox{\linewidth}{!}{
\begin{tabular}{lccccccc|ccccccc}
\toprule
Dataset&
      \multicolumn{7}{c}{\data{CIFAR-10}}  &
      \multicolumn{7}{c}{\data{CIFAR-100}}   \\ \cline{2-8} \cline{9-15}
      
Loss / Imbalance Factor&
 \data{200} &
  \data{100} &
  \data{50} &
  \data{20} &
  \data{10} &
  \data{1} &
  \textbf{Avg.} &
  \data{200} &
  \data{100} &
  \data{50} &
  \data{20} &
  \data{10} &
  \data{1} &
  \textbf{Avg.} \\
\cline{1-15}
Focal Loss \cite{lin2017focal}& 65.29 & 70.38 & 76.71 &82.76 & 86.66 & 93.03 & 79.14
 & 35.62 & 38.41 & 44.32 & 51.95 & 55.78 & 70.52 & 49.43 \\
Class Balanced Loss \cite{cui2019class} & \textbf{68.89} & \textbf{74.57} &79.27 &84.36& 87.49& 92.89& 81.25 &
36.23& 39.60& 45.32& 52.59& 57.99& 70.50& 50.21\\
\midrule
\multicolumn{15}{l}{Cross Entropy (CE)}\\
\midrule
 Default  &   65.98  & 70.36  &  74.81 & 82.23  & 86.39 & 92.89 & 78.78 & 
34.84 &    38.32  & 43.85  & 51.14 & 55.71  & 70.50   & 49.06 \\
\namet~\textbf{(Ours)}     &    64.16  & 72.56  &  77.86 & 83.88  & 86.84 & 92.99 & 79.72 & 
36.22  & 39.87  & 43.74 & 51.86  & 56.9 & 70.80  & 49.90 \\
\namee~\textbf{(Ours)}    &   67.90  & 73.75  &  \textbf{79.63} & \textbf{85.44} & \textbf{88.00} & \textbf{93.27}  & \textbf{81.33} &  
\textbf{38.62}  & \textbf{41.89}  &  \textbf{46.40} & \textbf{53.48} & \textbf{58.5} & \textbf{71.30}  & \textbf{51.70} \\

\bottomrule
\end{tabular}}
\label{tab:long_tailed_cifar}

\end{table*}

\begin{table*}[hbt!]
\caption{Test Accuracy of ResNet-32 on Long-Tailed CIFAR-10, and CIFAR-100 dataset. We use the symbol \textcolor{red}{$\star$} to denote approaches that use additional data (as the meta-dataset). We use $\underline{underline}$ symbol to depict performances which are second-best across baselines. Our experiments show that we can get competitively similar performance to such models as well without training a second neural network.}
\vskip 0.15in
\resizebox{\linewidth}{!}{
\begin{tabular}{lccccccc|ccccccc}
\toprule
Dataset&
      \multicolumn{7}{c}{\data{CIFAR-10}}  &
      \multicolumn{7}{c}{\data{CIFAR-100}}   \\ \cline{2-8} \cline{9-15}
      
Loss / Imbalance Factor&
 \data{200} &
  \data{100} &
  \data{50} &
  \data{20} &
  \data{10} &
  \data{1} &
  \textbf{Avg.} &
  \data{200} &
  \data{100} &
  \data{50} &
  \data{20} &
  \data{10} &
  \data{1} &
  \textbf{Avg.} \\
\cline{1-15}
Fine-tuning \textcolor{red}{$\star$} & 66.08 & 71.33 &77.42 &83.37 &86.42 &\underline{93.23} & 79.64 & \underline{38.22} & 41.83& 46.40& 52.11& 57.44& 70.72& 51.12\\
L2RW \cite{ren2018learning} \textcolor{red}{$\star$} & 66.51& \underline{74.16} & 78.93& 82.12& 85.19& 89.25& 77.69& 33.38& 40.23 &44.44 &51.64 &53.73& 64.11& 47.92\\
Meta-Weight-Net \cite{shu2019meta} \textcolor{red}{$\star$} & \textbf{68.91} & \textbf{75.21} & \textbf{80.06} &\underline{84.94} & \underline{87.84} &92.66& \textbf{81.60} &37.91 &\textbf{42.09} & \textbf{46.74} & \textbf{54.37} & \underline{58.46} &  \underline{70.37} & \underline{51.65} \\
\midrule
\multicolumn{15}{l}{Cross Entropy (CE)}\\
\midrule
  Default  &   65.98  & 70.36  &  74.81 & 82.23  & 86.39 & 92.89 & 78.78 & 
34.84 &    38.32  & 43.85  & 51.14 & 55.71  & 70.50   & 49.06 \\
\namet~\textbf{(Ours)}     &    64.16  & 72.56  &  77.86 & 83.88  & 86.84 & 92.99 & 79.72 & 
36.22  & 39.87  & 43.74 & 51.86  & 56.9 & 70.80  & 49.90 \\
\namee~\textbf{(Ours)}   &   \underline{67.90}  & 73.75  &  \underline{79.63} & \textbf{85.44} & \textbf{88.00} & \textbf{93.27}  & \underline{81.33} &  
\textbf{38.62}  & \underline{41.89}  &  \underline{46.40} & \underline{53.48} & \textbf{58.5} & \textbf{71.30}  & \textbf{51.70} \\

\bottomrule
\end{tabular}}
\label{tab:long_tailed_cifar_more_info}

\end{table*}

\begin{table*}[hbt!]
\caption{Additional Ablation study comparing with Label smoothing and Large Margin Softmax Loss. Test Accuracy of ResNet-32 on Long-Tailed CIFAR-10, and CIFAR-100 dataset.}
\vskip 0.15in
\resizebox{\linewidth}{!}{
\begin{tabular}{lccccccc|ccccccc}
\toprule
Dataset&
      \multicolumn{7}{c}{\data{CIFAR-10}}  &
      \multicolumn{7}{c}{\data{CIFAR-100}}   \\ \cline{2-8} \cline{9-15}
      
Loss / Imbalance Factor&
 \data{200} &
  \data{100} &
  \data{50} &
  \data{20} &
  \data{10} &
  \data{1} &
  \textbf{Avg.} &
  \data{200} &
  \data{100} &
  \data{50} &
  \data{20} &
  \data{10} &
  \data{1} &
  \textbf{Avg.} \\
\cline{1-15}
\multicolumn{15}{l}{Cross Entropy (CE)}\\
\midrule
 + Label Smoothing  &   61.22  & 73.80  &  77.95 & 84.40  & 86.96 & 92.18 & 79.42 & 
37.14 & 41.05  & 44.76  & 50.67 & 57.74  & 70.97 & 50.39 \\
+ Large Margin Softmax (LMS)  &  \textbf{68.67} & 72.78 &  78.84 & 85.23  & \textbf{88.26} & 92.75 & 81.09 & 
36.77 &   40.38  & 45.24  & 51.25  & 56.9  & 71.01 & 50.26 \\
\namee~\textbf{(Ours)}   &   67.90  & \textbf{73.75}  &  \textbf{79.63} & \textbf{85.44} & 88.00 & \textbf{93.27}  & \textbf{81.33} &  
\textbf{38.62}  & \textbf{41.89}  &  \textbf{46.40} & \textbf{53.48} & \textbf{58.5} & \textbf{71.30}  & \textbf{51.70} \\

\bottomrule
\end{tabular}}
\label{tab:lmargin_imbalance}

\end{table*}

\subsection{Vanilla Classification}
\label{appendix:supervised_learning}

This section briefly discusses a few additional results from our experiments on standard supervised learning (in particular vanilla classification). Table~\ref{tab:bert_glue_additional} depicts the performance of our other variant \namet~in comparison to the baseline approach. 

\paragraph{EfficientNet finetuning.} We also show fine-tuning improvements of EfficientNet-v2-l over various tasks such as Cars and Food101 as depicted in Table~\ref{tab:vision_results}. In these experiments, we take a pre-trained EfficientNet backbone and fine-tune it for various tasks.

\paragraph{DieT-S for ImageNet-1K classification.} Similarly, we also report that \namel~boosts the performance of the baseline DeiT-S model by $0.1\%$ when trained from scratch on the Imagenet-1K benchmark as depicted in Table~\ref{tab:vision_results}. Note that similar to our setup in ViT-S on Imagenet-1K benchmark, we perform no tuning, and simply use $\tau=1$, and same learning rate as baseline.

\paragraph{MLP for classification.}  We also demonstrate that our proposed approach is simple and shows significant improvements, not only for SOTA approaches but also basic MLP procedures as depicted in Table~\ref{tab:met_multi_label_pure}, Table~\ref{tab:met_binary_accuracy_pure}, and Table~\ref{tab:met_binary_auroc_pure}. These tables help showcase that the simple addition of our proposed approach does show significant improvements of \textcolor{blue}{\textbf{+2.77\%}} (in accuracy) on multi-class and \textcolor{blue}{\textbf{+1.77\%}} (in AUROC) on binary class tasks respectively.

\begin{table*}[hbt!]
\centering
\caption{Additional Ablation study to showcase the gain achieved by using RGD for various tasks in GLUE benchmark for bert-base-uncased.}
\vskip 0.15in
\resizebox{0.6\linewidth}{!}{
\begin{tabular}{ccccccccc}
\toprule
\data{bert-base-uncased} & \data{MNLI}           & \data{QQP}            & \data{QNLI}           & \data{SST-2}          & \data{MRPC}           & \data{RTE}            & \data{COLA}           & \data{\textbf{Avg}}            \\
\midrule
Default           & 81.33          & 89.62          & 87.93          & 90.63          & \textbf{89.55} & 67.19          & 54.53          & 80.11          \\
\namet~\textbf{(Ours)}           & 82.97 & 89.87 & \textbf{90.79} & 91.28 & 88.54 & 71.23 & 59.28 & 81.97 \\
\namee~\textbf{(Ours)}           & \textbf{83.06} & \textbf{91.06} & 90.35 & \textbf{91.78} & 88.28          & \textbf{71.48} & \textbf{58.56} & \textbf{82.05} \\

\bottomrule
\end{tabular}}
\label{tab:bert_glue_additional}

\end{table*}

\begin{table*}[hbt!]
\caption{Ablation study to showcase the gain achieved by using RGD for various tasks in Vision benchmarks for ViT-S on Imagenet-1K, DeiT-S on Imagenet-1K and EfficientNet-v2-l on Finetuning tasks.}
\centering
\vskip 0.15in
\resizebox{0.6\linewidth}{!}{
\begin{tabular}{ccccc}
\toprule
{} &\data{Imagenet-1K (ViT-S)} &\data{Imagenet-1K (DeiT-S)}           & \data{Cars-FineTuning}  & \data{Food101-FineTuning}     \\
\midrule
Default           & 78.1  & 80.05     & 92.03   & 92.65     \\
\namee~\textbf{(Ours)}           & \textbf{79.0} & \textbf{80.13} & \textbf{92.62} & \textbf{92.75}   \\

\bottomrule
\end{tabular}}
\label{tab:vision_results}

\end{table*}

\paragraph{MiniGPT Pre-training.} We extend our work to large-scale tasks in NLP such as LLM pre-training which has become more prevalent over the recent years. Our experiments on  miniGPT pre-training, as illustrated in this section, showcase the efficacy of our approach in these settings. miniGPT~\citep{zhu2023minigpt} is a minimal implementation of a decoder-only transformer language model. We consider a 6 layer model and train on the lm1B small dataset which has 1B tokens. We trained for 100K steps with a batch size of 256. We used the default learning rate of 0.0016 for the baseline. For RGD, we fix the clipping threshold to 1 and tune the learning rate. We achieved 1\% improvement on the eval log-perplexity score. Table~\ref{tab:minigpt_train} illustrates our results in this setting.

\begin{table}[]
\centering
\caption{Results on standard Lm1b pre-training task with miniGPT. Off-the-hat addition of \namel~ leads to improvements of evaluation log-perplexity by $1\%$.}
\vskip 0.15in
\resizebox{0.4\linewidth}{!}{
\begin{tabular}{lcc}
\toprule
\textbf{Dataset} &
 \data{Lm1b eval Perplexity}\\
\midrule
Default      & 2.9218  \\
\namee~\textbf{(Ours)}     & 	\textbf{2.8979}   \\
\bottomrule
\end{tabular}}
\label{tab:minigpt_train}

\end{table}

\subsection{Tabular Classification}
\label{tabular_additional_results}

This section discusses a few additional results from our experiments on Tabular classification. Table~\ref{tab:met_multi_label_additional} depicts our proposed approach's accuracy compared to other baselines on multi-class tabular datasets. Our method outperforms previous SOTA in this problem by \textcolor{blue}{\textbf{+1.27\%}}. Furthermore, Table~\ref{tab:met_binary_auroc_additional} illustrates the AUROC score of our proposed approach in comparison to state-of-the-art baselines on binary-class tabular datasets. Our approach shows an improvement of \textcolor{blue}{\textbf{+1.5\%}} in this setting as well. The performance metrics of the baseline approaches were taken from \cite{majmundar2022met}. 

\begin{table}[]
\centering
\caption{Results on standard multi-class tabular datasets (Accuracy): The bottom partition shows results of our method with \namee\ loss. We show that the addition of our proposed approach significantly outperforms the base MLP model.}
\vskip 0.15in
\resizebox{0.6\linewidth}{!}{
\begin{tabular}{lccccc}
\toprule
\textbf{Algorithm} &
 \data{FMNIST} &
  \data{CIFAR10} &
  \data{MNIST} &
  \data{CovType} &
  \textbf{Avg.} \\
\midrule
Default \cite{majmundar2022met}   & 87.62       & 16.50          & 96.95         & 65.47       & 66.64  \\

\namet~\textbf{(Ours)} &  88.52     & 20.31    & \textbf{97.86}   & 68.81 &  68.875 \\
\namee~\textbf{(Ours)} &  \textbf{89.03} &   \textbf{21.32}   &   97.57 &  \textbf{69.73}  &  \textbf{69.41} \\

\bottomrule
\end{tabular}}
\label{tab:met_multi_label_pure}

\end{table}

\begin{table}[]
\centering
\caption{Results on standard binary-class tabular datasets (Accuracy): The bottom partition shows results of our method with \namee\ loss. We show that the addition of our proposed approach significantly outperforms the base MLP model.}
\vskip 0.15in
\resizebox{0.6\linewidth}{!}{
\begin{tabular}{lccccc}
\toprule
\textbf{Algorithm} &
 \data{Obesity} &
  \data{Income} &
  \data{Criteo} &
  \data{Thyroid} &
  \textbf{Avg.} \\
\midrule
Default      & 58.1      & 84.36        & 74.28 & 50 & 66.69 \\
\namet~\textbf{(Ours)}  &  59.12     &  84.5   &  75.06   &  \textbf{57.3} & 69  \\
\namee~\textbf{(Ours)}  &  \textbf{58.83} &  \textbf{85.5}  &   \textbf{75.87}   & 56.04 & \textbf{69.1}   \\
\bottomrule
\end{tabular}}
\label{tab:met_binary_accuracy_pure}

\end{table}

\begin{table}[]
\centering
\caption{Results on standard binary-class tabular datasets (AUROC): The bottom partition shows results of our method with \namee\ loss. We show that the addition of our proposed approach significantly outperforms the base MLP model.}
\vskip 0.15in
\resizebox{0.6\linewidth}{!}{
\begin{tabular}{lccccc}
\toprule
\textbf{Algorithm} &
 \data{Obesity} &
  \data{Income} &
  \data{Criteo} &
  \data{Thyroid} &
  \textbf{Avg.} \\
\midrule
MLP      & 52.3  & 89.39 & 79.82 & 62.3 & 70.95 \\
\namet~\textbf{(Ours)} &  54.31     & \textbf{91.1}   & 79.85 & 62.25 & 71.88\\
\namee~\textbf{(Ours)} &  \textbf{55.96} &  90.8  &  \textbf{80.1}  &  \textbf{64} &  \textbf{72.72}\\
\bottomrule
\end{tabular}}
\label{tab:met_binary_auroc_pure}

\end{table}

\begin{table}[]
\centering
\caption{Results on standard multi-class tabular datasets (Accuracy): The bottom partition shows results of our method with \namel\ loss. We show that the addition of our proposed approach significantly outperforms existing methods, as well as SOTA.}
\vskip 0.15in
\resizebox{0.7\linewidth}{!}{
\begin{tabular}{lccccc}
\toprule
\textbf{Algorithm} &
 \data{FMNIST} &
  \data{CIFAR10} &
  \data{MNIST} &
  \data{CovType} &
  \textbf{Avg.} \\
\midrule
MLP      & 87.62       & 16.50          & 96.95         & 65.47       & 66.64  \\
RF \cite{breiman2001random}     & 88.43       & 42.73          & 97.62         & 71.37       & 75.04  \\
GBDT \cite{friedman2001greedy}  & 88.71       & 45.7          & \textbf{100}         & 72.96      & 76.84  \\
RF-G   \cite{rahimi2008weighted}   & 89.84       & 29.32         & 97.65         & 71.57       & 72.10  \\
MET-R  \cite{majmundar2022met}    & 88.84       & 28.94          & 97.44         & 69.68       & 71.23  \\
VIME  \cite{yoon2020vime}   & 80.36       & 34.00          & 95.77         & 62.80       & 68.23 \\
DACL+  \cite{verma2021towards}    & 81.40    & 39.70          & 91.40         & 64.23       & 69.18 \\
SubTab  \cite{ucar2021subtab}    & 87.59      & 39.34          & 98.31         & 42.36       & 66.90  \\
TabNet \cite{arik2019tabnet}    & 88.18      & 33.75          & 96.63        & 65.13      & 70.92 \\
MET \cite{majmundar2022met} & \textbf{91.68} & 47.82 & 99.19 & 76.71 & 78.85\\
 \midrule
\multicolumn{6}{l}{MET-S}\\
\midrule
Default \cite{majmundar2022met}   & 90.94       & 48.00          & 99.01         & 74.11       & 78.02 \\
\namet~\textbf{(Ours)} &  91.12     & 49.17    & 99.28   & 79.41 &  79.75 \\
\namee~\textbf{(Ours)} &  91.54 &      \textbf{49.54}   &   99.69 &  \textbf{79.72}  &  \textbf{80.12} \\
\bottomrule
\end{tabular}}
\label{tab:met_multi_label_additional}

\end{table}

\begin{table}[]
\centering
\caption{Results on standard binary-class tabular datasets (AUROC): The bottom partition shows results of our method with \namel\ loss. We show that the addition of our proposed approach significantly outperforms existing methods, as well as SOTA.}
\vskip 0.15in
\resizebox{0.7\linewidth}{!}{
\begin{tabular}{lccccc}
\toprule
\textbf{Algorithm} &
 \data{Obesity} &
  \data{Income} &
  \data{Criteo} &
  \data{Thyroid} &
  \textbf{Avg.} \\
\midrule
MLP      & 52.3      & 89.39       & 79.82              & 62.3 & 70.95 \\
RF \cite{breiman2001random}     & 64.36    & 91.53        & 77.57             & 99.62 & 83.27 \\
GBDT  \cite{friedman2001greedy}     & 64.4  & 92.5       & 78.77            & 99.34 & 83.75\\
RF-G  \cite{rahimi2008weighted}    & 54.45    & 90.09    & 80.32             & 52.65  & 69.37\\
MET-R \cite{majmundar2022met}     & 53.2   & 83.54 & 79.17         & 82.03  & 74.49 \\
VIME  \cite{yoon2020vime}  & 57.27    & 87.37    & 74.28   & 94.87 & 78.45\\
DACL+ \cite{verma2021towards}    & 61.18   & 89.01     & 75.32          & 86.63 & 78.04\\
SubTab \cite{ucar2021subtab}  & 64.92   & 88.95    & 76.57             & 88.93 & 79.00 \\
TabNet \cite{arik2019tabnet} & 69.40     & 77.30          & 80.91              & 96.98 & 81.15 \\
 \midrule
\multicolumn{6}{l}{MET-S}\\
\midrule
 Default \cite{majmundar2022met} & 71.84    & 93.85 & 86.17    & 99.81 & 87.92 \\
\namet~\textbf{(Ours)}  &  76.23     &  93.90   & 86.92 & 99.82 & 89.22  \\
\namee~\textbf{(Ours)}  &  \textbf{76.87} &  \textbf{93.96}       &  \textbf{86.98}  &  \textbf{99.92} & \textbf{89.43} \\
\bottomrule
\end{tabular}}
\label{tab:met_binary_auroc_additional}

\end{table}

\subsection{Meta-Learning}
\label{appendix:meta_learning}

This section discusses some additional results from our experiments in the meta-learning domain. Table~\ref{tab:meta_learning} depicts a complete table and comparison of our proposed approach on the MAML baseline compared to others. Overall, we notice improvement across the board, especially in the outliers, as shown in the Worst-k\% metrics. Note that although our \namel~has been applied on MAML~\citep{finn2017model} in our current experiments, our approach is analogous to the model and can be extended to other meta-learning techniques such as Protonet~\citep{snell2017prototypical}, CNAPs~\citep{requeima2019fast}, etc. as well. The performance metrics of the baseline approaches were taken from \citep{kumar2022effect}.

\begin{table}[]
\caption{Results on meta-learning datasets. We report the Worst-K\% performance as well to help study the performance distribution over all tasks. Overall, we expect our reweighting scheme to give more importance to those tasks which are difficult and rare. We show that the addition of our proposed approach significantly outperforms existing methods as shown in Omniglot 5-way 1-shot, as well as \textit{mini}ImageNet 5-way 1-shot setting.}
\vskip 0.15in
\resizebox{\linewidth}{!}{
\begin{tabular}{lcccccc}
\toprule
\textbf{Algorithm} &
 \data{Worst 10\%} &
 \data{Worst 20\%} &
  \data{Worst 30\%} &
  \data{Worst 40\%} &
  \data{Worst 50\%} &
  \data{Overall} \\
\midrule
\multicolumn{7}{l}{Omniglot 5-way 1-shot}\\
\midrule
MAML  & 91.71 \std{0.73}	& 94.16 \std{0.50}	& 95.41 \std{0.39} &	96.22 \std{0.32} &	96.76 \std{0.27} & 98.38 \std{0.17} \\
Reptile & 82.78 \std{0.85} &	86.22 \std{0.64} &	88.33 \std{0.54} &	89.79 \std{0.48} &	90.93 \std{0.43} &	94.64 \std{0.32} \\
Protonet & 88.72 \std{0.99} &	92.24 \std{0.70} &	93.95 \std{0.54} &	95.06 \std{0.44} &	95.79 \std{0.38} &	97.82 \std{0.23} \\
Matching Networks & 79.70 \std{0.95} &	84.01 \std{0.78} &	86.78 \std{0.68} &	88.83 \std{0.62} &	90.41 \std{0.56} &	94.71 \std{0.39} \\
MAML + \namee  & \textbf{92.14} \std{0.84} &	\textbf{94.54} \std{0.53} &	\textbf{95.72} \std{0.40} &	\textbf{96.46} \std{0.33} &	\textbf{96.90} \std{0.27} &  \textbf{98.45} \std{0.17} \\
\midrule
\multicolumn{7}{l}{Omniglot 20-way 1-shot}\\
\midrule
MAML  & 84.33 \std{0.40} &	85.86 \std{0.29} &	86.92 \std{0.26} &	87.73 \std{0.24} &	88.42 \std{0.22} & 91.28 \std{0.22} \\
Reptile & 83.13 \std{0.42} &	84.71 \std{0.31} &	85.77 \std{0.26} &	86.60 \std{0.24} &	87.30 \std{0.23} &	90.09 \std{0.22} \\
Protonet & \textbf{87.19} \std{0.33} &	\textbf{88.71} \std{0.27} &	\textbf{89.73} \std{0.24} &	\textbf{90.54} \std{0.23} &	\textbf{91.20} \std{0.22} &	\textbf{93.72} \std{0.20} \\
Matching Networks & 62.82 \std{0.60} &	65.50 \std{0.48} &	67.25 \std{0.42} &	68.61 \std{0.39} &	69.75 \std{0.37} &	74.62 \std{0.38} \\
MAML + \namee  & 86.61 \std{0.36} &	88.09 \std{0.28} &	89.09 \std{0.24} &	89.87 \std{0.23} &	90.50 \std{0.21} & 93.01 \std{0.20} \\

\midrule
\multicolumn{7}{l}{\textit{mini}ImageNet 5-way 1-shot}\\
\midrule
MAML  & 30.94 \std{0.70} &	34.52 \std{0.62} &	36.93 \std{0.57} &	38.94 \std{0.55} &	40.68 \std{0.53} & 48.86 \std{0.62} \\
Reptile & 25.37 \std{0.74} &	28.59 \std{0.59} &	30.71 \std{0.52} &	32.52 \std{0.50} &	34.11 \std{0.48} &	41.42 \std{0.56} \\
Protonet & 30.93 \std{0.76} &	34.62 \std{0.65} &	37.06 \std{0.58} &	38.94 \std{0.54} &	40.66 \std{0.52} &	48.56 \std{0.60} \\
Matching Networks & 27.19 \std{0.68} &	30.42 \std{0.57} &	32.64 \std{0.52} &	34.45 \std{0.50} &	36.10 \std{0.49} &	43.84 \std{0.58} \\
MAML + \namee  & \textbf{33.33} \std{0.90} &	\textbf{36.67} \std{0.65} &	\textbf{39.12} \std{0.59} &	\textbf{41.20} \std{0.56} &	\textbf{42.96} \std{0.55} & \textbf{51.21} \std{0.63}\\

\bottomrule
\end{tabular}}
\label{tab:meta_learning}

\end{table}

\subsection{DomainBed}
\label{domain_bed_additional_results}

\begin{table}[]
\caption{Results on DomainBed (Model selection: training-domain validation set): The bottom partition shows results of our method with \namel\ loss. In both cases, with (top) and without (bottom) fixed linear layer, the proposed approach outperforms existing methods, as well as SOTA.}
\vskip 0.15in
\resizebox{\linewidth}{!}{
\begin{tabular}{lccccl}
\toprule
\textbf{Algorithm} &
 \data{PACS} &
  \data{VLCS} &
  \data{OfficeHome} &
  \data{DomainNet} &
  \textbf{Avg.} \\
\midrule
ERM   \cite{gulrajani2020search}   & 85.5   \std{0.1}       & 77.5 \std{0.4}          & 66.5 \std{0.2}         & 40.9 \std{0.1}               & 67.6   \\
IRM \cite{arjovsky2019invariant} & 83.5 \std{0.8}          & 78.5 \std{0.5}          & 64.3 \std{2.2}           & 33.9 \std{2.8}          & 65.1          \\
GroupDRO \cite{sagawa2019distributionally} & 84.4 \std{0.8}  & 76.7 \std{0.6}       & 66.0 \std{0.7}            & 33.3 \std{0.2} & {65.1} \\
Mixup  \cite{yan2020improve} & 84.6 \std{0.6}    &  77.4 \std{0.6}       & 68.1 \std{0.3}               & 39.2 \std{0.1} & {67.33} \\
MLDG   \cite{li2018learning}  & 84.9 \std{1.0}    & 77.2 \std{0.4}       & 66.8 \std{0.6}         & 41.2 \std{0.1} & {67.53} \\
CORAL \cite{sun2016deep} & 86.2\std{0.3}          & 78.8\std{0.6}          & 68.7\std{0.3}      & 41.5 \std{0.1}          & 68.8 \\
MMD   \cite{li2018domain}   & 84.6 \std{0.5}    & 77.5 \std{0.9}              & 66.3 \std{0.1}        & 23.4 \std{9.5} & {62.95}\\
DANN   \cite{ganin2016domain}  & 83.6 \std{0.4}   & 78.6 \std{0.4}                & 65.9 \std{0.6}            & 38.3 \std{0.1} & 66.6 \\
CDANN \cite{li2018deep}     & 82.6 \std{0.9}     & 77.5 \std{0.1}              & 65.8 \std{1.3}             & 38.3 \std{0.3} & 66.05 \\
MTL     \cite{blanchard2021domain}    & 84.6 \std{0.5}     & 77.2 \std{0.4}          & 66.4 \std{0.5}     & 40.6 \std{0.1} & {67.2}\\
SagNet \cite{nam2021reducing}        & 86.3 \std{0.2}    & 77.8 \std{0.5}           & 68.1 \std{0.1}             & 40.3 \std{0.1} & {68.13}\\
ARM   \cite{zhang2021adaptive}     & 85.1 \std{0.4}    & 77.6 \std{0.3}            & 64.8 \std{0.3}      & 35.5 \std{0.2} & {65.75}\\
VREx   \cite{krueger2021out}     & 84.9 \std{0.6}    & 78.3 \std{0.2}            & 66.4 \std{0.6}      & 33.6 \std{2.9} & 65.8 \\
RSC     \cite{huang2020self}    & 85.2 \std{0.9}    & 77.1 \std{0.5}            & 65.5 \std{0.9}           & 38.9 \std{0.5} & {66.68}\\
MIRO \cite{cha2022domain} & 85.4 \std{0.4}          & {79.0 \std{0.0}} & {70.5 \std{0.4}}  & {44.3 \std{0.2}} & {69.8}  \\
\midrule
\multicolumn{6}{l}{ERM + FRR-L}\\
\midrule
Default \cite{addepalli2022learning} & 85.7 \std{0.1}  & 76.6 \std{0.2} & 68.4 \std{0.2} & 44.2 \std{0.1} &  68.73  \\
\namet~\textbf{(Ours)} & 87.2 \std{0.3}  & 78.6 \std{0.3} & 69.4 \std{0.2} & 45.8 \std{0.0} & 70.25 \\
\namee~\textbf{(Ours)} &  \textbf{87.6} \std{0.3} & 78.6 \std{0.3} & \textbf{69.8} \std{0.2} & \textbf{46.0} \std{0.0} & \textbf{70.48} \\
 \midrule
\multicolumn{6}{l}{ERM + FRR}\\
\midrule
 Default \cite{addepalli2022learning} & 87.5 \std{0.1} & 77.6 \std{0.3} & 69.4 \std{0.1}  & 45.1 \std{0.1}  & 69.9
 \\
\namet~\textbf{(Ours)} & 87.6 \std{0.3}  & 78.1 \std{0.1} & \textbf{69.9} \std{0.1} & 45.8 \std{0.0} & 70.35 \\
 \namee~\textbf{(Ours)} & \textbf{88.2} \std{0.2} & 78.6 \std{0.3} & 69.8 \std{0.2}& 45.8 \std{0.0} & \textbf{70.6}\\

\bottomrule
\end{tabular}}
\label{tab:domain_bed_additional}

\end{table}

\begin{table}[hbt!]
\centering
\small
\renewcommand{\arraystretch}{1.1}
\caption{\small\textbf{Out-of-domain accuracies (\%) on} \data{PACS}\textbf{.}}
\vskip 0.15in
\begin{tabular}{lllllc}
\toprule
\textbf{Algorithm} & \textbf{A} & \textbf{C} & \textbf{P} & \textbf{S} & \textbf{Avg} \\
\midrule
CDANN & 84.6 \std{1.8} & 75.5 \std{0.9} & 96.8 \std{0.3} & 73.5 \std{0.6} & 82.6 \\
MASF & 82.9 & 80.5 & 95.0 & 72.3 & 82.7 \\
DMG & 82.6 & 78.1 & 94.5 & 78.3 & 83.4 \\
IRM & 84.8 \std{1.3} & 76.4 \std{1.1} & 96.7 \std{0.6} & 76.1 \std{1.0} & 83.5 \\
MetaReg & 87.2 & 79.2 & 97.6 & 70.3 & 83.6 \\
DANN & 86.4 \std{0.8} & 77.4 \std{0.8} & 97.3 \std{0.4} & 73.5 \std{2.3} & 83.7 \\
GroupDRO & 83.5 \std{0.9} & 79.1 \std{0.6} & 96.7 \std{0.3} & 78.3 \std{2.0} & 84.4 \\
MTL & 87.5 \std{0.8} & 77.1 \std{0.5} & 96.4 \std{0.8} & 77.3 \std{1.8} & 84.6 \\
I-Mixup & 86.1 \std{0.5} & 78.9 \std{0.8} & 97.6 \std{0.1} & 75.8 \std{1.8} & 84.6 \\
MMD & 86.1 \std{1.4} & 79.4 \std{0.9} & 96.6 \std{0.2} & 76.5 \std{0.5} & 84.7 \\
VREx & 86.0 \std{1.6} & 79.1 \std{0.6} & 96.9 \std{0.5} & 77.7 \std{1.7} & 84.9 \\
MLDG & 85.5 \std{1.4} & 80.1 \std{1.7} & 97.4 \std{0.3} & 76.6 \std{1.1} & 84.9 \\
ARM & 86.8 \std{0.6} & 76.8 \std{0.5} & 97.4 \std{0.3} & 79.3 \std{1.2} & 85.1 \\
RSC & 85.4 \std{0.8} & 79.7 \std{1.8} & 97.6 \std{0.3} & 78.2 \std{1.2} & 85.2 \\
Mixstyle & 86.8 \std{0.5} & 79.0 \std{1.4} & 96.6 \std{0.1} & 78.5 \std{2.3} & 85.2 \\
ER & 87.5 & 79.3 & 98.3 & 76.3 & 85.3 \\
pAdaIN & 85.8 & 81.1 & 97.2 & 77.4 & 85.4 \\
ERM & 84.7 \std{0.4} & 80.8 \std{0.6} & 97.2 \std{0.3} & 79.3 \std{1.0} & 85.5 \\
EISNet & 86.6 & 81.5 & 97.1 & 78.1 & 85.8 \\
CORAL & 88.3 \std{0.2} & 80.0 \std{0.5} & 97.5 \std{0.3} & 78.8 \std{1.3} & 86.2 \\
SagNet & 87.4 \std{1.0} & 80.7 \std{0.6} & 97.1 \std{0.1} & 80.0 \std{0.4} & 86.3 \\
DSON & 87.0 & 80.6 & 96.0 & 82.9 & 86.6 \\
 \midrule
\multicolumn{6}{l}{ERM + FRR-L}\\
\midrule
Default & 83.2 \std{0.3} & 79.8 \std{0.4} & 95.9 \std{0.3} & 83.5 \std{0.4} & 85.7 \\
\namet & 88.7 \std{0.5}  & 83.0 \std{0.5}   &    97.8 \std{0.1}      &   79.4  \std{1.0}  & 87.2 \\
\namee & 88.4 \std{0.3} & 83.3 \std{0.8} & 97.5 \std{0.3} & 81.1  \std{0.5} & \textbf{87.6}\\
 \midrule
\multicolumn{6}{l}{ERM + FRR}\\
\midrule
Default & 86.8 \std{0.3} & 82.2 \std{0.4} & 96.4 \std{0.1} & 84.5 \std{0.2} & 87.5 \\
\namet &   87.7 \std{0.8} &  84.0 \std{0.6}   &  97.6 \std{0.1}    &   81.2 \std{0.5}   & 87.6 \\
\namee & 88.8 \std{0.3} &  84.0 \std{0.8}  &  97.7 \std{0.1}  &  82.4 \std{0.6} & \textbf{88.2}\\
\bottomrule
\end{tabular}
\label{pacs_overall}
\end{table}

\begin{table}[hbt!]
\centering
\small
\renewcommand{\arraystretch}{1.1}
\caption{\small\textbf{Out-of-domain accuracies (\%) on} \data{VLCS}\textbf{.}}
\vskip 0.15in
\begin{tabular}{lllllc}
\toprule
\textbf{Algorithm} & \textbf{C} & \textbf{L} & \textbf{S} & \textbf{V} & \textbf{Avg} \\
\midrule
GroupDRO & 97.3 \std{0.3} & 63.4 \std{0.9} & 69.5 \std{0.8} & 76.7 \std{0.7} & 76.7 \\
RSC & 97.9 \std{0.1} & 62.5 \std{0.7} & 72.3 \std{1.2} & 75.6 \std{0.8} & 77.1 \\
MLDG & 97.4 \std{0.2} & 65.2 \std{0.7} & 71.0 \std{1.4} & 75.3 \std{1.0} & 77.2 \\
MTL & 97.8 \std{0.4} & 64.3 \std{0.3} & 71.5 \std{0.7} & 75.3 \std{1.7} & 77.2 \\
I-Mixup & 98.3 \std{0.6} & 64.8 \std{1.0} & 72.1 \std{0.5} & 74.3 \std{0.8} & 77.4 \\
ERM & 97.7 \std{0.4} & 64.3 \std{0.9} & 73.4 \std{0.5} & 74.6 \std{1.3} & 77.5 \\
MMD & 97.7 \std{0.1} & 64.0 \std{1.1} & 72.8 \std{0.2} & 75.3 \std{3.3} & 77.5 \\
CDANN & 97.1 \std{0.3} & 65.1 \std{1.2} & 70.7 \std{0.8} & 77.1 \std{1.5} & 77.5 \\
ARM & 98.7 \std{0.2} & 63.6 \std{0.7} & 71.3 \std{1.2} & 76.7 \std{0.6} & 77.6 \\
SagNet & 97.9 \std{0.4} & 64.5 \std{0.5} & 71.4 \std{1.3} & 77.5 \std{0.5} & 77.8 \\
Mixstyle & 98.6 \std{0.3} & 64.5 \std{1.1} & 72.6 \std{0.5} & 75.7 \std{1.7} & 77.9 \\
VREx & 98.4 \std{0.3} & 64.4 \std{1.4} & 74.1 \std{0.4} & 76.2 \std{1.3} & 78.3 \\
IRM & 98.6 \std{0.1} & 64.9 \std{0.9} & 73.4 \std{0.6} & 77.3 \std{0.9} & 78.6 \\
DANN & 99.0 \std{0.3} & 65.1 \std{1.4} & 73.1 \std{0.3} & 77.2 \std{0.6} & 78.6 \\
CORAL & 98.3 \std{0.1} & 66.1 \std{1.2} & 73.4 \std{0.3} & 77.5 \std{1.2} & 78.8 \\
 \midrule
\multicolumn{6}{l}{ERM + FRR-L}\\
\midrule
Default & 97.1 \std{0.2} & 63.3 \std{0.3} & 72.0 \std{0.3} & 74.3 \std{0.3} & 76.6 \\
\namet &  98.8     \std{0.1}    &   64.8  \std{0.2}  &  73.9   \std{0.2}   &   77.0 \std{1.1}    & 78.6 \\
\namee & 98.9 \std{0} & 64.9 \std{0.4} & 73.2 \std{0.4} &  77.5 \std{0.6} & 78.6 \\
 \midrule
\multicolumn{6}{l}{ERM + FRR}\\
\midrule
Default & 96.7 \std{0.6} & 65.2 \std{0.8} & 73.4 \std{0.1} & 75.6 \std{0.4} & 77.6 \\
\namet & 98.3 \std{0.1} & 64.5 \std{0.2} & 72.3 \std{0.1} & 77.2 \std{0.3} & 78.1 \\
\namee & 97.1 \std{0.5}& 65.4 \std{0.8} & 74.3 \std{0.1} & 77.5 \std{0.3} & 78.6 \\
\bottomrule
\end{tabular}
\label{vlcs_overall}
\end{table}

\begin{table}[hbt!]
\centering
\small
\renewcommand{\arraystretch}{1.1}
\caption{\small\textbf{Out-of-domain accuracies (\%) on} \data{OfficeHome}\textbf{.}}
\vskip 0.15in
\begin{tabular}{lllllc}
\toprule
\textbf{Algorithm} & \textbf{A} & \textbf{C} & \textbf{P} & \textbf{R} & \textbf{Avg} \\
\midrule
Mixstyle & 51.1 \std{0.3} & 53.2 \std{0.4} & 68.2 \std{0.7} & 69.2 \std{0.6} & 60.4 \\
IRM & 58.9 \std{2.3} & 52.2 \std{1.6} & 72.1 \std{2.9} & 74.0 \std{2.5} & 64.3 \\
ARM & 58.9 \std{0.8} & 51.0 \std{0.5} & 74.1 \std{0.1} & 75.2 \std{0.3} & 64.8 \\
RSC & 60.7 \std{1.4} & 51.4 \std{0.3} & 74.8 \std{1.1} & 75.1 \std{1.3} & 65.5 \\
CDANN & 61.5 \std{1.4} & 50.4 \std{2.4} & 74.4 \std{0.9} & 76.6 \std{0.8} & 65.7 \\
DANN & 59.9 \std{1.3} & 53.0 \std{0.3} & 73.6 \std{0.7} & 76.9 \std{0.5} & 65.9 \\
GroupDRO & 60.4 \std{0.7} & 52.7 \std{1.0} & 75.0 \std{0.7} & 76.0 \std{0.7} & 66.0 \\
MMD & 60.4 \std{0.2} & 53.3 \std{0.3} & 74.3 \std{0.1} & 77.4 \std{0.6} & 66.4 \\
MTL & 61.5 \std{0.7} & 52.4 \std{0.6} & 74.9 \std{0.4} & 76.8 \std{0.4} & 66.4 \\
VREx & 60.7 \std{0.9} & 53.0 \std{0.9} & 75.3 \std{0.1} & 76.6 \std{0.5} & 66.4 \\
ERM & 61.3 \std{0.7} & 52.4 \std{0.3} & 75.8 \std{0.1} & 76.6 \std{0.3} & 66.5 \\
MLDG & 61.5 \std{0.9} & 53.2 \std{0.6} & 75.0 \std{1.2} & 77.5 \std{0.4} & 66.8 \\
I-Mixup & 62.4 \std{0.8} & 54.8 \std{0.6} & 76.9 \std{0.3} & 78.3 \std{0.2} & 68.1 \\
SagNet & 63.4 \std{0.2} & 54.8 \std{0.4} & 75.8 \std{0.4} & 78.3 \std{0.3} & 68.1 \\
CORAL & 65.3 \std{0.4} & 54.4 \std{0.5} & 76.5 \std{0.1} & 78.4 \std{0.5} & 68.7 \\
 \midrule
\multicolumn{6}{l}{ERM + FRR-L}\\
\midrule
Default & 64.4 \std{0.1} & 55.6 \std{0.5} & 76.5 \std{0.2} & 77.5 \std{0.2} & 68.4 \\
\namet & 64.2  \std{0.3}   &  55.9  \std{0.5}    &   77.6   \std{0.2}  & 79.9  \std{0.3}   & 69.4 \\
\namee &  64.5 \std{0.3} & 56.9 \std{0.5} & 77.8 \std{0.3} & 80.0 \std{0.4} & \textbf{69.8}\\
 \midrule
\multicolumn{6}{l}{ERM + FRR}\\
\midrule
Default & 64.5 \std{0.2} & 58.4 \std{0.1} & 76.6 \std{0.3} & 78.3 \std{0.1} & 69.4 \\
\namet &  65.6  \std{0.3} &    57.1 \std{0.3}   & 76.8 \std{0.3} & 80.2 \std{0.2} & \textbf{69.9} \\
\namee & 65.6 \std{0.5} & 56.9 \std{0.3} & 76.9 \std{0.1}  & 79.7 \std{0.3}  & 69.8 \\
\bottomrule
\end{tabular}
\label{office_overall}
\end{table}

\begin{table}[hbt!]
\centering
\small
\renewcommand{\arraystretch}{1.1}
\caption{\small\textbf{Out-of-domain accuracies (\%) on} \data{DomainNet}\textbf{.}}
\vskip 0.15in
\begin{tabular}{lllllllc}
\toprule
\textbf{Algorithm} & \textbf{clip} & \textbf{info} & \textbf{paint} & \textbf{quick} & \textbf{real} & \textbf{sketch} & \textbf{Avg} \\
\midrule
MMD & 32.1 \std{13.3} & 11.0 \std{4.6} & 26.8 \std{11.3} & 8.7 \std{2.1} & 32.7 \std{13.8} & 28.9 \std{11.9} & 23.4 \\
GroupDRO & 47.2 \std{0.5} & 17.5 \std{0.4} & 33.8 \std{0.5} & 9.3 \std{0.3} & 51.6 \std{0.4} & 40.1 \std{0.6} & 33.3 \\
VREx & 47.3 \std{3.5} & 16.0 \std{1.5} & 35.8 \std{4.6} & 10.9 \std{0.3} & 49.6 \std{4.9} & 42.0 \std{3.0} & 33.6 \\
IRM & 48.5 \std{2.8} & 15.0 \std{1.5} & 38.3 \std{4.3} & 10.9 \std{0.5} & 48.2 \std{5.2} & 42.3 \std{3.1} & 33.9 \\
Mixstyle & 51.9 \std{0.4} & 13.3 \std{0.2} & 37.0 \std{0.5} & 12.3 \std{0.1} & 46.1 \std{0.3} & 43.4 \std{0.4} & 34.0 \\
ARM & 49.7 \std{0.3} & 16.3 \std{0.5} & 40.9 \std{1.1} & 9.4 \std{0.1} & 53.4 \std{0.4} & 43.5 \std{0.4} & 35.5 \\
CDANN & 54.6 \std{0.4} & 17.3 \std{0.1} & 43.7 \std{0.9} & 12.1 \std{0.7} & 56.2 \std{0.4} & 45.9 \std{0.5} & 38.3 \\
DANN & 53.1 \std{0.2} & 18.3 \std{0.1} & 44.2 \std{0.7} & 11.8 \std{0.1} & 55.5 \std{0.4} & 46.8 \std{0.6} & 38.3 \\
RSC & 55.0 \std{1.2} & 18.3 \std{0.5} & 44.4 \std{0.6} & 12.2 \std{0.2} & 55.7 \std{0.7} & 47.8 \std{0.9} & 38.9 \\
I-Mixup & 55.7 \std{0.3} & 18.5 \std{0.5} & 44.3 \std{0.5} & 12.5 \std{0.4} & 55.8 \std{0.3} & 48.2 \std{0.5} & 39.2 \\
SagNet & 57.7 \std{0.3} & 19.0 \std{0.2} & 45.3 \std{0.3} & 12.7 \std{0.5} & 58.1 \std{0.5} & 48.8 \std{0.2} & 40.3 \\
MTL & 57.9 \std{0.5} & 18.5 \std{0.4} & 46.0 \std{0.1} & 12.5 \std{0.1} & 59.5 \std{0.3} & 49.2 \std{0.1} & 40.6 \\
ERM & 58.1 \std{0.3} & 18.8 \std{0.3} & 46.7 \std{0.3} & 12.2 \std{0.4} & 59.6 \std{0.1} & 49.8 \std{0.4} & 40.9 \\
MLDG & 59.1 \std{0.2} & 19.1 \std{0.3} & 45.8 \std{0.7} & 13.4 \std{0.3} & 59.6 \std{0.2} & 50.2 \std{0.4} & 41.2 \\
CORAL & 59.2 \std{0.1} & 19.7 \std{0.2} & 46.6 \std{0.3} & 13.4 \std{0.4} & 59.8 \std{0.2} & 50.1 \std{0.6} & 41.5 \\
MetaReg & 59.8 & 25.6 & 50.2 & 11.5 & 64.6 & 50.1 & 43.6 \\
DMG & 65.2 & 22.2 & 50.0 & 15.7 & 59.6 & 49.0 & 43.6 \\
 \midrule
\multicolumn{6}{l}{ERM + FRR-L}\\
\midrule
Default & 63.6 \std{0.1} & 20.5 \std{0.0} & 50.7 \std{0.0} & 14.6 \std{0.1} & 63.8 \std{0.1} & 53.4 \std{0.0} & 44.2 \\
\namet &65.7  \std{0.1}   &  21.9 \std{0.0} &  52.0  \std{0.1}  &  15.1 \std{0.1}   & 65.2 \std{0.1}  & 54.9 \std{0.1}  & 45.8 \\
\namee &  65.8 \std{0.1} & 22.1 \std{0.0} & 52.3 \std{0.1}  & 15.1 \std{0.1} & 65.7 \std{0.0} & 54.8 \std{0.1} & \textbf{46.0}\\
 \midrule
\multicolumn{6}{l}{ERM + FRR}\\
\midrule
Default & 64.3 \std{0.1} & 21.2 \std{0.3} & 51.1 \std{0.2} & 14.9 \std{0.6} & 64.7 \std{0.1} & 54.1 \std{0.2} & 45.1 \\
\namet & 65.6 \std{0.0} & 21.9 \std{0.0}  & 52.0 \std{0.1} & 15.0 \std{0.1}  & 65.5 \std{0.0} & 54.8 \std{0.1}
 & 45.8\\
\namee & 65.6 \std{0.0} & 21.5 \std{0.0} & 52.1 \std{0.0} & 15.0 \std{0.0}  & 65.7 \std{0.0} &  55.1 \std{0.0} & 45.8\\

\bottomrule
\end{tabular}
\label{domainnet_overall}
\end{table}

\subsubsection{DomainBed Benchmark}
\label{domainbed_dataset}

In this section, we describe the DomainBed benchmark, a challenging benchmark used to study the out-of-domain generalization capabilities of our model. To briefly explain, consider the dataset \data{PACS}, which consists of Photos, Art, cartoons, and sketches of the same set of classes (for instance, dogs and cats, amongst others). The goal of the task is to learn from three of these domains and evaluate the performance of the left-out domain (similar to a k-fold cross-validation). By doing so, we can assess the out-of-domain generalization performance of our models. In general, the metric used in this domain involves taking an average of the performance of the different k-fold splits. More information about this benchmark is available at \citet{gulrajani2020search}.

\subsubsection{Additional Results}
In this section, we briefly discuss additional results from our DomainBed experiments. Table~\ref{tab:domain_bed_additional} depicts a complete table and comparison of our proposed approach to a multitude of state-of-the-art approaches in this field. Furthermore, we also show that our proposed approach outperforms previous SOTA by \textcolor{blue}{\textbf{+0.7\%}}. 

Moreover, we also report the performance improvements when \namel~is trained with model weight averaging methods such as SWAD \cite{cha2021swad}. Table~\ref{tab:domain_bed_swad} depicts the performance improvements of \namel~over SWAD.

\begin{table*}[hbt!]
\caption{Results on DomainBed (Model selection: training-domain validation set) on the model weight averaging models such as SWAD~\cite{cha2021swad}: The bottom partition shows results of our method with \namel\ loss. }
\vskip 0.15in
\centering
\Huge
\resizebox{0.7\linewidth}{!}{
\begin{tabular}{lccccl}
\toprule
\textbf{Algorithm} &
 \data{PACS} &
  \data{VLCS} &
  \data{OfficeHome} &
  \data{DomainNet} &
  \textbf{Avg.} \\
\midrule
\multicolumn{6}{l}{SWAD}\\
\midrule
 Default  \cite{cha2021swad} & 86.5 \stdlarge{1.0} & \textbf{76.0} \stdlarge{0.7} & 66.3 \stdlarge{0.2}  & 43.8 \stdlarge{0.1}  & 68.15
 \\
\namee~\textbf{(Ours)} & \textbf{87.6} \stdlarge{0.2} & 75.4 \stdlarge{1.1} & \textbf{67.5} \stdlarge{0.3}& \textbf{44.0} \stdlarge{0.1} & \textbf{68.63}\\
\bottomrule
\end{tabular}}
\label{tab:domain_bed_swad}
\vspace{-0.2in}
\end{table*}

Furthermore, we also present the per-environment breakdown of our approach in various datasets in Table~\ref{pacs_overall}, Table~\ref{vlcs_overall}, Table~\ref{office_overall}, and Table~\ref{domainnet_overall} for PACS, VLCS, OfficeHome, and DomainNet respectively. The performance metrics of the baseline approaches were taken from \cite{gulrajani2020search}.

\subsection{Convergence of \namel~and additional costs}

\paragraph{Convergence in extreme class imbalance setting:} In the extreme class imbalance setting, we note that uniform sampling for mini-batch generation + \namel~re-weighting (as done in Algorithm~\ref{alg:algorithm_overall}) would be slower to converge than using importance sampling for mini-batch generation. This is because the former tends to have higher variance. But this is easily fixable in Algorithm~\ref{alg:algorithm_overall}. We simply update the mini-batch generation step with importance sampling; that is, we select point ‘$i$’ with probability proportional to its weight $\exp(\ell_i)$ (instead of uniform sampling that is currently done). The main reason for not considering this in this work is our desire to illustrate the generality of our approach and its applicability to a wide variety of learning tasks, without focusing too much on the class imbalance task. We believe this generality and simplicity is what makes our method quite attractive to the practitioner as showcased in some of experiments including Natural Language Processing, Image Classification, Tabular Classification, Distribution Shifts, and Meta-learning. Furthermore, Figure~\ref{fig:convergence} illustrates the convergence plots of miniGPT pre-training and ViT-S on Imagenet-1K. Overall, we note similar stable training convergence on both, while \namel~is able to focus more heavily on harder samples and reach a better minima. 

\paragraph{Additional costs of \namel:} 
Furthermore, we note that \namel~poses \textbf{NO} additional cost over standard approaches. The approach is a simple modification of the loss with a closed-form function with $\mathcal{O}(1)$ complexity, and without any changes in architecture, training regime, etc.

\begin{figure*}[hbt!]
\centering
\subfigure[Convergence on miniGPT pretraining]{%
\label{fig:mingpt_conv}
\includegraphics[width=0.47\linewidth]{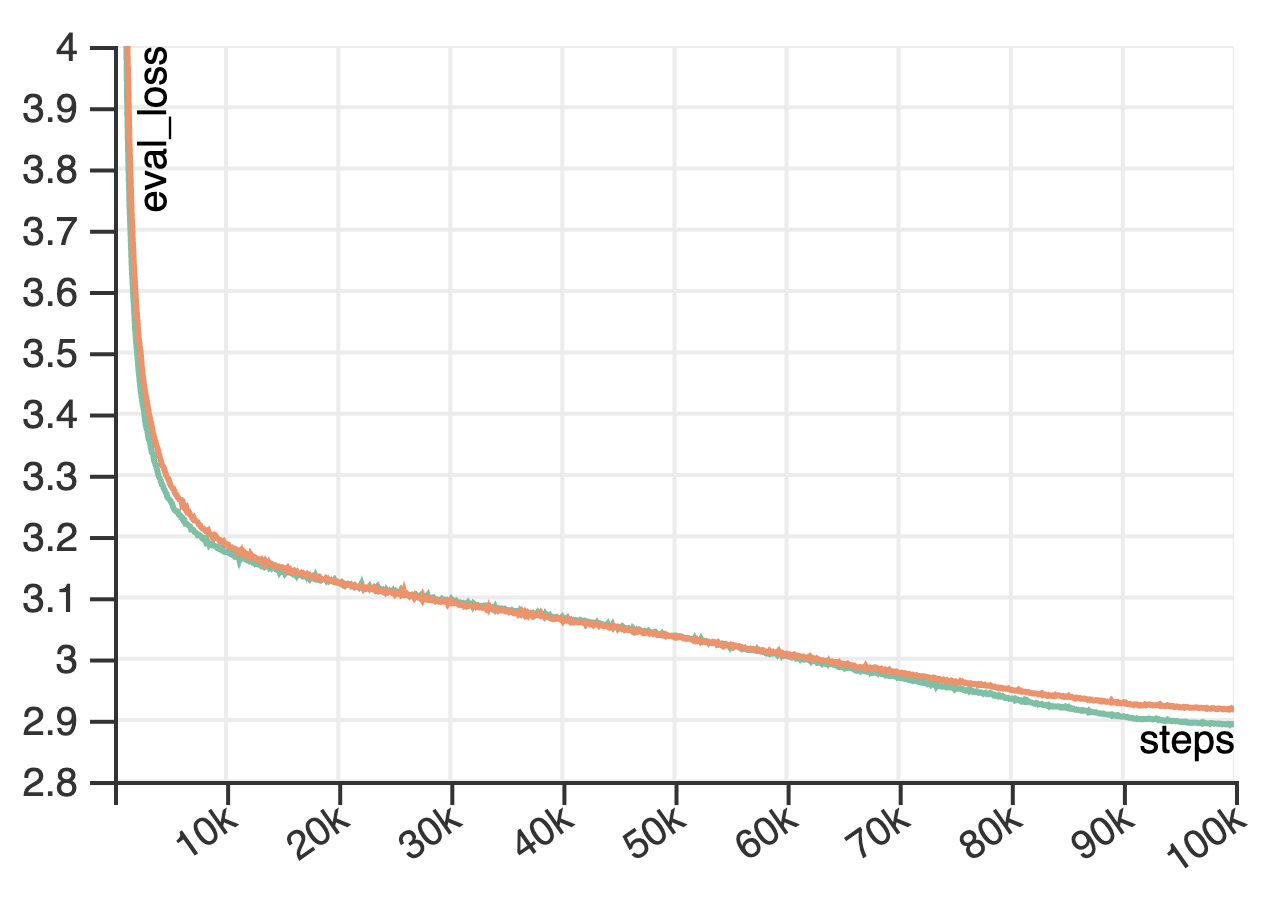}}
\quad
\subfigure[Convergence on Imagenet-1K with VIT]{%
\label{fig:vit_conv}
\includegraphics[width=0.47\linewidth]{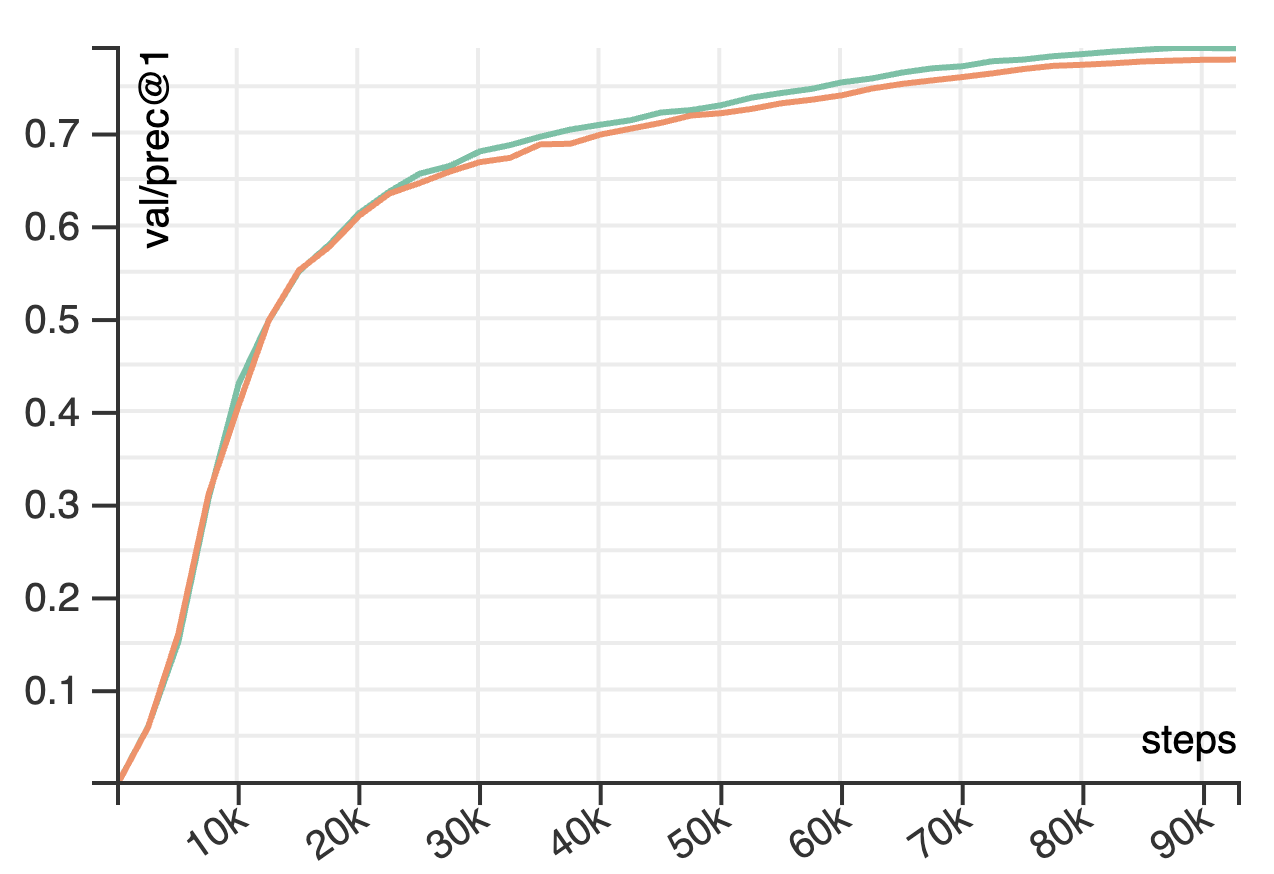}}

\caption{Convergence plots of RGD and Default training regime on real-world datasets. Here the \textcolor{orange}{orange} line denotes the default training regime, and the \textcolor{green}{green} line denotes RGD.}
\label{fig:convergence}
\end{figure*}

\end{document}